\documentclass[aos]{imsart}

\RequirePackage{amsthm,amsmath,amsfonts,amssymb}
\RequirePackage[numbers]{natbib}
\RequirePackage[colorlinks,citecolor=blue,urlcolor=blue]{hyperref}
\RequirePackage{graphicx}

\startlocaldefs
\numberwithin{equation}{section}
\theoremstyle{plain}

\newtheorem{thm}{Theorem}[section]
\newtheorem{lem}{Lemma}[section]

\newtheorem{prop}{Proposition}[section]
\theoremstyle{remark}

\newtheorem{remark}{Remark}[section]

\newcommand{\Div}{\mbox{div}}

\newcommand{\bx}{{\bf x}}

\newcommand{\bX}{{\bf X}}
\newcommand{\bY}{{\bf Y}}
\newcommand{\bZ}{{\bf Z}}

\newcommand{\cC}{{\cal C}}
\newcommand{\cD}{{\cal D}}
\newcommand{\cE}{{\cal E}}

\newcommand{\cL}{{\cal L}}
\newcommand{\cM}{{\cal M}}

\newcommand{\cP}{{\cal P}}
\newcommand{\cR}{{\cal R}}
\newcommand{\cS}{{\cal S}}

\newcommand{\mR}{\mathbb{R}}

\newcommand{\mE}{\mathbb{E}}
\newcommand{\mF}{\mathbb{F}}
\newcommand{\mP}{\mathbb{P}}

\newcommand{\fracpartial}[2]{\frac{\partial #1}{\partial  #2}}
\newcommand{\fracderiv}[2]{\frac{d #1}{d #2}}
\newcommand{\ra}{\rightarrow}
\newcommand{\iy}{\infty}
\newcommand{\pt}{\partial}

\newcommand{\cov}{\mbox{Cov}}
\newcommand{\ve}{\varepsilon}

\endlocaldefs

\begin{document}

\begin{frontmatter}

\title{
Computational and Statistical Asymptotic Analysis of the JKO Scheme for Iterative Algorithms to update distributions} 
\runtitle{Computational and Statistical Analysis of the JKO Scheme}

\begin{aug}
\author[A]{\fnms{Shang}~\snm{Wu}\ead[label=e1]{shangwu@fudan.edu.cn}}
\and 
\author[B]{\fnms{Yazhen}~\snm{Wang}\ead[label=e2]{yzwang@stat.wisc.edu}}
\address[A]{Fudan University\printead[presep={,\ }]{e1}}

\address[B]{University of Wisconsin-Madison\printead[presep={,\ }]{e2}}

\end{aug}

\begin{abstract}
The seminal paper of Jordan, Kinderlehrer, and Otto \cite{jordan1998variational} introduced what is now widely known as the JKO scheme, an iterative algorithmic framework for computing distributions. This scheme can be interpreted as a Wasserstein gradient flow and has been successfully applied in machine learning contexts, such as deriving policy solutions in reinforcement learning \cite{zhang2018policy}.
In this paper, we extend the JKO scheme to accommodate models with unknown parameters. Specifically, we develop statistical methods to estimate these parameters and adapt the JKO scheme to incorporate the estimated values. To analyze the adopted statistical JKO scheme, we establish an asymptotic theory 
via stochastic partial differential equations that describes its limiting dynamic behavior. Our framework allows both the sample size used in parameter estimation and the number of algorithmic iterations to go to infinity. 
This study offers a unified framework for joint computational and statistical asymptotic analysis of the statistical JKO scheme. On the computational side, we examine the scheme's dynamic behavior as the number of iterations increases, while on the statistical side, we investigate the large-sample behavior of the resulting distributions computed through the scheme.
We conduct numerical simulations to evaluate the finite-sample performance of the proposed methods and validate the developed asymptotic theory.



\end{abstract}
\begin{keyword}[class=MSC]
\kwd[Primary ]{60J60}
\kwd{62E20}
\kwd[; secondary ]{35A15}
\kwd{35K15}
\end{keyword}

\begin{keyword}
\kwd{JKO scheme}
\kwd{Wasserstein gradient flow}
\kwd{Langevin diffusion}
\kwd{Fokker-Planck equation}
\kwd{Higher-order asymptotic distribution}
\kwd{Stochastic partial differential equation}
\end{keyword}

\end{frontmatter}

 \section{Introduction}
 
 The seminal work of Jordan, Kinderlehrer, and Otto \cite{jordan1998variational} developed what is now widely known as the JKO scheme, a foundational method for generating iterative algorithms to compute distributions and reshaping our understanding of sampling algorithms. The JKO scheme can be interpreted as a gradient flow of the free energy with respect to the Wasserstein metric, often referred to as the Wasserstein gradient flow. This interpretation has led to significant advancements in machine learning, including applications in reinforcement learning to solve policy-distribution optimization 
 problems \cite{zhang2018policy}. While the JKO scheme traditionally assumes that the underlying model is fully known, in this paper, we relax this assumption by allowing models with unknown parameters. We develop statistical approaches to estimate these parameters and adapt the JKO scheme to work with the estimated values. 
 
 Specifically, Langevin equations---stochastic differential equations---play a key role in describing the evolution of physical systems, facilitating stochastic gradient descent in machine learning, and enabling Markov chain Monte Carlo (MCMC) simulations in numerical computing. For examples and detailed discussions, see \cite{bishop2006pattern, bernton2018langevin, shen2019randomized, dalalyan2020sampling, chewi2021analysis, lee2021structured, ma2021there}. Solutions to Langevin equations, known as Langevin diffusions, are stochastic processes whose distributions evolve according to the Fokker-Planck equations \cite{gardiner1985handbook, risken1996fokker}.
The JKO scheme provides an iterative framework for computing these distributions by solving a sequence of optimization problems \cite{jordan1998variational}. It has been established that the outputs of this iterative algorithm converge to the true distributions as the number of iterations approaches infinity. Furthermore, the scheme offers a new perspective by interpreting Langevin diffusion as the gradient flow of the Kullback-Leibler (KL) divergence over the Wasserstein space of probability measures \cite{ambrosio2008gradient}. This insight has provided a deeper understanding of MCMC algorithms based on Langevin diffusions and has paved the way for novel optimization techniques to develop advanced MCMC and gradient descent algorithms. Relevant examples can be found in \cite{alquier2016properties, liu2016stein, blei2017variational, wibisono2018sampling, duncan2019geometry, durmus2019analysis, lu2019accelerating, salim2020wasserstein, bunne2022proximal, yan2023gaussian}.
In this paper, we establish an asymptotic theory for the adopted JKO scheme, focusing on its dynamic behavior and convergence properties in the presence of estimated model parameters.

We consider a Langevin equation with an unknown parameter and analyze the associated JKO scheme. The unknown parameter is estimated using observations of the Langevin diffusion at discrete time points. Based on this estimated parameter, we formulate the JKO scheme to develop iterative algorithms for computing the distributions of the Langevin diffusion. Both online and offline parameter estimation approaches are considered. We derive the asymptotic distributional theory for the scaled difference between the computed and true distributions of the Langevin diffusion. This analysis is carried out as both the number of iterations in the algorithm and the number of observations used for parameter estimation tend to infinity. The scaling is achieved by the square root of the number of observations and the time step size. 
The asymptotic distribution of the scaled difference is governed by stochastic partial differential equations. The theory establishes a novel, unified framework for conducting a joint computational and statistical asymptotic analysis of the JKO scheme. From a statistical perspective, this joint analysis provides tools for inferential analysis of statistical methods associated with the JKO scheme. From a computational perspective, it allows us to understand and quantify random fluctuations and their impact on the dynamic and convergence behavior of learning algorithms derived from the JKO scheme.

The rest of the paper proceeds as follows. Section 2 introduces the Langevin and Fokker-Planck equations, along with a review of the Kullback-Leibler divergence and the Wasserstein distance. Section 3 presents the JKO scheme and its associated gradient flow. It also describes methods for estimating unknown parameters in the Langevin equation using observations of Langevin diffusion at discrete time points. Both online and offline parameter estimation approaches are considered, and the JKO scheme is formulated under the estimation scenario to develop iterative algorithms for computing the distributions of the Langevin diffusion. Section 4 establishes the asymptotic distributional theory for the outputs of the JKO scheme in the offline case as the number of observations used in parameter estimation increases to infinity and the time step approaches zero. Section 5 extends this analysis to the online case, where we propose various estimators under different online frameworks and derive the corresponding asymptotic distributions. Section 6 applies our theoretical framework to scenarios where the distribution is constrained within the Bures-Wasserstein space. Both offline and online results are established for this application. Section 7 presents a numerical study using a simple example to evaluate the performance of our theoretical results. All technical proofs are deferred to the Supplementary Materials.

\section{Reviews on Langevin diffusions and gradient flows in the Wasserstein space} 
\label{Sec-review} 

Consider the Langevin equation 
\begin{equation} \label{Langevin}
  dX(t) = - \nabla \! \Psi(X(t)) dt + \sqrt{ 2/\beta } \, dB_t ,  \;\; X(0) = X_0, 
\end{equation} 
where $\nabla$ is the gradient operator, $\Psi(x)$ denotes a potential that is a smoothing function from $\mR^d$ to $\mR_+=[0, \infty)$, $\beta>0$ is a constant, $B_t$ is a standard $d$-dimensional Brownian motion, and initial value $X_0$ is a $d$-dimensional
random vector. 
The solution $X(t)$ to the Langevin equation (\ref{Langevin}) refers to the Langevin
diffusion. This equation arises, for example, as the approximation to the motion of chemically bound particles \cite{chandrasekhar1943stochastic, schuss1980singular}. The probability density function $\rho(t,x)$ of $X(t)$ is given by the solution to the 
Fokker-Planck equation with the following form \cite{schuss1980singular, gardiner1985handbook, risken1996fokker, vempala2019rapid}
\begin{equation} \label{Fokker-Planck}
  \fracpartial{\rho(t,x)}{t} = \Div [  \rho(t,x)  \nabla \Psi(x) ] + \beta^{-1} \Delta \rho(t,x), \;\; \rho(0,x) = \rho^0(x), 
\end{equation} 
where 
$\rho^0(x)$ is the probability density function of the initial value $X_0$ on $\mR^d$.
Note that the solution $\rho(t,x)$ of (\ref{Fokker-Planck}) must be a probability density on $\mR^d$ for almost every fixed time $t$---that is, $\rho(t,x) \geq 0$ for almost 
every $(t,x) \in \mR^d \times [0, \infty)$, and $ \int_{\mR^d} \rho(t,x) dx = 1$ for almost every $t \in [0, \infty)$.  

For the potential $\Psi(x)$ satisfying appropriate growth conditions, $X(t)$ has a unique stationary distribution with probability density function $\pi(x)$ that takes the form of 
the Gibbs distribution \cite{gardiner1985handbook, risken1996fokker},
\begin{equation} \label{Gibbs} 
   \pi(x) = \frac{1}{Z} \exp[  - \beta \Psi(x) ] , \;\; Z = \int_{\mR^d} \exp[  - \beta \Psi(x) ] dx , 
\end{equation} 
where $Z$ is called the partition function. 

The Gibbs distribution $\pi$ satisfies a variational principle  \cite{jordan2017extended}---it minimizes $\mF_\Psi(\rho)$ over all probability densities $\rho$ on $\mR^d$, where 
$\mF_\Psi(\rho)$ denotes the following 
functional, 
\begin{eqnarray} \label{energy-functional} 
\mF_\Psi (\rho) &=& \cE_\Psi(\rho) - \beta^{-1} \cS(\rho) = \beta^{-1} \cD_{KL}( \rho \| \pi) 
  - \beta^{-1} \log Z, 
\end{eqnarray} 
\begin{equation} \label{E-S-functional} 
\cE_\Psi(\rho) =  \int_{\mR^d} \Psi(x) \rho(x) dx , \;\; \cS(\rho)  = -\int_{\mR^d} \rho(x) \log[\rho(x)] \,dx. 
 \end{equation}
Here $\cE_\Psi(\rho)$ represents an energy functional, $\cS(\rho)$ stands for the 
entropy functional, and they obey 
\[ \fracpartial{\cE_\Psi(\rho)}{\rho} = \Psi, \;\;  \fracpartial{\cS(\rho)}{\rho} = - \log \rho - 1 . \]
Furthermore, if $\rho(t,x)$ satisfies  (\ref{Fokker-Planck}), then $\mF_\Psi(\rho)$ will decrease with time \cite{jordan1996route, risken1996fokker}. 
$\cD_{KL}(P_1 \| P_2)$ in (\ref{energy-functional}) 
denotes the Kullback-Leibler divergence (also known as relative entropy) between probability measures $P_1$ and 
$P_2$ on $\mR^d$ \cite{amari2000methods, ay2017information}, 
\[ \cD_{KL}(P_1 \| P_2) = \int_{\mR^d} \fracderiv{P_1}{P_2} \log \left( \fracderiv{P_1}{P_2}\right) d P_2, \]
where $\fracderiv{P_1}{P_2}$ represents the Radon-Nikodym derivative of $P_1$ with respect to $P_2$. If $P_1$ and $P_2$ have 
probability density functions $\rho_1$ and $\rho_2$ with respect to the Lebesgue measures on $\mR^d$ , then we write $\cD_{KL}(\rho_1 \| \rho_2)$ for 
$\cD_{KL}(P_1 \| P_2)$ and 
\[  \cD_{KL}(\rho_1 \| \rho_2) = \int_{\mR^d} \rho_1(\bx)  \log \left[ \frac{\rho_1(\bx)} {\rho_2(\bx) }\right] d\bx. \]

Suppose that $P_1$ and $P_2$ are two probability measures on $\mR^d$. The Wasserstein distance of order two between $P_1$ and 
$P_2$ is defined by 
\begin{equation} \label{Wasserstein-1}
 d^2_W (P_1, P_2) = \inf_{P \in \cP(P_1, P_2)} \int_{\mR^d \times \mR^d} \| x - y \|^2 P(dxdy) ,  
\end{equation} 
where the symbol $\| \cdot \|$ stands for the usual Euclidean norm on $\mR^d$, and $\cP(P_1, P_2)$ denotes the set of all probability measures on $\mR^d \times \mR^d$ with the marginal measures given by $P_1$ and $P_2$---that is, $\cP(P_1, P_2)$ consists of all probability measures $P$ on $\mR^d \times \mR^d$ that satisfy 
\[  P(A \times \mR^d) = P_1(A), \;  P(\mR^d \times A) = P_2(A) \]
for every Borel set $A$ on $\mR^d$. 
$d_W$ defines a metric on the set of probability measures with finite second moments \cite{givens1984class, rachev1991probability}, and the infimum can be reached when $P_1$ and $P_2$ have finite second moments \cite{givens1984class}. The Wasserstein distance is equivalent to the definition \cite{rachev1991probability}
\begin{equation} \label{Wasserstein-2}
 d^2_W (P_1, P_2) = \inf_{P \in \cP(P_1, P_2)}  \mE_P\left\{ \| \bX - \bY \|^2; (\bX, \bY) \sim P \right \} , 
 \end{equation}
where $\mE_P$ denotes the expectation under probability measure $P$, and the infimum is taken over all random variables $\bX$ and $\bY$ such that 
$(\bX, \bY)$ has a joint distribution $P$, with the marginal distribution $P_1$ for $\bX$ and the marginal distribution $P_2$ for $\bY$. In other words,  the infimum is taken over all possible couplings of the random 
variables $\bX$ and $\bY$ with the marginal distributions $P_1$ and $P_2$, respectively. This equivalent definition (\ref{Wasserstein-2}) and 
the well-known Skorokhod representation theorem  \cite{billingsley2017probability} immediately indicate that convergence in the Wasserstein distance is equivalent to the usual 
weak convergence plus convergence of second moments. 
If the probability measures $P_1$ and $P_2$ are absolutely continuous with respect to the Lebesgue measure on $\mR^d$, with probability 
densities given by $\rho_1$ and $\rho_2$, respectively, we denote by $\cP(\rho_1, \rho_2)$ the set of all probability measures on $\mR^d \times \mR^d$ 
with the first and second marginal densities given by $\rho_1$ and $\rho_2$, respectively.  Correspondingly, we write $d_W (\rho_1, \rho_2)$ for the Wasserstein distance between $P_1$ and $P_2$. We define the Wasserstein space of distributions on $\mR^d$ to be the space of all distributions
 equipped with the Wasserstein distance. See also \cite{brenier1991polar, gangbo1995optimal, caffarelli2017allocation, chizat2018interpolating}.
 
Note that the Gibbs distribution $\pi$ satisfies 
the variational principle for $\mF_\Psi(\rho)$ in (\ref{energy-functional}), and $\mF_\Psi(\rho)$ is a functional over the Wasserstein space of distributions, which is equal to a Kullback-Leibler divergence up to a constant. Thus, 
the distributional evolution of the Langevin diffusion is a gradient flow of a Kullback-Leibler divergence over the Wasserstein space of probability distributions \cite{ambrosio2008gradient}.

\section{The JKO scheme}
\label{SEC-JKO} 
We fix some notations and conditions.  Given a positive number $x$, denote its integer part by $[x]$. 
Let $\mR_+ = [0, \infty)$ and $\mR^d$ be the $d$-dimensional Euclidean space. For $a, b \in \mR^d$, let 
$\langle a,b\rangle=a'b$ be the natural inner product of $a$ and $b$. 
Denote by $L^\infty$ and $L^1$ the classes of all functions which are bounded or their absolute integrals are bounded, respectively. 
$C^\infty(\mR^d)$ denotes the class of all functions on $\mR^d$ with continuous derivatives of all order, and 
$C^\infty_0(\mR^d)$ denotes the class of all functions in $C^\infty(\mR^d)$ that have bounded supports. 

Suppose that $\Psi(x)$ satisfies the assumption, 
\[ \mbox{(A1): } \Psi \in C^\infty(\mR^d), \;\; \Psi(x) \geq 0, \;\; |\nabla \Psi(x) | \leq c [ \Psi(x) + 1] \mbox{ for all } x \in \mR^d, \]
where $c$ is a constant. 
Denote by $\Xi$ the set of all probability density functions on $\mR^d$ with finite second moments---namely, 
\[ \Xi = \left\{ p: \mR^d \rightarrow  \mR_+ \mbox{ is measurable } \left|  \int_{\mR^d} p(x) dx = 1, \cM(\rho) 
   < \infty  \right\} \right.  , \]
where $\cM(\rho)$ denotes the second moment of $p$, 
\begin{equation} \label{M-functional} 
  \cM(\rho) = \int_{\mR^d} \| x \|^2 p(x) dx . 
\end{equation} 

\subsection{The plain scheme} 
\label{SEC-JKO-plain}
The original JKO scheme is referred to as the plain JKO scheme in this paper. 
Given an initial probability distribution $\rho^{(0)}=\rho^0$ on $\mR^d$ and a time step $\delta$, consider the following iterative discrete algorithm for 
computing $\rho^{(k)}$ from $\rho^{(k-1)}$ \cite{
jordan1998variational}, 
\begin{equation} \label{JKO-1}
 \rho^{(k)} = \arg \min_{\rho \in \Xi} \left\{ \frac{1}{2} \left[ d_W \left(\rho, \rho^{(k-1)}  \right)\right]^2 + \delta \mF_\Psi(\rho) \right\}  , 
\end{equation}
where $\mF_\Psi(\cdot)$ is the energy functional defined in (\ref{energy-functional}) and $d_W (\rho_1, \rho_2)$ is the Wasserstein distance between $\rho_1$ and $\rho_2$ given by (\ref{Wasserstein-1}). For the origin of the JKO scheme, see \cite{jordan1995statistical, jordan1997ideal, otto1998dynamics}.

Given the probability density function sequence $\rho^{(k)}$, define the interpolation $\rho_\delta: \mR_+ \times\mR^d \rightarrow \mR_+$ as follows. For any $t \in \mR_+$, 
let $k=\lceil t/\delta\rceil$ (the integer part of $t/\delta$), and define $\rho_\delta(t,x) = \rho^{(k)}(x)$---that is, for each $t\geq 0$, we define the probability density function  
\begin{equation}\label{JKO-2} 
 \rho_\delta(t,x) = \rho^{(\lceil t/\delta\rceil)}(x) . 
\end{equation} 
The plain JKO scheme refers to the described iterative algorithm and interpolation to obtain $\rho_\delta(t,x)$. 

It is shown in \cite{jordan1998variational} that as $\delta \rightarrow 0$, for all $t \in (0, \infty)$, $\rho_\delta(t,\cdot)$ weakly converges to $\rho(t,\cdot)$ in $L^1(\mR^d)$ and strongly converges to $\rho$ in $L^1( (0, T)\times \mR^d)$ for all $T < \infty$. To be specific, 
for all $t \in (0, \infty)$, we have as $\delta \rightarrow 0$, 
\[      \int_{\mR^d} \rho_\delta(t,x) h(x)dx \longrightarrow  \int_{\mR^d} \rho(t,x) h(x)dx, \]
for any bounded continuous function $h(\cdot)$ on $\mR^d$, and 
\[   \int_0^T dt \int_{\mR^d} | \rho_\delta(t,x) - \rho(t,x) | dx \rightarrow 0 \mbox{ as } \delta \rightarrow 0 . \]    
There is also an exact upper bound of the difference. For any $\xi\in \cC^\infty_0(\mR_+\times \mR^d)$, we have
\[ \left|  \int_{\mR_+ \times \mR^d} [ \rho_\delta (t,x) - \rho(t,x)] \xi(t, x) dt dx \right|  = O(\delta ) ,  \]
where $\rho(t,x)$ is the solution to the Fokker-Planck equation (\ref{Fokker-Planck}) with initial condition 
\[   \rho(t,\cdot) \rightarrow \rho^{(0)}(\cdot) \mbox{ strongly in } L^1(\mR^d) \mbox{ for } t \downarrow 0  \Longleftrightarrow 
    \int_{\mR^d} | \rho(t,x) - \rho^{(0)}(x) | dx  \rightarrow 0 \mbox{ as } t \downarrow 0 , \]
and
     $\cM(\rho(t,\cdot)), \cE_\Psi(\rho(t,\cdot)) \in L^\infty((0, T))$ for all $T <\infty$, 
where $\cE_\Psi(\cdot)$ and $\cM(\cdot)$ are defined in (\ref{E-S-functional}) and (\ref{M-functional}), respectively. 



\subsection{The statistical JKO scheme} 

The plain JKO scheme assumes a known stochastic model---namely, the function $\Psi(x)$ is known so that we have complete 
knowledge about the Langevin equation. In applications such as machine learning, $\Psi$ may be unknown, and we need to estimate it from data. 
The JKO scheme with parameter estimation is referred to as the statistical JKO scheme. 
This section investigates  the JKO scheme for an unknown $\Psi$ that is estimated based on discrete observations from the Langevin equation. 
We consider both online and offline estimation scenarios. 

\subsubsection{The statistical JKO scheme in the offline case} 
\label{SEC-JKO-offline}
In the offline case, we have $n$ fixed discrete observations from the Langevin equation, and $\Psi$ is estimated once by using all the $n$ observations. 
Denote by $\hat{\Psi}_n$ an offline estimator of $\Psi$. We construct the estimator $\hat{\Psi}_n$ as follows. Suppose that $\Psi(x)$ is known up to an unknown parameter, namely, $\Psi$ can be parametrized as $\Psi_\theta(x)$, where we know the function form of $\Psi_\theta(x)$ but do not know the parameter $\theta$, and the unknown parameter $\theta$ is assumed to be in a parameter space $\Theta$. 
We estimate $\Psi(x)$ by plugging an estimator of $\theta$ into $\Psi_\theta(x)$. Specifically, assume that we have $n$ discrete observations from 
the Langevin diffusion, and denote by $X(t_i)$, $i=1, \cdots, n$, the $n$ observations at discrete time points $t_i = i \eta$, where $\eta$ is a fixed constant. We define estimator $\hat{\theta}_n$ to be a solution to the following estimating equation,
\begin{equation}\label{eq-Est1}
\sum_{i=1}^n\nabla_x\Psi_{\theta}(X(t_i))=0 .
\end{equation}
And let $\hat{\Psi}_n(x)=\Psi_{\hat{\theta}_n}(x)$. 

As $\rho(t,x)$ and $\rho_\delta (t,x)$ with $\Psi$ are described in sections \ref{Sec-review} and \ref{SEC-JKO-plain},  respectively, we denote by 
 $\tilde{\rho}^n(t,x)$ and $\hat{\rho}_\delta^n(t,x)$ their counterparts with $\Psi$ replaced by its offline estimator $\hat{\Psi}_n$, respectively. 
 Specifically, we define $\hat{\rho}_\delta^n(t,x)$ by (\ref{JKO-1}) and  (\ref{JKO-2}) with $\Psi$ replaced by $\hat{\Psi}_n$ as follows, 
\begin{equation}\label{JKO-off1}
\hat{\rho}_{n}^{(k)}={\rm arg}\min_{\rho\in \Xi} \left\{\frac{1}{2}\left[d_W\left(\rho,\hat{\rho}_{n}^{(k-1)}\right)\right]^2+\delta \mF_{\hat{\Psi}_n}(\rho)\right\} , 
\; \; \; \hat{\rho}_{n}^{(0)}=\rho^0 ,
\end{equation}
and
\begin{equation} \label{JKO-off2}
\hat{\rho}_{\delta}^n(t,x)=\hat{\rho}_{n}^{(\lceil t/\delta\rceil)}(x)  .
\end{equation} 
We denote by $\tilde{\rho}^n(t,x)$ the solution to the following Fokker-Planck equation obtained from (\ref{Fokker-Planck}) with $\Psi$ replaced by $\hat{\Psi}_n$,
\begin{equation} \label{FP-off1}
  \fracpartial{\tilde{\rho}^n(t,x)}{t} = \Div [  \tilde{\rho}^n(t,x)  \nabla \hat{\Psi}_n(x) ] + \beta^{-1} \Delta \tilde{\rho}^n(t,x), \;\; \tilde{\rho}^n(0,x) = \rho^0(x) . 
\end{equation} 

In the rest of the paper, we always use the notation $\hat{\rho}(t,x)$ to express the discrete process defined by the JKO scheme, and we use $\tilde{\rho}(t,x)$ to express the solution to the Fokker-Planck equation where $\Psi$ is estimated from data.

\subsubsection{The statistical JKO scheme in the online case} 
\label{SEC-JKO-online}
Applications such as reinforcement learning need to consider the online estimation case that the estimation of $\Psi$ is periodically updated 
with new data available from observing the Langevin diffusion 
during the system evolution. 
We study two online frameworks. In one framework, assume that we have batches of independent observations, where each batch is of size $m$, and different batches are independent. In step $k$, we have the $k$-th batch observations $X^{(k)}=\{X^{(k)}(\eta),\dots,X^{(k)}(m\eta)\}$. Using all observations before step $k$, that is, $\left\{X^{(1)}, X^{(2)},\dots,X^{(k)}\right\}$, we can obtain the estimator $\hat{\theta}^k$ by solving the estimating equation,
\begin{equation}\label{eq-Est2}
\sum_{j=1}^k\sum_{i=1}^m\nabla_x\Psi_{\theta}(X^{(j)}(t_i))=0 .
\end{equation}

Then we estimate $\Psi$ by $\Psi_{\hat{\theta}^k}$ in the $k$-th step, and define
\begin{equation}\label{JKO-on3}
\hat{\rho}_{m,1}^{(k)}={\rm arg}\min_{\rho\in \Xi} \left\{\frac{1}{2}\left[d_W\left(\rho,\hat{\rho}_{m,1}^{(k-1)}\right)\right]^2+\delta \mF_{\Psi_{\hat{\theta}^k}}(\rho)\right\} , \;\;\; \hat{\rho}_{m,1}^{(0)}=\rho^0 , 
\end{equation}
and
\begin{equation}\label{JKO-on4}
\hat{\rho}_{\delta,1}^m(t,x)=\hat{\rho}_{m,1}^{(\lceil t/\delta\, \rceil)}(x) . 
\end{equation}


Section \ref{SEC-Asymptotic-online}  will discuss other variants and extensions in this framework. 

In another framework, assume that we sequentially observe the Langevin diffusion and construct sequential estimators of $\theta$. Denote by $X(i\eta)$, $i \geq 1$, the sequential observations  from the Langevin equation (\ref{Langevin}). In the $k$-th step, we use the cumulative observations, 
$X(1\eta), \cdots, X(k \eta)$, up to the time $k\eta$ to estimate $\theta$ by solving the estimating equation (\ref{eq-Est1}) with $X(1 \eta), \cdots, X(n \eta)$ replaced by $X(1 \eta),  \cdots, X(k \eta)$, and the estimator is exactly $\hat{\theta}_k$. Then we define
\begin{equation}\label{JKO-on1}
\hat{\rho}^{(k)}_2={\rm arg}\min_{\rho\in \Xi} \left\{\frac{1}{2}\left[d_W\left(\rho,\hat{\rho}^{(k-1)}\right)\right]^2+\delta \mF_{\Psi_{\hat{\theta}_k}}(\rho)\right\} , \;\;\;
\hat{\rho}^{(0)}=\rho^0 ,
\end{equation}
and let
\begin{equation}\label{JKO-on2}
\hat{\rho}_{\delta,2}(t,x)=\hat{\rho}_2^{(\lceil t/ \delta\,\rceil)}(x) . 
\end{equation}

\section{Asymptotic theory of the statistical JKO scheme in the offline case} 
\label{SEC-Asymptotic-offline} 

For the original JKO scheme, $\rho(t,x)$ and $\rho_\delta(t,x)$ are described by the Fokker-Planck equation (\ref{Fokker-Planck}) and the iterative algorithm (\ref{JKO-1})-(\ref{JKO-2}), respectively. It has been shown that as the time step $\delta \rightarrow 0$, for all $t \in (0, \infty)$, $\rho_\delta(t,\cdot)$ weakly converges to $\rho(t,\cdot)$ in $L^1(\mR^d)$,  and $\rho_\delta(\cdot,\cdot)$ strongly converges to $\rho(\cdot, \cdot)$ in $L^1( (0, T)\times \mR^d)$ for all $T < \infty$. The asymptotic results for the original JKO scheme are purely computational, without any statistical consideration. We will establish asymptotic theory for the joint computational and statistical analysis of the statistical JKO scheme---namely, the JKO scheme with statistical estimation of the model parameter. 

To ensure the existence of solutions to equations, optimizations, and proper definitions of estimators and study their asymptotics, we need to impose the following assumptions on 
$\nabla_x\Psi_\theta(x)$ \cite{jones2004markov, locherbach2013ergodicity}, where $\nabla_x$ denotes the gradient operator with respect to $x$. 

\vspace{0.1in}

(A2): 
$\nabla_x\Psi_\theta(x)$ is of linear growth in $x$, $\nabla^\prime_\theta \nabla_x\Psi_\theta(x)$ is continuous in $\theta$, where $\nabla^\prime_\theta$ denotes the transpose of the gradient operator with respect to $\theta$. 

(A3): There exist constants $M_0>0$ and $r>2$ such that 
$
\langle \nabla_x\Psi_\theta(x),x\rangle\geq \frac{2r+d}{\beta}\ {\rm for}\ |x|\geq M_0.
$

(A4): $E_\pi[\nabla^\prime_\theta \nabla_x\Psi_\theta(X)]$ is invertible, and for $r$ given in (A42), there exists $\kappa>\frac{2}{r-2}$ such that $E_\pi | \nabla_x\Psi_\theta(X(t))  |^{2+\kappa}<\iy$, where $\pi$ is the invariant distribution of $X(t)$. 

\vspace{0.1in}

Recall that $\hat{\rho}_\delta^n(t,x)$ and $\tilde{\rho}^n(t,x)$ defined in Section \ref{SEC-JKO-offline} are 
the counterparts of $\rho_\delta (t,x)$ and $\rho(t,x)$ 
with $\Psi$ replaced by its offline estimator $\hat{\Psi}_n=\Psi_{\hat{\theta}_n}$, respectively, 
where the estimator  $\hat{\theta}_n$ is defined by the estimating equation (\ref{eq-Est1}). 
Let $\hat{V}_\delta^n (t,x) = \sqrt{n} [ \hat{\rho}_\delta^n(t,x) - \rho(t,x)]$, $\tilde{V}^n (t,x) = \sqrt{n} [ \tilde{\rho}^n(t,x) - \rho(t,x)]$. We have the following theorem to establish the asymptotic distribution of 
$\hat{V}_\delta^n (t,x)$ and $\tilde{V}^n (t,x)$.



\begin{thm} \label{thm-1}  
Assume the conditions (A1),(A2),(A3),(A4). 
Then $\hat{V}_\delta^n (t,x)$ converges to $V(t,x)$ in the sense that for any $\xi\in\cC_0^\iy(\mR_+\times\mR^d)$, as $n\ra\iy$, $\delta\ra 0$, $\delta\sqrt{n}\ra 0$,
\begin{equation}\label{eq-dd1}
\int_{\mR_+\times\mR^d}[\hat{V}_\delta^n(t,x)-V(t,x)]\xi dtdx\stackrel{P}{\ra} 0, 
\end{equation}
where $V(t, x)$ satisfies the following PDE, 
\begin{equation}\label{eq-dd2}
\pt_t V = \Div \left( V \nabla \Psi \right) +  \Div\left( \rho \nabla \tau(x)\bZ\right) + \beta^{-1} \Delta V , V(0,x)=0, 
\end{equation}
$\tau(x)=\nabla_\theta'\Psi_\theta(x)\gamma_\theta$, $\gamma_\theta^2$ is the asymptotic variance of $\hat{\theta}_n$ \cite{jones2004markov}, and 
$\bZ$ is a random vector with zero mean and identity covariance matrix. 
Furthermore, the solution $V(t,x)$ linearly depends on $\bZ$. 
Let $v(t,x)$ obey the following PDE, 
\[ \partial_t v = \Div \left( v \nabla \Psi \right) +   \Div \left( \rho  \nabla \tau \right)  + \beta^{-1} \Delta v, v(0,x)=0 . \]
Then $V(t,x) = v(t,x) \bZ$. As $n\ra\iy$, $\tilde{V}^n (t,x)$ also converges to $V(t,x)$ in the same manner as (\ref{eq-dd1}).

\end{thm}

\begin{remark}
In this theorem, there are two types of asymptotic analysis. One type is computational asymptotics, where we employ continuous partial differential equations to model discrete iterate sequences generated from the JKO scheme, which is associated with $\delta$ treated as the time step. For each $n$, the Fokker-Planck equation (\ref{FP-off1}) provides continuous solutions as the limit of discrete iterate sequences generated from algorithm (\ref{JKO-off1}), that is,  as $\delta\ra 0$,  the process $\hat{\rho}_\delta^n(t,x)$ will converge to $\tilde{\rho}^n(t,x)$, the solution to the Fokker-Planck equation (\ref{FP-off1}) with $\Psi$ replaced by its estimator $\hat{\Psi}_n$. And the discretization error, that is, the difference between $\hat{\rho}_\delta^n(t,x)$ and $\tilde{\rho}^n(t,x)$, is bounded by $O_P(\delta)$. Another type of asymptotic analysis is statistical asymptotics, where we use random samples generated from the Langevin diffusions to estimate the function $\Psi$, with sample size $n$. As $n\ra \iy$, we can show that the sample Wasserstein gradient flow $\tilde{\rho}^n(t,x)$ converges to the population gradient flow $\rho$ with order $n^{-1/2}$, that is, $\tilde{V}^n (t,x)$ converges to its limit process $V(t,x)$. We can also establish both asymptotics when $\delta\ra 0$ and $n\ra\iy$ simultaneously. Since $\hat{\rho}_\delta^n(t,x)-\tilde{\rho}^n(t,x)$ is bounded by $O_P(\delta)$, we can select $\delta$ and $n$ properly so that the difference is of order smaller than $n^{-1/2}$, for example, $\delta\sqrt{n}\ra 0$, then $\hat{\rho}_\delta^n(t,x)$ has the same asymptotic distribution $V(t,x)$ as $\tilde{\rho}^n(t,x)$. If we are able to study the property of the distribution of the limiting process $V(t,x)$, then we can perform statitical inference, for example, conduct confidence intervals or hypothesis testing for both the discrete iterate sequences generated from the JKO algorithm and the sample Wasserstein gradient flow of the Langevin diffusion.
\end{remark}

\section{Asymptotic theory of the stochastic JKO scheme in the online case} 
\label{SEC-Asymptotic-online}

In the online case, we encounter the difficulty that rises from using different estimators of $\Psi$ in different steps of the iterative discrete scheme. We will develop new techniques 
to solve this problem. Let $\hat{\Psi}_k(x)$ be the estimator of $\Psi(x)$ in the k-th step of the online JKO scheme. 


In the first online framework, assume we have batches of independent observations $\left\{X^{(1)}, X^{(2)},\dots,X^{(k)}\right\}$.
Recall that $\hat{\rho}^m_{\delta,1}(t,x)$ are defined by (\ref{JKO-on4}) in Section \ref{SEC-JKO-online}. We denote by $\tilde{\rho}_1^m(t,x)$ the solution to the following Fokker-Planck equation.
\begin{equation}\label{eq-ii}
\pt_t\tilde{\rho}_1^m(t,x)=\Div(\tilde{\rho}_{1}^{m}\nabla_x\Psi_{\hat{\theta}^{\lceil t/\delta\rceil}}(x))+\beta^{-1}\Delta\tilde{\rho}_{1}^{m},\  \tilde{\rho}_{1}^{m}(0,x)=\rho^0(x).
\end{equation}
Let $\hat{V}_{\delta,1}^m (t,x) = \sqrt{m/\delta} [ \hat{\rho}_{\delta,1}^m(t,x) - \rho(t,x)]$, $\tilde{V}_{1}^m (t,x) = \sqrt{m/\delta} [ \tilde{\rho}_{1}^m(t,x) - \rho(t,x)]$. We have the following theorem to establish the asymptotic distribution of 
$\hat{V}_{\delta,1}^m (t,x)$ and $\tilde{V}_{1}^m (t,x)$.

\begin{thm}\label{thm-2}
Assume $\hat{\Psi}_k(x)=\Psi_{\hat{\theta}^k}(x)$ satisfies the conditions (A5),(A6):

\vspace{0.1cm}

(A5): For any $T>0$ and compact set $D\subset\mR^d$, 
$\left\{\delta\sum_{k=1}^{\lceil T/\delta\rceil}\max_{x\in D}|\nabla\hat{\Psi}_k(x)|\right\}=O_P(1)$. 

(A6): For any $T>0$,  there exists $\gamma\in (1/2,1)$,
$\left\{\max_{x\in \mR^d, 1\leq k< \lceil T/\delta\rceil}\frac{k^{\gamma}m^{1/2}}{(1+|x|^2)}|\hat{\Psi}_{k+1}(x)-\hat{\Psi}_k(x)|\right\}=O_P(1)$. 

\vspace{0.1cm}

And the function $\Psi$ and $\rho$ satisfies:

\vspace{0.1cm}

(A7): There exists $\alpha>0$, $\left|E(\sum_{i=1}^m\nabla_x\Psi_{\theta}(X(i\eta)))\right|=O(m^{-\alpha})$. 

\vspace{0.1cm}

(A8): For fixed $T>0$, we can find $a<1/2$, for any $t\in(0,T)$, 
$\int_{\mR^d}  \left|\Div[\rho(t,x)\nabla \tau(x)]\right|dx=O(t^{-a})$. 

\vspace{0.1cm}

Then $\hat{V}_{\delta,1}^m (t,x)$ converges to $V_1(t,x)$ in the sense that for any $\xi\in\cC_0^\iy(\mR_+\times\mR^d)$, 
as $m \rightarrow \infty$, $\delta \rightarrow 0$, $\delta m\ra 0$, $\delta m^{1+2\alpha}\ra\iy$, $\delta m^{1/(2-2\gamma)}>C>0$, 
\begin{equation}\label{eq-cv}
 \int_{\mR_+\times\mR^d} [ \hat{V}_{\delta,1}^m(t,x) - V_1(t,x)] \xi dtdx  \stackrel{P}{\rightarrow} 0, 
 \end{equation}
where 
$V_1(t, x)$ satisfies the following stochastic PDE, 
\begin{equation}\label{eqthm2}
  \partial_t V_1(t,x) = \Div \left( V_1 \nabla \Psi \right) +   \Div\left( \rho \nabla  \tau(x)  t^{-1}W(t)  \right) + \beta^{-1} \Delta V_1 , V_1(0,x)=0,
 \end{equation}
and there exists an a.s. finite unique solution to (\ref{eqthm2}).
If we only assume (A7),(A8), then as $m \rightarrow \infty$, $\delta \rightarrow 0$, $\delta m^{1+2\alpha}\ra\iy$, $\tilde{V}_{1}^m (t,x)$ also converges to $V_1(t,x)$ in the same manner as (\ref{eq-cv}). If we only assume (A8), this convergence also holds as $\delta m\ra\iy$.
\end{thm}




\begin{remark}
In the online setup, there are also two types of asymptotic analysis. One type is computational asymptotics, where we employ stochastic partial differential equations to model discrete iterate sequences generated from the online JKO scheme. Assumptions (A5) and (A6) are used to control the discretization error of the algorithm, that is, the difference between the solution $\tilde{\rho}_1^m$ to the `dynamic' Fokker-Planck equation (\ref{eq-ii}) and the iterate sequences $\hat{\rho}_{\delta,1}^m$ generated from algorithm (\ref{JKO-on3}). As a result, given (A5),(A6), as $\delta\ra 0, m\ra\iy, \delta m^{1/(2-2\gamma)}>C>0$, the discretization error is bounded by $O_P(\delta)$. Another type is statistical asymptotics, where we use dynamic online samples generated from several independent Langevin diffusions to estimate and update the function $\Psi$, with sample size $mk$ in the k-th step. Given (A8), we can guarantee the existence of the solution to the SPDE (\ref{eqthm2}). Then, as $\delta\ra 0, m\ra \iy, \delta m\ra\iy$, we can show that the dynamic Wasserstein gradient flow $\tilde{\rho}_1^m$ converges to the population gradient flow $\rho$ with order $(m/\delta)^{-1/2}$, that is, $\tilde{V}_1^m(t,x)$ converges to its limit process $V_1(t,x)$. If (A7) is further assumed, it will control the bias in the central limit theorem and guarantee the uniform convergence of $(mk)^{1/2}(\hat{\theta}^k-\theta)$ to the averaged Gaussion variables when $k=o(m^{1+2\alpha})$. As a result, the condition $\delta m\ra\iy$ can be weakened to $\delta m^{1+2\alpha}\ra\iy$. Combining both asymptotics, we can derive the higher-order convergence result for the iterate sequences generated from the online JKO scheme. Since the discretization error is bounded by $O_P(\delta)$, if we set $\delta m\ra 0$, then $O_P(\delta)$ is of the order smaller than $(m/\delta)^{-1/2}$, and $\hat{\rho}_{\delta,1}^m(t,x)$ has the same asymptotic distribution $V_1(t,x)$ as $\tilde{\rho}_1^m(t,x)$.
\end{remark}

\begin{remark}
(A5) is satisfied in weak conditions. For example, let $Y_k=\max_{x\in D}|\nabla\hat{\Psi}_k(x)|$. Using Markov inequality, we get
\[P\left(\delta\sum_{k=1}^{\lceil T/\delta\rceil}Y_k>M\right)\leq\frac{\delta\sum_{k=1}^{\lceil T/\delta\rceil}E(Y_k)}{M}.\]
If $E(Y_k)$ is uniformly bounded or $\frac{1}{m}\sum_{k=1}^mE(Y_k)$ is bounded, then (A5) is true. 

(A6) basically controls the difference between consecutive estimated functions $\hat{\Psi}_k$. Here we show a simple example in which (A6) is true.
Consider $\Psi_\theta(x)=\frac{1}{2}|x-\theta|^2$. In this case $\hat{\theta}^k=\frac{1}{mk}\sum_{j=1}^k\sum_{i=1}^m X^{(j)}(i\eta)$. Let $x_j=\frac{1}{m}\sum_{i=1}^m X^{(j)}(i\eta)-\theta=O_P(m^{-1/2})$, then $\hat{\theta}^k-\theta=\frac{1}{k}\sum_{j=1}^k x_j=\bar{x}_k$, and
\[\hat{\theta}^{k+1}-\hat{\theta}^k=\bar{x}_{k+1}-\bar{x}_k=\frac{x_{k+1}-\bar{x}_k}{k+1}.\]
Since both $x_{k+1}$ and $\bar{x}_k$ are $O_P(m^{-1/2})$, $\hat{\theta}^{k+1}-\hat{\theta}^k=O_P(k^{-1}m^{-1/2})$. We can take any $\gamma<1$ and (A6) is satisfied. More general cases can be found in the appendix.

(A7) is satisfied, for example, when $X_0\sim \pi(x)$. In this case $E(\nabla_x\Psi_{\theta}(X(i\eta)))=0$, since all $X(i\eta)\sim\pi$ and $E_\pi(\nabla_x\Psi_{\theta}(X))=0$. In this case $\alpha$ is arbitrary, and we do not need to assume $\delta m^{1+2\alpha}\ra\iy$. In another example, if we assume $\Psi_\theta(x)=\frac{1}{2}(x-\theta)'A(x-\theta)$ is a quadratic function, then (A7) is true (the expectation is zero) when $E(X_0)=\theta$.

(A8) mainly controls the irregularity of $\rho(t,x)$ around $t=0$ and guarantees the existence of the solution to (\ref{eqthm2}).
\end{remark}

Next we propose a variant of the above scheme and show that the resulting processes have the same asymptotic distribution. We will replace $X(i \eta)$, $i=1, \cdots, n$ in (\ref{eq-Est1}) with $X^{(k)}(i \eta)$, $i=1, \cdots, m$ to obtain the estimator $\hat{\theta}^{(k)}_m$, that is, $\hat{\theta}^{(k)}_m$ solves
$\sum_{i=1}^m\nabla_x\Psi_{\theta}(X^{(k)}(i\eta))=0$. And denote
\[\hat{\Psi}_m^{(k)}(x)=\frac{1}{k}\sum_{j=1}^k \Psi_{\hat{\theta}^{(j)}_m}(x).\]
Then we define $\hat{\rho}_{m,2}^{(k)}$ as follows:
\[
\hat{\rho}_{m,2}^{(k)}={\rm arg}\min_{\rho\in \Xi} \left\{\frac{1}{2}\left[d_W\left(\rho,\hat{\rho}_{m,2}^{(k-1)}\right)\right]^2+\delta \mF_{\hat{\Psi}_m^{(k)}}(\rho)\right\}
\ ,\hat{\rho}_{m,2}^{(0)}=\rho^0
\]
\[\hat{\rho}_{\delta,2}^m(t,x)=\hat{\rho}_{m,2}^{(\lceil t/\delta\rceil)}(x),\]
and denote by $\tilde{\rho}_2^m(t,x)$ the solution to the following Fokker-Planck equation:
\begin{equation}\label{eq-i2}
\pt_t\tilde{\rho}_{2}^{m}(t,x)=\Div(\tilde{\rho}_{2}^{m}\nabla\hat{\Psi}_m^{(\lceil t/\delta\rceil)}(x))+\beta^{-1}\Delta\tilde{\rho}_{2}^{m},\  \tilde{\rho}_{2}^{m}(0,x)=\rho^0(x).
\end{equation}
Let $\hat{V}_{\delta,2}^m(t,x)=\sqrt{m/\delta}[\hat{\rho}_{\delta,2}^{m}(t,x)-\rho(t,x)]$, $\tilde{V}_2^m(t,x)=\sqrt{m/\delta}[\tilde{\rho}_{2}^{m}(t,x)-\rho(t,x)]$. We can show that $\hat{V}_{\delta,2}^m(t,x)$ and $\tilde{V}_2^m(t,x)$ have the same limiting distribution as $\hat{V}_{\delta,1}^m(t,x)$.

\begin{thm}\label{prop-44}
Assume (A8), $\hat{\Psi}_k(x)=\hat{\Psi}_m^{(k)}(x)$ satisfies (A5),(A6), and assume

\vspace{0.1cm}
(A9): There exists $\alpha>0$, $|E(\nabla_x\Psi_{\hat{\theta}_m}(x)-\nabla_x\Psi_\theta(x))|=O(m^{-1-\alpha})$ for $x\in D$, where $D$ is a compact subset of $\mR^d$. 

\vspace{0.1cm}
Then $\hat{V}_{\delta,2}^m (t,x)$ converges to $V_1(t,x)$ in the sense that for any $\xi\in\cC_0^\iy(\mR_+\times\mR^d)$, as $m\ra\iy, \delta\ra 0, \delta m\ra 0$, $\delta m^{1+2\alpha}\ra\iy$, $\delta m^{1/(2-2\gamma)}>C>0$, 
\[ \int_{\mR_+\times\mR^d} [ \hat{V}_{\delta,2}^m(t,x) - V_1(t,x)] \xi dtdx  \stackrel{P}{\rightarrow} 0, \]
where $V_1(t, x)$ is defined in Theorem \ref{thm-2}. If we only assume (A8),(A9), then as $m \rightarrow \infty$, $\delta \rightarrow 0$, $\delta m^{1+2\alpha}\ra\iy$, $\tilde{V}_{2}^m (t,x)$ also converges to $V_1(t,x)$ in the same manner as (\ref{eq-cv}). If we only assume (A8), this convergence also holds as $\delta m\ra\iy$.
\end{thm}

\begin{remark}
The overall idea of this theorem is the same as Theorem \ref{thm-2}. The only difference is that we use a different estimator of $\Psi$. Assumption (A9) plays a similar role as (A7). Assume (A9), $\tilde{\rho}_2^m(t,x)$ will have asymptotic distribution $V_1(t,x)$ when $\delta m^{1+2\alpha}\ra\iy$. Then given all conditions, $\hat{\rho}_{\delta,2}^m(t,x)$ will have the same asymptotic distribution $V_1(t,x)$.

(A9) is satisfied, for example, when $\Psi_\theta(x)=x'A\theta-H(x)+\zeta(\theta)$, $X_0\sim\pi(x)$. Indeed, in this case $\nabla_x\Psi_\theta(x)=A\theta-h(x)$, $h(x)=\nabla H(x)$. $\hat{\theta}_m=\frac{A^{-1}}{m}\sum_{i=1}^m h(X(i\eta))$. Since $E_\pi\nabla_x\Psi_\theta(X)=0$, when $X_0\sim\pi(x)$, $X(t)\sim\pi(x)$, we have $E_\pi(h(X(t)))=A\theta$, $E(\hat{\theta}_m)=\theta$. Then $E(\nabla_x\Psi_{\hat{\theta}_m}(x)-\nabla_x\Psi_{\theta}(x))=E(A(\hat{\theta}_m-\theta))=0$. In another example, if $\Psi_\theta(x)=\frac{1}{2}(x-\theta)'A(x-\theta)$ is a quadratic function, then the expectation is zero when $E(X_0)=\theta$. 
\end{remark}

In the above two schemes, we utilize all samples before step $k$ to estimate $\hat{\Psi}_k$. Now we will try a different scheme where we only use the samples $X^{(k)}=\{X^{(k)}(\eta),\dots,X^{(k)}(m\eta)\}$ in step $k$ to estimate $\hat{\Psi}_k$ (Note that it still belongs to the first online framework). Consider $\hat{\Psi}_k=\Psi_{\hat{\theta}^{(k)}_m}$, and define $\hat{\rho}_{m,3}^{(k)}$ as follows:
\begin{equation}\label{JKO-on5}
\hat{\rho}_{m,3}^{(k)}={\rm arg}\min_{\rho\in \Xi} \left\{\frac{1}{2}\left[d_W\left(\rho,\hat{\rho}_{m,3}^{(k-1)}\right)\right]^2+\delta \mF_{\Psi_{\hat{\theta}^{(k)}_m}}(\rho)\right\} , \;\;\; \hat{\rho}_{m,3}^{(0)}=\rho^0 , 
\end{equation}
and
\begin{equation}\label{JKO-on6}
\hat{\rho}_{\delta,3}^m(t,x)=\hat{\rho}_{m,3}^{(\lceil t/\delta\, \rceil)}(x) . 
\end{equation}
We denote by $\tilde{\rho}_3^m(t,x)$ the solution to the following Fokker-Planck equation:
\begin{equation}\label{eq-i1}
\pt_t\tilde{\rho}_3^m(t,x)=\Div(\tilde{\rho}_{3}^{m}\nabla_x\Psi_{\hat{\theta}^{(\lceil t/\delta\rceil)}_m}(x))+\beta^{-1}\Delta\tilde{\rho}_{3}^{m},\  \tilde{\rho}_{3}^{m}(0,x)=\rho^0(x).
\end{equation}
Let $\tilde{V}_3^m(t,x)=\sqrt{m/\delta}[\tilde{\rho}_{3}^{m}(t,x)-\rho(t,x)]$. We can prove the following asymptotic result of $\tilde{V}_3^m(t,x)$ with weaker conditions.

\begin{prop}\label{prop-42}
Assume (A8), then $\tilde{V}_{3}^m (t,x)$ converges to $V_2(t,x)$ in the sense that for any $\xi\in\cC_0^\iy(\mR_+\times\mR^d)$, as $m\ra\iy, \delta\ra 0, \delta m\ra\iy$, 
\[ \int_{\mR_+\times\mR^d} [ \tilde{V}_{3}^m(t,x) - V_2(t,x)] \xi dtdx  \stackrel{P}{\rightarrow} 0, \]
where $V_2(t, x)$ satisfies the following stochastic PDE, 
\begin{equation}\label{eqp42}
  \partial_t V_2(t,x) = \Div \left( V_2 \nabla \Psi \right) +   \Div\left( \rho \nabla  \tau(x)  \dot{W}(t)  \right) + \beta^{-1} \Delta V_2 , V_2(0,x)=0,
  \end{equation}
$\dot{W}(t)=\frac{dW(t)}{dt}$ is the white noise, and there exists an a.s. finite unique solution to (\ref{eqp42}).
\end{prop}

\begin{remark}
Unfortunately, in this scheme it is challenging to prove the convergence of $\hat{V}_{\delta,3}^m(t,x)=\sqrt{m/\delta}[\hat{\rho}_{\delta,3}^{m}(t,x)-\rho(t,x)]$ to $V_2(t,x)$. We need (A6) to control the difference between $\hat{\Psi}_{k+1}$ and $\hat{\Psi}_k$, but in this case $\hat{\theta}^{(k+1)}_m-\hat{\theta}^{(k)}_m$ is of order $m^{-1/2}$, so we can only take $\gamma=0$ in (A6), which means we need $\delta m^{1/2}>C$. However, the discretization error in $\hat{\rho}_{\delta,3}^{m}(t,x)-\rho(t,x)$ is at least $O_P(\delta)$, so $\delta m$ must converge to zero, which is contradicted with $\delta m^{1/2}>C$.
\end{remark}

In the second online framework, we sequentially observe the Langevin diffusion and construct sequential estimators of $\theta$. We estimate $\Psi$ by $\Psi_{\hat{\theta}_k}$ in step $k$, where $\hat{\theta}_k$ is estimated by solving the following equation
\begin{equation}\label{Est-E}
\sum_{i=1}^k\nabla_x\Psi_{\theta}(X(i\eta))=0.
\end{equation}
Then we can define
\begin{equation}
\hat{\rho}_2^{(k)}={\rm arg}\min_{\rho\in \Xi} \left\{\frac{1}{2}\left[d_W\left(\rho,\hat{\rho}_2^{(k-1)}\right)\right]^2+\delta \mF_{\Psi_{\hat{\theta}_k}}(\rho)\right\}
\ ,\hat{\rho}^{(0)}=\rho^0
\end{equation}
and
\begin{equation}
\hat{\rho}_{\delta,2}(t,x)=\hat{\rho}^{(\lceil t/\delta\rceil)}(x).
\end{equation}
Its Fokker-Planck counterpart, $\tilde{\rho}_2$, is defined by the following PDE
\begin{equation}\label{eq-41}
\pt_t \tilde{\rho}_2(t,x)=\Div(\tilde{\rho}_2\nabla_x\Psi_{\hat{\theta}_\delta(t)})+\beta^{-1}\Delta \tilde{\rho}_2, \ \tilde{\rho}_2(0,x)=\rho^0(x),
\end{equation}
where $\hat{\theta}_\delta(t)=\hat{\theta}_{\lceil t/\delta\rceil}$. 
Let $\tilde{V}_2(t,x)=\delta^{-1/2}(\tilde{\rho}_2(t,x)-\rho(t,x))$. We can show that $\tilde{V}_2(t,x)$ has the same limiting distribution as $\hat{V}_{\delta,1}^m(t,x)$ and $\hat{V}_{\delta,2}^m(t,x)$.

\begin{prop}\label{prop-47}
Assume (A8), $\tilde{V}_{2}(t,x)$ converges to $V_1(t,x)$ in the sense that for any $\xi\in\cC_0^\iy(\mR_+\times\mR^d)$, as $\delta\ra 0$,
\[\int_{\mR_+\times\mR^d}[\tilde{V}_2(t,x)-V_1(t,x)]\xi dtdx\stackrel{P}{\rightarrow} 0,\]
where $V_1(t, x)$ is defined in Theorem \ref{thm-2}. 
\end{prop}

\begin{remark}
In this case, we cannot show the convergence of $\hat{V}_{\delta,2}(t,x)=\delta^{-1/2}(\hat{\rho}_{\delta,2}(t,x)-\rho(t,x))$ to $V_1(t,x)$. For the condition (A6), in the best case we can take $\gamma=1$ without the `$m^{1/2}$' term. Then 
we can show that the discretization error is $O_P(\delta^{1-\cC}|\log(\delta)|)$ ($\cC$ is some random variable bounded in probability), which is too big for the convergence.
\end{remark}

Note that, in the above cases, we always fix $\eta$ and use the central limit theorem of $\hat{\theta}$ to derive the asymptotic result. Next we will try the case when $\eta$ depends on $\delta$. Then as $\eta\ra 0$, it might be easier to control the difference between $\hat{\Psi}_{k+1}$ and $\hat{\Psi}_k$ and get conditions stronger than (A6). However, we will show that, if the samples used to estimate $\theta$ are in a fixed interval, then the resulting $\hat{\theta}$ estimated from (\ref{Est-E}) is biased and has no higher-order asymptotic distribution. We will use a counterexample to illustrate this point.

Consider the example when $\eta=\delta$, $d=1$, $\beta=1$, $\Psi_\theta(x)=\frac{1}{2}|x-\theta|^2$. Then as $\delta\ra 0$, it is straightforward to see that $\hat{\theta}_\delta(t)$ converges to $\hat{\theta}(t)$,  where
\[\hat{\theta}_\delta(t)=\frac{1}{\lceil t/\delta\rceil}\sum_{i=1}^{\lceil t/\delta\rceil}X(i\delta),\ \ \hat{\theta}(t)=\frac{1}{t}\int_0^t X(s)ds.\]
In the next proposition, we will show $\hat{\theta}(t)$ is a biased estimator of $\theta$ with non-negligible variance, and we cannot find a proper higher-order convergence result for $\hat{\theta}_\delta(t)-\hat{\theta}(t)$.      

\begin{prop}\label{prop-lst}
Assume $\eta=\delta$, $d=1$, $\beta=1$, $\Psi_\theta(x)=\frac{1}{2}|x-\theta|^2$, then $\hat{\theta}(t)=\frac{1}{t}\int_0^t X(s)ds$ is a biased estimator of $\theta$. If we assume $X_0=\theta$, then as $n\ra\iy$, $\delta=t/n$, we have
\begin{equation}\label{eq-c1}
n(\hat{\theta}_\delta(t)-\hat{\theta}(t))\stackrel{d}{\ra} N\left(0, \frac{1}{6}t+\frac{1}{4}(1-e^{-2t})\right).
\end{equation}
However, we cannot find the limiting distribution of $(t/\delta)(\hat{\theta}_\delta(t)-\hat{\theta}(t))$ for arbitrary $\delta\ra 0$. Neither can we find the limiting process of $\hat{\theta}_\delta(t)-\hat{\theta}(t)$.
\end{prop}

Since the equation (\ref{eq-41}) is based on $\hat{\theta}_\delta(t)$, the failure in finding the asymptotic distribution of $\hat{\theta}_\delta(t)$ indicates that we cannot establish any higher-order asymptotic result for $\tilde{\rho}_2(t,x)$ or $\hat{\rho}_{\delta,2}(t,x)$ when $\eta=\delta$. 

\section{Application on Bures-Wasserstein gradient flow}\label{SEC-APP}

In this section, we restrict the probability distribution in the JKO scheme (\ref{JKO-1}) to Gaussian distribution. The space of the non-degenerate Gaussian distribution on $\mR^d$ equipped with the Wasserstein distance forms the Bures-Wasserstein space $\textbf{BW}(\mR^d)$. We may define the JKO scheme on this subspace. Let
\begin{equation} \label{JKO-BW1}
 p^{(k+1)}_\delta = \underset{p \in \textbf{BW}(\mR^d)}{\arg \min} \left\{ \frac{1}{2} \left[ d_W \left(p, p^{(k)}_\delta  \right)\right]^2 + \delta \mF_\Psi(p) \right\}  , p^{(0)}_\delta=p^0\sim N(\mu^0,\Sigma^0),
\end{equation}
\[p_\delta(t,x)=p^{(\lfloor t/\delta\rfloor)}_\delta(x).\]
Let $\mu_\delta(t)$ and $\Sigma_\delta(t)$ be the mean and covariance matrix of $p_\delta(t,\cdot)$, respectively. Lambert et al. \cite{lambert2022variational} proved that $\mu_\delta(t)$ and $\Sigma_\delta(t)$ will converge to the solution of the following ODE as $\delta$ goes to zero (with initial condition $\mu_0=\mu^0, \Sigma_0=\Sigma^0$):
\begin{equation}\label{eqe1}
\begin{aligned}
\dot{\mu}_t&=-\mE_{p_t}(\nabla\Psi(X)) \\
\dot{\Sigma}_t&=2I-\mE_{p_t}(\nabla\Psi(X)(X-\mu_t)'+(X-\mu_t)\nabla '\Psi(X)),
\end{aligned}
\end{equation}
where $p_t\sim N(\mu_t, \Sigma_t)$. Indeed, the objective function in (\ref{JKO-BW1}) has explicit expression, and we can solve the optimization problem (\ref{JKO-BW1}) by taking partial derivative w.r.t. $\mu$ and $\Sigma$, then deduce the limit (\ref{eqe1}). See also \cite{bhatia2019bures}. For more studies in variational Gaussian approximation, please refer to \cite{lambert2022continuous, lambert2022recursive, lambert2023limited} 

\vspace{0.1cm}
Now, we consider the case when $\Psi(x)$ is estimated offline by $n$ samples, that is, $\hat{\Psi}_n(x)=\Psi_{\hat{\theta}_n}(x)$. 
Define $\mu_\delta^n(t)$ and $\Sigma_\delta^n(t)$ by replacing $\Psi(x)$ with $\hat{\Psi}_n(x)$ in (\ref{JKO-BW1}), and define $\mu_t^n$ and $\Sigma_t^n$ by replacing $\Psi(x)$ with $\hat{\Psi}_n(x)$ in (\ref{eqe1}):
\begin{equation}\label{eqe6}
\begin{aligned}
\dot{\mu}_t^n&=-\mE_{p_t^n}(\nabla\hat{\Psi}_n(X)) \\
\dot{\Sigma}_t^n&=2I-\mE_{p_t^n}(\nabla\hat{\Psi}_n(X)(X-\mu_t^n)'+(X-\mu_t^n)\nabla '\hat{\Psi}_n(X))
\end{aligned}
\end{equation}
Here for simplicity, we omit the hat symbol $\hat{\cdot}$ over $\mu,\Sigma,V,\dots$ when they are calculated from estimated data in this section. Let $V_{\mu,\delta}^n(t)=\sqrt{n}(\mu_\delta^n(t)-\mu(t)), V_{\Sigma,\delta}^n(t)=\sqrt{n}(\Sigma_\delta^n(t)-\Sigma(t))$, $V_\mu^n(t)=\sqrt{n}(\mu_t^n-\mu_t), V_\Sigma^n(t)=\sqrt{n}(\Sigma_t^n-\Sigma_t)$. Then we have the following theorem.

\begin{thm}\label{thm-s6}
Assume $\frac{|\nabla\hat{\Psi}_n(x)|}{1+|x|^p}=O_P(1)$ for some $p>0$, and for fixed $T>0$, $\Sigma_t$ is non-degenerate over the interval $[0,T]$. Then as $n\ra\iy$, $V_\mu^n(t)$ and $V_\Sigma^n(t)$ weakly converge to $V_\mu(t)$ and $V_\Sigma(t)$, respectively, where $V_\mu(t)$ and $V_\Sigma(t)$ satisfy the linear ODE set:
\begin{equation}\label{eqe11}
\left\{
\begin{aligned}
\dot V_\mu(t)=&-\int_{\mR^d} \nabla\tau(x) \bZ p_t(x)dx-\int_{\mR^d}\nabla\Psi(x)\left[\nabla_\mu p_t(x)\cdot V_\mu(t)+\nabla_\Sigma p_t(x)\cdot V_\Sigma(t)\right]dx \\
\dot V_\Sigma(t)=&-\int_{\mR^d} [\nabla\tau(x) \bZ(x-\mu_t)'-\nabla\Psi(x)V_\mu(t)']p_t(x)dx- \\
&\int_{\mR^d}\nabla\Psi(x)(x-\mu_t)'[\nabla_\mu p_t(x)\cdot V_\mu(t)+\nabla_\Sigma p_t(x)\cdot V_\Sigma(t)]dx- \\
&\int_{\mR^d} [(x-\mu_t)(\nabla\tau(x) \bZ)'-V_\mu(t)\nabla'\Psi(x)]p_t(x)dx- \\
&\int_{\mR^d}(x-\mu_t)\nabla'\Psi(x)[\nabla_\mu p_t(x)\cdot V_\mu(t)+\nabla_\Sigma p_t(x)\cdot V_\Sigma(t)]dx
\end{aligned}
\right.
\end{equation}
When $\delta\ra 0$ and $\delta\sqrt{n}\ra 0$, $V_{\mu,\delta}^n(t)$ and $V_{\Sigma,\delta}^n(t)$ converge to the same limit $V_\mu(t)$ and $V_\Sigma(t)$.
\end{thm}

\begin{remark}
When we restrict the probability space on the Bures-Wasserstein space, the iterate sequences of the JKO scheme are represented by the sequences of mean and variance of the Gaussian distribution. In the offline case, the function $\Psi$ is estimated using samples generating from the Langevin diffusion with size $n$. Given the condition on $\nabla\hat{\Psi}_n(x)$ and the non-degeneration of the covariance matrix, we can show the Lipschitz property of $p(x,\mu,\Sigma)$ and $\mE_p(\nabla\hat{\Psi}_n(X))$ with respect to $\mu$ and $\Sigma$, where the Lipschitz constant is bounded in probability. Then we can derive the asymptotic distribution of $\mu_t^n$ and $\Sigma_t^n$ with order $n^{-1/2}$, which is the statistical asymptotics. For the original iterate sequences $\mu_\delta^n(t)$ and $\Sigma_\delta^n(t)$, we use the continuous differential equations (\ref{eqe6}) to model the discrete sequences, and the discretization error is bounded by $O_P(\delta)$. If we require $\delta\sqrt{n}\ra 0$, then this error is neglibible in the higher-order convergence. Thus $\mu_\delta^n(t)$ and $\Sigma_\delta^n(t)$ have the same asymptotic distributions as $\mu_t^n$ and $\Sigma_t^n$, respectively.
\end{remark}

Next, we study the case when $\Psi(x)$ is estimated online. Here we will use $\hat{\Psi}_k(x)=\Psi_{\hat{\theta}^k}(x)$. (We should distinguish it from $\hat{\Psi}_n$ in the offline case). Define $\mu_{\delta,1}^{m}(t)$ and $\Sigma_{\delta,1}^{m}(t)$ by replacing $\Psi(x)$ with $\hat{\Psi}_k(x)$ in (\ref{JKO-BW1}), and define $\mu_1^{m}(t)$ and $\Sigma_1^{m}(t)$ below:
\begin{equation}\label{eqe12}
\begin{aligned}
\dot{\mu}_1^{m}(t)&=-\mE_{p_1^{m}(t)}(\nabla\hat{\Psi}_{\lceil t/\delta\rceil}(X)) \\
\dot{\Sigma}_1^{m}(t)&=2I-\mE_{p_1^{m}(t)}(\nabla\hat{\Psi}_{\lceil t/\delta\rceil}(X)(X-\mu_1^{m}(t))'+(X-\mu_1^{m}(t))\nabla '\hat{\Psi}_{\lceil t/\delta\rceil}(X))
\end{aligned}
\end{equation}
Let $V_{\mu,\delta,1}^{m}(t)=\sqrt{m/\delta}(\mu_{\delta,1}^{m}(t)-\mu(t)), V_{\Sigma,\delta,1}^{m}(t)=\sqrt{m/\delta}(\Sigma_{\delta,1}^{m}(t)-\Sigma(t))$, $V_{\mu,1}^{m}(t)=\sqrt{m/\delta}(\mu_1^{m}(t)-\mu(t)), V_{\Sigma,1}^{m}(t)=\sqrt{m/\delta}(\Sigma_1^{m}(t)-\Sigma(t))$. Then we have the following theorem.

\begin{thm}\label{thm-s7}
Assume for fixed $T>0$, $\Sigma_t$ is non-degenerate over the interval $[0,T]$, and $\hat{\Psi}_k=\Psi_{\hat{\theta}^k}$ is defined by (\ref{eq-Est2}). Assume (A7), $\frac{|\nabla\hat{\Psi}_k(x)|}{1+|x|^p}=O_P(1)$ for some $p>0$, and there exists $a<1/2$, for $t<T$, $\int_{\mR^d} p_t(x)|\nabla\tau(x)|dx=O(t^{-a})$. Then as $m\ra\iy$, $\delta\ra 0$, $\delta m^{1+2\alpha}\ra\iy$, $V_{\mu,1}^{m}(t)$ and $V_{\Sigma,1}^{m}(t)$ weakly converge to $V_{\mu,1}(t)$ and $V_{\Sigma,1}(t)$, respectively, where $V_{\mu,1}(t)$ and $V_{\Sigma,1}(t)$ satisfy the linear SDE set:
\begin{equation}\label{eqe16}
\left\{
\begin{aligned}
\dot V_{\mu,1}(t)=&-\int_{\mR^d} \nabla\tau(x) t^{-1}W_t p_t(x)dx-\int_{\mR^d}\nabla\Psi(x)\left[\nabla_\mu p_t(x)\cdot V_{\mu,1}(t)+\nabla_\Sigma p_t(x)\cdot V_{\Sigma,1}(t)\right]dx \\
\dot V_{\Sigma,1}(t)=&-\int_{\mR^d} [\nabla\tau(x) t^{-1}W_t(x-\mu_t)'-\nabla\Psi(x)V_{\mu,1}(t)']p_t(x)dx- \\
&\int_{\mR^d}\nabla\Psi(x)(x-\mu_t)'[\nabla_\mu p_t(x)\cdot V_{\mu,1}(t)+\nabla_\Sigma p_t(x)\cdot V_{\Sigma,1}(t)]dx- \\
&\int_{\mR^d} [(x-\mu_t)(\nabla\tau(x) t^{-1}W_t)'-V_{\mu,1}(t)\nabla'\Psi(x)]p_t(x)dx- \\
&\int_{\mR^d}(x-\mu_t)\nabla'\Psi(x)[\nabla_\mu p_t(x)\cdot V_{\mu,1}(t)+\nabla_\Sigma p_t(x)\cdot V_{\Sigma,1}(t)]dx
\end{aligned}
\right.
\end{equation}
If we further assume $\delta m\ra 0$, then $V_{\mu,\delta,1}^{m}(t)$ and $V_{\Sigma,\delta,1}^{m}(t)$ converge to the same limit $V_{\mu,1}(t)$ and $V_{\Sigma,1}(t)$.
\end{thm}

\begin{remark}
In the online case, the function $\Psi$ is estimated using samples generated from the Langevin diffusion with size $mk$, and the setup is the same as in Theorem \ref{thm-2}.  Given (A7) and the regularity condition, as $\delta m^{1+2\alpha}\ra\iy$, we can derive the asymptotic distribution of $\mu_1^m(t)$ and $\Sigma_1^m(t)$ with order $(m/\delta)^{-1/2}$. We then use the differential equations (\ref{eqe12}) to model the iterate discrete sequences generated from the algorithm, and the discretization error is still bounded by $O_P(\delta)$ (Here we do not need assumptions (A5) or (A6) due to the properties of Gaussian distribution). If $\delta m\ra 0$, then this error is neglibible and $\mu_{\delta,1}^m(t)$ and $\Sigma_{\delta,1}^m(t)$ have the same asymptotic distributions as $\mu_1^m(t)$ and $\Sigma_1^m(t)$, respectively.
\end{remark}

\section{Numerical studies}\label{SEC-Num}

In this section, we conduct some numerical works to illustrate our theory. To do the simulation, we will encounter two tasks. The first is to numerially solve the JKO scheme, and the second is to simulate the stochastic PDE. For the first task, we introduce a method which is easy to implement in one dimension. Assume now we want to optimize the functional $\frac{1}{2}d_W^2(\rho^o,\rho)+\mF_{\Psi}(\rho)$ when $\rho^o$ is given. We will construct a sequence of probability densities such that the sequence will converge to the minimum $\rho^*$. Consider the flux $\Phi_\tau:\mR^d\ra\mR^d$ defined by
\[\pt_\tau\Phi_\tau(y)=\xi(\Phi_\tau(y)), \Phi_0(y)=y.\]
For any $\tau$, let the measure $\rho_\tau(y)dy$ be the push forward of $\rho^s(y)dy$ under $\Phi_\tau$, where $\rho^s$ is the current candidate density (not necessarily be $\rho^*$). Let $P$ be optimal in the definition of $d_W^2(\rho^o,\rho^s)$ with density $p(x,y)$, then for small $\tau$, using approximation $d_W^2(\rho^o,\rho_\tau)-d_W^2(\rho^o,\rho^s)\approx \mE_P(|\Phi_\tau(Y)-X|^2-|Y-X|^2)$, we obtain
\[\begin{split}
&\frac{\pt}{\pt\tau}\left.\left(\frac{1}{2}d_W^2(\rho^o,\rho_\tau)+\mF_{\Psi}(\rho_\tau)\right)\right|_{\tau=0}\\
\approx&\int_{\mR^d\times\mR^d}(y-x)\cdot\xi(y)P(dxdy)+\delta\int_{\mR^d}(\nabla\Psi(y)\cdot\xi(y)-\beta^{-1}\Div\xi(y))\rho^s(y)dy\\
=&\int_{\mR^d}\xi(y)\cdot\left[\int_{\mR^d}(y-x)p(x,y)dx+\delta(\nabla\Psi(y)\rho^s(y)+\beta^{-1}\nabla\rho^s(y))\right]dy.
\end{split}\]
Denote $\alpha(y)=\int_{\mR^d}(y-x)p(x,y)dx+\delta(\nabla\Psi(y)\rho^s(y)+\beta^{-1}\nabla\rho^s(y))$, then $\alpha(y)\equiv0$ if $\rho^s=\rho^*$ (see Theorem 5.1 in \cite{jordan1998variational}). When $\rho^s$ is not the minimum point of $\frac{1}{2}d_W^2(\rho^o,\rho^s)+\mF_{\Psi}(\rho^s)$, we will update it with the push forward of $\rho^s$ under $\Phi_\tau$ for some $\xi$ and $\tau$. Inspired by the gradient descent method, we choose $\xi$ to minimize the gradient $\int_{\mR^d}\xi(y)\cdot\alpha(y)dy$ under the constraint $\|\xi\|_{\cL^2}=1$---that is, $\xi^*=-\alpha/\|\alpha\|_{\cL^2}$. Then after several updates with $\tau\ra 0$, $\rho^s$ will converge to the global minimum $\rho^*$ of the convex function $\frac{1}{2}d_W^2(\rho^o,\rho)+\mF_{\Psi}(\rho)$. To calculate $\rho_\tau$, note that
\[\rho_\tau(\Phi_\tau(y))|\det\nabla_y\Phi_\tau(y)|=\rho^s(y),\]
for small $\tau$, using approximation $\Phi_\tau(y)\approx y+\tau\xi(y)$, we get
\[\rho_\tau(y+\tau\xi(y))\approx|\det[I+\tau\nabla\xi(y)]|^{-1}\rho^s(y),\]
then using the proper interpolation method to calculate $\rho_\tau(y)$, for example, when $d=1$, we have
\begin{equation}\label{eq-nu1}
\begin{split}
\rho_\tau(y+\tau\xi(y))&\approx\rho_\tau(y)+\tau\xi(y)\pt_y\rho_\tau(y)\\
&\approx\rho_\tau(y)+\tau\xi(y)\frac{\rho_\tau(y+h+\tau\xi(y+h))-\rho_\tau(y-h+\tau\xi(y-h))}{2h+\tau(\xi(y+h)-\xi(y-h))}
\end{split}
\end{equation}
Now we treat the specific example. Consider the case when $d=1, \Psi_\theta(x)=\frac{1}{2}|x-\theta|^2, \rho^0\sim N(\mu_0,\sigma_0^2)$. We partition $[-D,D]$ into $J$ intervals, let $h=2D/J$, then we simulate the function value at the $J+1$ discrete points $-D,-D+h,-D+2h,\dots,D-h,D$. 
Assume $\rho(-D)=\rho(D)=0$, $\{\tau_l\}$ is a proper sequence of step size, then we can calculate $\rho^{(k)}$ from $\rho^{(k-1)}$ using the following scheme:
\begin{enumerate}
    \item Let $\rho^{(k,0)}=\rho^{(k-1)}$, $l=0$;
    \item Denote $\rho^o=\rho^{(k-1)}, \rho^s=\rho^{(k,l)}$, find the optimum joint density in the definition of $d_W^2(\rho^o,\rho^s)$, denote by $p(x,y)$ for simplicity, and calculate 
    \[\alpha(y)\approx\int_{\mR^d}(y-x)p(x,y)dx+\delta\left((y-\theta)\rho^s(y)+\beta^{-1}\frac{\rho^s(y+h)-\rho^s(y-h)}{2h}\right),\]
    where $\alpha(-D)=\alpha(D)=0$. If $\|\alpha\|_{\cL^1}<\kappa$, go to step 3. Otherwise, let $\xi=-\alpha/\|\alpha\|_{\cL^2}$, $\rho_\tau(y+\tau\xi(y))=\left[1+\tau_{l+1}\frac{\xi(y+h)-\xi(y-h)}{2h}\right]^{-1}\rho^s(y)$, then calculate $\rho_\tau(y)$ for $y=jh-D$ using (\ref{eq-nu1}), let $\rho^{(k,l+1)}=\rho_\tau/\|\rho_\tau\|_{\cL^1}$, $l=l+1$, and return to the beginning of step 2.
    \item Let $\rho^{(k)}=\rho^s$.
\end{enumerate}

In step 2, we need to find the optimum joint density $p(x,y)$. In fact, this is a linear programming problem, and the variables are $p_{ij}=p(ih-D,jh-D), i,j=1,\dots,J-1$, which should satisfy $p_{ij}\geq 0, \sum_{j=1}^{J-1}p_{ij}=\rho^o[i]/h, \sum_{i=1}^{J-1}p_{ij}=\rho^s[j]/h$. The objective is to minimize $D_2(p)=\sum_{i,j=1}^{J-1}(i-j)^2p_{ij}$, since $E_P|X-Y|^2\approx h^4D_2(p)$. It is quite time consuming to directly apply linear programming to the task. To simplify the problem, note that $\rho^o$ and $\rho^s$ are close in distribution, and $p_{ij}$ should be zero when $|i-j|$ is large. Indeed, for $\beta=1,\delta=0.01, D=5, J=200, h=0.05$, after we examine several examples for enough steps, we find that $p_{ij}>0$ only when $|i-j|\leq 1$; therefore, we can largely simplify the linear programming problem and reduce the time complexity.

The step size sequence can be chosen as $\tau_l=\tau/l$ or $\tau_l=\tau/\log(l+1)$, and $\tau$ should be small to keep the approximation scheme stable. In practice, when the initial distribution has a small variance, for example, $\sigma_0^2=0.01$, the convergence from $\rho^{(k,l)}$ to $\rho^{(k)}$ is relatively slow. We get inspiration from Nesterov's acceleration and modify the scheme by letting $\rho^s=\rho^{(k,l)}+\frac{l-1}{l+2}(\rho^{(k,l)}-\rho^{(k,l-1)})$ for $l>0$ in the beginning of step 2. After applying this technique, we get a faster convergence speed than in the original scheme. This exhibits the attraction of Nesterov's acceleration in optimization problems. 
For different approaches to approximate the JKO scheme or other numerical works related to the scheme, see, for example, \cite{benamou2016augmented, benamou2016discretization, alvarez2021optimizing, fan2021variational, mokrov2021large, bonet2022efficient, carrillo2022primal}.

For the second task, we partition $[0,T]$ into $I$ intervals, let $\nu=T/I$, then we simulate function values of $V_1(t,x)$ at $t=i\nu$, $x=jh-D$, assuming $V_1(t,x)=0$ for $|x|\geq D$. Here we use an example to illustrate our method: assume $d=1$, $\beta=1$, $\Psi_\theta(x)=\frac{1}{2}|x-\theta|^2$, $\nabla_x\Psi_\theta(x)=x-\theta$, 
$\tau(x)=-\gamma_\theta(x-\theta) $, $\nabla\tau(x)=-\gamma_\theta$. Then $\Div(V_1\nabla\Psi)=\pt_x V_1\cdot (x-\theta)+V_1$, $\Div(\rho\nabla\tau(x)t^{-1}W_t)=-\gamma_\theta\pt_x\rho t^{-1}W_t$, and the forward difference formula of $V_1(t,x)$ is
\begin{equation}\label{eq-nu2}
\begin{split}
V_1^{i+1,j}-V_1^{i,j}=&\left[\frac{V_1^{i,j+1}-V_1^{i,j-1}}{2h}(x-\theta)+V_1^{i,j}+\frac{V_1^{i,j+1}-2V_1^{i,j}+V_1^{i,j-1}}{h^2}\right]\nu-\\
&\gamma_\theta\pt_x\rho\frac{W_{(i+1)\nu}}{i+1},
\end{split}
\end{equation}
where $x=jh-D, t=i\nu$, $V_1^{0,j}=0$, $V_1^{i,0}=V_1^{i,J}=0$, $\rho(t,x)$ is the density function of $X_t$, which follows normal distribution with mean $\mu_t=\theta+(\mu_0-\theta)e^{-t}$ and variance $\sigma_t^2=1-(1-\sigma_0^2)e^{-2t}$, $W_t$ are sample paths of the standard Brownian motion. However, the forward difference formula is unstable when $\nu/h^2>1/2$, so we apply the Crank-Nicolson method \cite{crank1947practical} instead, which uses the average of $V_1^{i,j}$ and $V_1^{i+1,j}$ instead of $V_1^{i,j}$ on the right hand side of (\ref{eq-nu2}). Since (\ref{eq-nu2}) is linear, we can rewrite it as
\[V_1^{i+1}-V_1^i=B_iV_1^i+g_i\]
where $V_1^i=(V_1^{i,1},\dots,V_1^{i,J-1})'$, $B_i\in\mR^{(J-1)\times(J-1)}$ and $g_i=-\gamma_\theta\pt_x\rho\frac{W_{(i+1)\nu}}{i+1}\in\mR^{J-1}$ (let $x=jh-D$). Then the Crank-Nicolson formula is
\[V_1^{i+1}-V_1^i=B_i(V_1^i+V_1^{i+1})/2+g_i,\]
that is
\[(I-B_i/2)V_1^{i+1}=(I+B_i/2)V_1^i+g_i,\]
which can be solved in $O(J)$ since $B_i$ is nearly diagonal. 

Now we show the results for specific examples. Consider the case when $d=1,\beta=1,\delta=0.01,T=0.5,D=5,I=50,J=200$ (we take $I=50$ such that $\nu=T/I=0.01=\delta$). Let the step size be $\tau_l=0.001/\log(1+l)$, we run until $l>2000$ or $\|\alpha\|_{\cL^1}<\kappa=0.0001$. We estimate $\hat{\theta}^{k}$ from $mk$ independent sequences, where $m=10,\eta=1$. For $\Psi_\theta(x)=\frac{1}{2}|x-\theta|^2,\theta=0,\rho^0\sim N(0,1.44)$, we use our method to simulate the iterate scheme $\hat{\rho}_{m,1}^{(k)}$ and compare 
$\hat{\rho}_{\delta,1}^m(T,x)$ with $\rho(T,x)$ for $x=jh-D$, then we obtain Figure \ref{fig1}. 
We see that $\hat{\rho}_{\delta,1}^m(T,x)$ and $\rho(T,x)$ are very close. Next we compare $\hat{V}_{\delta,1}^m(t,x)=(m/\delta)^{1/2}(\hat{\rho}_{\delta,1}^m(t,x)-\rho(t,x))$ and $V_1(t,x)$. To realize $\hat{\theta}^k$ and $W_t$ on the same probability space, we use the limiting distribution
\[\sqrt{mk}(\hat{\theta}^k-\theta)-\sqrt{k}\gamma_\theta\bar{\bZ}_k\ra 0\]
and $\sqrt{\delta}\sum_{i=1}^k Z_i\approx W(k\delta)$, and we get the approximation $(m/\delta)^{1/2}(\hat{\theta}^k-\theta)\approx \gamma_\theta\frac{W(k\delta)}{k\delta}$. Then we will simply use $g_i=-\pt_x\rho (m\delta)^{1/2}(\hat{\theta}^{i+1}-\theta)$ instead of $-\gamma_\theta\pt_x\rho\frac{W_{(i+1)\nu}}{i+1}$ in the simulation formula (\ref{eq-nu2}). Then using the Crank-Nicolson formula to simulate $V_1$, we will show $\hat{V}_{\delta,1}^m(T,x)$ and $V_1(T,x)$ in Figure \ref{fig2}. To compare $\hat{V}_{\delta,1}^m(t,x)$ and $V_1(t,x)$ as a bivariate function of $(t,x)$, we draw the contour plot of both functions in Figure \ref{fig3}. The figure shows that the two processes are close to each other, which supports our asymptotic theory.
\begin{figure}[htbp]
    \centering
    \includegraphics[width=0.99\textwidth]{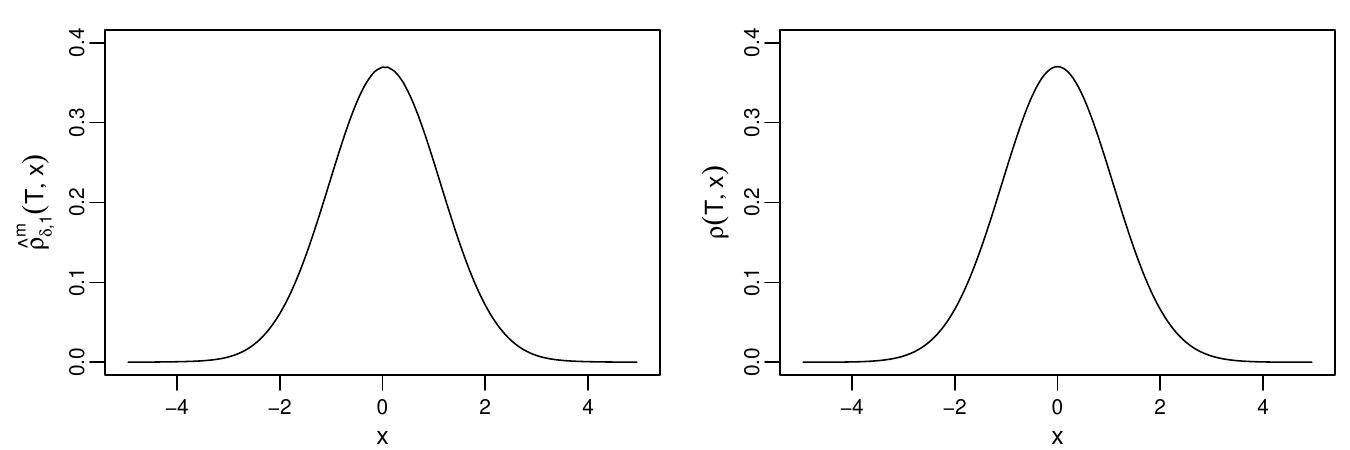}
    \caption{Function $\hat{\rho}_{\delta,1}^m(T,x)$ and $\rho(T,x)$}
    \label{fig1}
\end{figure}
\vspace{-0.2cm}
\begin{figure}[htbp]
    \centering
    \includegraphics[width=0.99\textwidth]{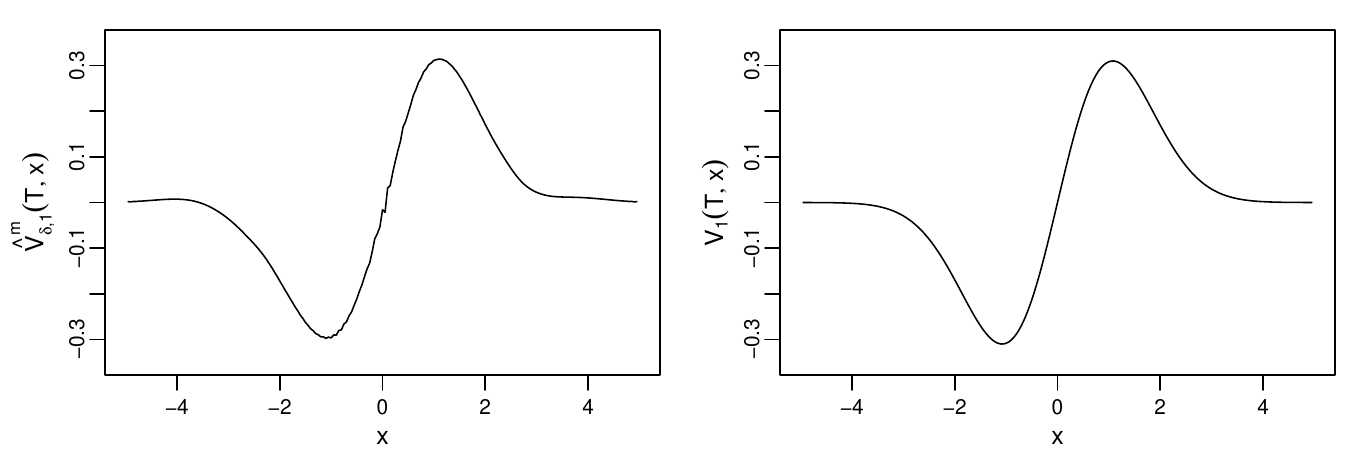}
    \caption{Function $\hat{V}_{\delta,1}^m(T,x)$ and $V_1(T,x)$}
    \label{fig2}
\end{figure}
\vspace{-0.2cm}
\begin{figure}[htbp]
    \centering
    \includegraphics[width=0.99\textwidth]{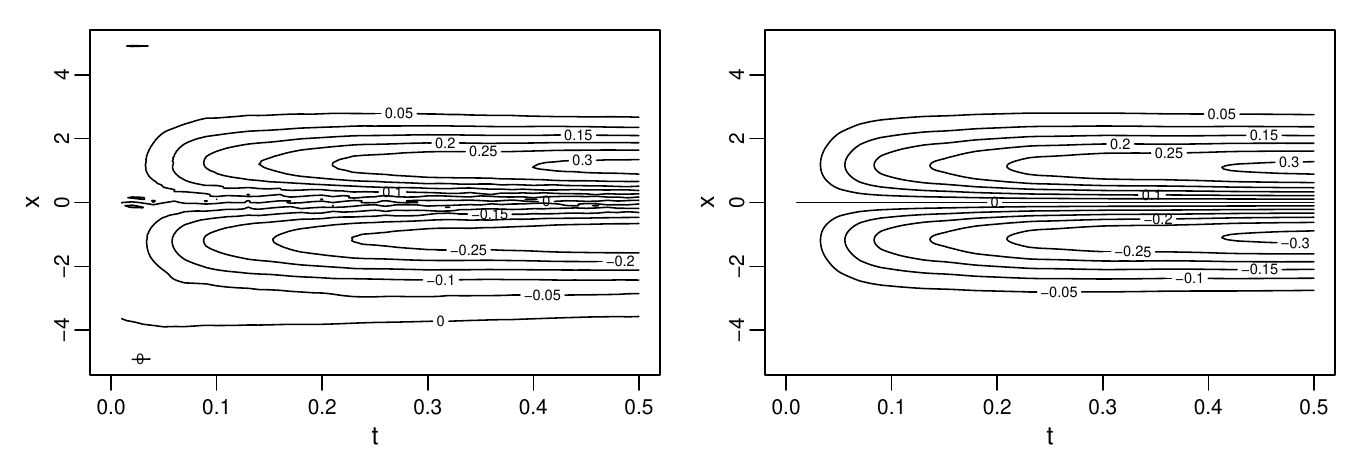}
    \caption{Contour plot of $\hat{V}_{\delta,1}^m(t,x)$ (left) and $V_1(t,x)$ (right)}
    \label{fig3}
\end{figure}

\section{Conclusion}

In this paper, we study the asymptotic property of the JKO scheme when the potential function is estimated from data. We study both offline and online estimators and propose several different frameworks. In the offline case, both the step function of the density derived by the JKO scheme and the Fokker-Planck counterpart of the JKO scheme converge to the original Fokker-Planck equation with speed $O_P(n^{-1/2})$, as well as establish the equation of the limiting distribution.

In the online case, the potentional function is sequentially estimated, and the estimators will change in the evolution of the JKO scheme. We introduce two online frameworks. In the first framework, when we estimate the potentional function using cumulative observations, we establish the higher-order convergence result for both the step function of the density derived by the JKO scheme and the Fokker-Planck counterpart to the true density function with speed $O_P((m/\delta)^{-1/2})$, as well as find the limiting distribution $V_1(t,x)$; when we estimate the potentional function only using current observations, the convergence is true only for the Fokker-Planck counterpart. In the second framework, we get the higher-order convergence result for the Fokker-Planck counterpart, then we find some negative results where the JKO scheme does not converge to the true density function and the higher-order convergence result does not exist if we let $\eta$ depend on $\delta$.


\begin{funding}
Shang Wu's research was partly supported by NSFC grant 12331009.
Yazhen Wang's research was partly supported by NSF grant DMS-1913149.
\end{funding}

\begin{supplement}
\stitle{Supplementary Materials for ``Computational and Statistical Asymptotic Analysis of the JKO Scheme for Iterative Algorithms to Update Distributions''}
\sdescription{The Supplementary Materials include the proofs of the theorems and propositions in Sections 4, 5, and 6.}
\end{supplement}

\begin{appendix}
\section{Sketch of the Proof in Section 4}
\label{appA} 
To study the asymptotic property of $\hat{\rho}_\delta^n(t,x)$, first we need to examine the asymptotic property of $\hat{\theta}_n$. Note that the estimator $\hat{\theta}_n$  satisfies (\ref{eq-Est1}).
We can establish the following asymptotic normality for $\hat{\theta}_n$. 

\begin{lem}\label{prop-th}
Assume $\Psi_\theta(x)$ satisfies (A1), let $f_\theta(x)=\nabla_x\Psi_\theta(x)$ satisfy (A2), (A3), (A4), then we have
\begin{equation} \label{asymp-theta} 
\sqrt{n}(\hat{\theta}_n-\theta)\stackrel{d}{\rightarrow} \gamma_\theta \, \bZ , 
\end{equation} 
where $\bZ$ denotes a standard normal vector, $\gamma_\theta=\Gamma_\theta^{1/2}$, 
$\Gamma_\theta=A^{-1}\Sigma_f(A')^{-1}$, $A=E_\pi(\nabla_\theta'f_\theta(X))$, $\Sigma_f={\rm Var}_\pi(f(X(0)))+2\sum_{i=1}^\infty\cov_\pi(f(X(0),f(X(i\eta)))$, and the variance 
${\rm Var}_\pi$ and covariance $\cov_\pi$ are evaluated under the situation that $X(0)$ follows the invariant distribution $\pi$ and the Langevin diffusion 
$X(t)$ evolves according to (\ref{Langevin}). 
\end{lem}
The proof of the lemma can be found in the supplementary material.
Let $\tau(x)=\nabla_\theta'\Psi_\theta(x)\gamma_\theta$. Note that $\Psi(x)=\Psi_\theta(x)$, and $\hat{\Psi}_n(x)=\Psi_{\hat{\theta}_n}(x)$. Using Lemma \ref{prop-th} and applying the delta method we obtain  
\[\sqrt{n}(\hat{\Psi}_n(x)-\Psi(x))\stackrel{d}{\rightarrow}\tau(x)\bZ,
\quad 
 \sqrt{n} [\nabla \hat{\Psi}_n(x) - \nabla\Psi(x)] \stackrel{d}{\rightarrow}  \nabla\tau(x) \bZ.
\]
For a compact set $D \subset \mR^d$, using the Skorokhod representation theorem we may have 
 \[   \sup_{x \in D} | \nabla \hat{\Psi}_n(x) - \nabla \Psi(x)  - n^{-1/2} \nabla\tau(x) \bZ  | =O_P(n^{-1} ) = o_P(n^{-1/2}). 
 \]
Finally, we can give the sketch of the proof of Theorem 4.1.

\subsection{Proof of Theorem 4.1}
\label{appA3}
Assume w.l.o.g. $\beta=1$. For any $\zeta\in \cC_0^\iy(\mR_+\times\mR^d)$, assume the support of $\zeta$ is a subset of $[0,T]\times D$, where $T>1$, $D\subset\mR^d$ is a compact set. Let $N_\delta=\lfloor T/\delta\rfloor+1$, then by the proof of Theorem 5.1 in \cite{jordan1998variational} (see also the supplementary material), we get
\begin{equation}\label{eq-32}
\begin{split}
&\left|\sum_{k=1}^{N_\delta}\int_{\mR^d}\left\{(\hat{\rho}_n^{(k)}-\hat{\rho}_n^{(k-1)})\zeta(t_{k-1},x)+\delta(\nabla\hat{\Psi}_n\cdot\nabla\zeta(t_{k-1},x)-\Delta\zeta(t_{k-1},x))\hat{\rho}_n^{(k)}\right\}dx\right|\\
\leq\ &\frac{1}{2}\sup_{t,x}|\nabla^2\zeta|\sum_{k=1}^{N_\delta}d_W^2(\hat{\rho}_n^{(k-1)},\hat{\rho}_n^{(k)}).
\end{split}
\end{equation}
For fix $x$, we have 
\begin{equation}\label{eq-33}
\sum_{k=1}^{N_\delta}(\hat{\rho}_n^{(k)}-\hat{\rho}_n^{(k-1)})\zeta(t_{k-1},x)
=-\zeta(0,x)\hat{\rho}_n^{(0)}-\int_0^{\iy}\pt_t\zeta(t,x)\hat{\rho}^n_\delta(t,x)dt.
\end{equation}
Since all derivatives of $\zeta$ are continuous with compact support, and $\max_{x\in D}|\nabla\hat{\Psi}_n(x)|=O_P(1)$, 
\begin{equation}\label{eq-34}
\begin{split}
&\left|\int_{\mR^d}\left[\sum_{k=1}^{N_\delta}\left[\delta(\nabla\hat{\Psi}_n\cdot\nabla\zeta(t_{k-1},x)-\Delta\zeta(t_{k-1},x))\hat{\rho}_n^{(k)}\right]-\int_0^{\iy}(\nabla\hat{\Psi}_n\cdot\nabla\zeta-\Delta\zeta)\hat{\rho}_\delta^n dt\right]dx\right|\\
\leq&\int_{\mR^d}\sum_{k=1}^{N_\delta}\int_{t_{k-1}}^{t_k} C\delta\left(\max_{x\in D}|\nabla\hat{\Psi}_n(x)|+1\right)\hat{\rho}_n^{(k)}dtdx
=O_P(\delta).
\end{split}
\end{equation}
Using the similar argument in \cite{jordan1998variational}, we have
\begin{equation}\label{eq-35}
\frac{1}{2}\sum_{k=1}^{N}d_W^2(\hat{\rho}_n^{(k-1)},\hat{\rho}_n^{(k)})
\leq\delta\left(|\mF_{\hat{\Psi}_n}(\hat{\rho}_n^{(0)})|+C(\cM(\hat{\rho}_n^{(N)})+1)^\alpha\right) 
\end{equation}
and $\cM(\hat{\rho}_n^{(N)})\leq 2d_W^2(\hat{\rho}_n^{(0)},\hat{\rho}_n^{(N)})+2\cM(\hat{\rho}_n^{(0)})\leq 2N\sum_{k=1}^Nd_W^2(\hat{\rho}_n^{(k-1)},\hat{\rho}_n^{(k)})+2\cM(\hat{\rho}_n^{(0)})$. 
Denote $\cR=\cM(\hat{\rho}_n^{(N)})+1$, note that $|\mF_{\hat{\Psi}_n}(\hat{\rho}_n^{(0)})|=O_P(1)$ and $N\delta$ is bounded, we get $\cR\leq C \cR^\alpha+\cC$, where $\cC=O_P(1)$. By Young's inequality, we found $\cR=O_P(1)$. Combining this with (\ref{eq-32}),(\ref{eq-33}),(\ref{eq-34}), we get
\[
\left|-\int_{\mR^d}\zeta(0,x)\rho^0 dx-\int_{\mR^d}\int_0^{\iy}\hat{\rho}_\delta^n(\pt_t\zeta-\nabla\hat{\Psi}_n\cdot\nabla\zeta+\Delta\zeta) dtdx\right|
=O_P(\delta).
\]
Together with
$\int_{\mR^d}\int_0^{\iy}\rho(\pt_t\zeta-\nabla\Psi\cdot\nabla\zeta+\Delta\zeta) dtdx
=-\int_{\mR^d}\zeta(0,x)\rho^0dx$, 
and multiplying the equation by $\sqrt{n}$, we get

\begin{equation}\label{eq-t11}
 \int_{\mR_+\times\mR^d}\hat{V}_\delta^n(\pt_t \zeta- \nabla\Psi\cdot\nabla\zeta +\Delta\zeta))dt dx  \\
= \int_{\mR_+\times\mR^d}\hat{\rho}_\delta^n\sqrt{n}(\nabla\hat{\Psi}_n-\nabla\Psi) \cdot\nabla\zeta dtdx+O_P(\delta\sqrt{n}). \\
\end{equation}
Assume the support of $\zeta$ is a subset of $[0,T]\times D$, where $D$ is a compact subset of $\mR^d$. Since $\sup_{x \in D} | \sqrt{n}(\nabla \hat{\Psi}_n(x) - \nabla \Psi(x))  - \nabla\tau(x) \bZ  |=o_P(1)$, we have
\begin{equation}\label{eq-t12}
\begin{split}
&\int_{\mR_+\times\mR^d}\hat{\rho}_\delta^n\sqrt{n}(\nabla\hat{\Psi}_n-\nabla\Psi) \cdot\nabla\zeta dtdx \\
=& \int_{\mR_+\times\mR^d}\rho\nabla\tau(x) \bZ \cdot\nabla\zeta dtdx+\int_{\mR_+\times\mR^d}(\hat{\rho}_\delta^n-\rho)\nabla\tau(x) \bZ \cdot\nabla\zeta dtdx+o_P(1)
\end{split}
\end{equation}
since for any $\xi \in \cC^\infty_0(\mR_+\times \mR^d)$, $\left|  \int_{\mR_+ \times \mR^d} ( \hat{\rho}_\delta^n - \rho) \xi  dt dx \right| =o_P(1)$, so
\[\int_{\mR_+\times\mR^d}(\hat{\rho}_\delta^n-\rho)\nabla\tau(x) \bZ \cdot\nabla\zeta dtdx=\int_{\mR_+\times\mR^d}(\hat{\rho}_\delta^n-\rho)\nabla'\zeta\nabla\tau(x)dtdx\cdot\bZ=o_P(1)\cdot\bZ=o_P(1).\]
Combining (\ref{eq-t11}), (\ref{eq-t12}), since $\delta\sqrt{n}\ra 0$, we get
\begin{equation}\label{eq-t4}
\begin{split}
& \int_{\mR_+\times\mR^d}\hat{V}_\delta^n(\pt_t \zeta- \nabla\Psi\cdot\nabla\zeta +\Delta\zeta))dt dx  \\
\stackrel{p}{\ra}& \int_{\mR_+\times\mR^d}[\rho\nabla\tau(x) \bZ]\cdot\nabla\zeta dtdx\\
=&\int_{\mR_+\times\mR^d}V(\pt_t \zeta-\nabla\Psi\cdot\nabla\zeta+\Delta\zeta)dtdx.
\end{split}
\end{equation}
For any $\xi\in\cC_0^\iy(\mR_+\times\mR^d)$, find $\zeta\in\cC_0^\iy(\mR_+\times\mR^d)$ such that
\[\xi=\pt_t\zeta-\nabla\Psi\cdot\nabla\zeta+\Delta\zeta,\]
then we get the desired result. The linear dependence is straightforward. 

By the Fokker-Planck equation of $\tilde{\rho}^n$, we can obtain the similar equality as (\ref{eq-t11}), where we can replace $\hat{\rho}_{\delta}^n$ with $\tilde{\rho}_n$ and remove the $O_P(\delta)$ term. Then using the same argument we can prove $\tilde{V}^n$ also converges to $V$ as $n\ra\iy$. 

\section{Sketch of the Proof in Section 5}
\label{appB}

First we give an important lemma on $\hat{\rho}_\delta$:
\begin{lem}\label{thm-41}
Assume $\{\hat{\Psi}_k\}$ satisfies (A5), (A6), $\delta m^{1/(2-2\gamma)}>C>0$, then for any $\zeta \in \cC^\infty_0(\mR_+\times \mR^d)$, we have
\[
\left|-\int_{\mR^d}\zeta(0,x)\rho^0dx-\int_{\mR^d}\int_0^{\iy}\hat{\rho}_\delta(\pt_t\zeta-\nabla\hat{\Psi}_{\lceil t/\delta\rceil}\cdot\nabla\zeta+\Delta\zeta) dtdx\right|
=O_P(\delta).
\]
\end{lem}

\begin{proof}
Here we only give some hints on the proof. Assume the support is of $\zeta$ a subset of $[0,T]\times D$. Given (A5), by the similar argument in the proof of Theorem \ref{thm-1}, we get
\begin{equation}\label{eqs-1}
\begin{split}
&\left|-\int_{\mR^d}\zeta(0,x)\rho^0dx-\int_{\mR^d}\int_0^{\iy}\hat{\rho}_\delta(\pt_t\zeta-\nabla\hat{\Psi}_{\lceil t/\delta\rceil}\cdot\nabla\zeta+\Delta\zeta) dtdx\right| \\
\leq &\frac{1}{2}\sup_{t,x}|\nabla^2\zeta|\sum_{k=1}^{N_\delta}d_W^2(\hat{\rho}^{(k-1)},\hat{\rho}^{(k)})+O_P(\delta).
\end{split}
\end{equation}
Next, for any $1\leq N\leq N_\delta$, we have
\[
\frac{1}{2}\sum_{k=1}^{N}d_W^2(\hat{\rho}^{(k-1)},\hat{\rho}^{(k)})
\leq\delta\left(\mF_{\hat{\Psi}_1}(\hat{\rho}^{(0)})-\mF_{\hat{\Psi}_{N}}(\hat{\rho}^{(N)})+\sum_{k=1}^{N-1}(\cE_{\hat{\Psi}_{k+1}}(\hat{\rho}^{(k)})-\cE_{\hat{\Psi}_k}(\hat{\rho}^{(k)}))\right).\]
Since $\{\hat{\Psi}_k\}$ satisfies (A6), let
\[\cC=\max_{x\in\mR^d, 1\leq k< \lceil T/\delta\rceil}\frac{k^\gamma m^{1/2}}{(1+|x|^2)}|\hat{\Psi}_{k+1}(x)-\hat{\Psi}_k(x)|=O_P(1),\]
then
\[\left|\cE_{\hat{\Psi}_{k+1}}(\hat{\rho}^{(k)})-\cE_{\hat{\Psi}_k}(\hat{\rho}^{(k)})\right|\leq\int_{\mR^d}\cC m^{-1/2}k^{-\gamma}(1+|x|^2)\hat{\rho}^{(k)}(x)dx=\cC m^{-1/2}k^{-\gamma}(1+\cM(\hat{\rho}^{(k)})),\]
\begin{equation}\label{eq-thm5}
\frac{1}{2}\sum_{k=1}^{N}d_W^2(\hat{\rho}^{(k-1)},\hat{\rho}^{(k)})
\leq\delta\left(|\mF_{\hat{\Psi}_1}(\hat{\rho}^{(0)})|+C(\cM(\hat{\rho}^{(N)})+1)^\alpha+\cC m^{-1/2}\sum_{k=1}^{N-1}k^{-\gamma}(1+\cM(\hat{\rho}^{(k)}))\right)
\end{equation}
and
$\cM(\hat{\rho}^{(N)})\leq 2d_W^2(\hat{\rho}^{(0)},\hat{\rho}^{(N)})+2\cM(\hat{\rho}^{(0)})\leq 2N\sum_{k=1}^Nd_W^2(\hat{\rho}^{(k-1)},\hat{\rho}^{(k)})+2\cM(\hat{\rho}^{(0)}).$
Denote $a_N=\cM(\hat{\rho}^{(N)})+1$. Note that $|\mF_{\hat{\Psi}_1}(\hat{\rho}^{(0)})|=O_P(1)$ and $N\delta$ is bounded. We get
$a_N\leq \cC(a_N^\alpha+ m^{-1/2}\sum_{k=1}^{N-1}k^{-\gamma}a_k)+\cC$, where $\cC=O_P(1)$ may differ from each other. By Young's inequality, we get
$a_N\leq \cC+\cC m^{-1/2}\sum_{k=1}^{N-1}k^{-\gamma}a_k$.
Since $\delta m^{1/(2-2\gamma)}>C$, $N\leq T/\delta$, we have $N^{1-\gamma} m^{-1/2}=O(1)$. Let $S_N=\sum_{k=1}^N k^{-\gamma}a_k$, then
\[
S_N\leq \cC N^{-\gamma}+(1+\cC N^{-\gamma}m^{-1/2})S_{N-1} 
\leq \left[\prod_{k=1}^N(1+\cC k^{-\gamma}m^{-1/2})\right]\sum_{k=1}^N \cC k^{-\gamma} 
\leq \cC N^{1-\gamma}
\]
then
$a_N\leq \cC+\cC m^{-1/2} S_N=O_P(1)$
and
\begin{equation}\label{eq-thm7}
\frac{1}{2}\sum_{k=1}^{N_\delta}d_W^2(\hat{\rho}^{(k-1)},\hat{\rho}^{(k)})\leq \delta(O_P(1)+\cC m^{-1/2} S_{N_\delta-1})=O_P(\delta),
\end{equation}

which gives the desired result.
\end{proof}

The next lemma will control the difference between  $\nabla_x\Psi_{\hat{\theta}^k}(x)$ and $\nabla_x\Psi_\theta(x)$, and its proof can be found in the supplementary material.
\begin{lem}\label{lem5-1}
Assume (A7) and $k=o(m^{1+2\alpha})$, we have
\begin{equation}\label{eqq-51}
\sqrt{mk}(\nabla_x\Psi_{\hat{\theta}^k}(x)-\nabla_x\Psi_\theta(x))-\sqrt{k}\nabla\tau(x)\bar{\bZ}_k\xrightarrow{L_2}0
\end{equation}
uniformly for $x\in D$, where $D$ is a compact set.
\end{lem}

Now we can give the sketch of the proof of Theorem 5.1.

\subsection{Proof of Theorem 5.1}
By Lemma \ref{thm-41}, we get
\begin{equation}\label{eq520n2}
\begin{split}
& \int_{\mR_+\times\mR^d}\hat{V}_{\delta,1}^m(\pt_t \zeta- \nabla\Psi\cdot\nabla\zeta +\Delta\zeta))dt dx  \\
=& \int_{\mR_+\times\mR^d}\rho(m/\delta)^{1/2}(\nabla_x\hat{\Psi}_{\lceil t/\delta\rceil}-\nabla\Psi) \cdot\nabla\zeta dtdx+\\
&\int_{\mR_+\times \mR^d}[\hat{\rho}_{\delta,1}^{m}-\rho](m/\delta)^{1/2}(\nabla_x\hat{\Psi}_{\lceil t/\delta\rceil}-\nabla\Psi) \cdot\nabla\zeta dtdx+O_P((m\delta)^{1/2}).
\end{split}
\end{equation}

By Lemma \ref{lem5-1}, since $T/\delta=o(m^{1+2\alpha})$, for any $\ve>0$, $\exists N$, when $m>N$ and $k\leq T/\delta$, we have
$E\left|m^{1/2}(\nabla_x\hat{\Psi}_{k}-\nabla\Psi)-\nabla\tau(x)\bar{\bZ}_k\right|<\ve k^{-1/2}$.
Then
\begin{equation}\label{eq520n3}
\begin{split}
&E\left|\int_{[0,T]\times D}\rho(m/\delta)^{1/2}(\nabla_x\hat{\Psi}_{\lceil t/\delta\rceil}-\nabla\Psi) \cdot\nabla\zeta dtdx-\sum_{k=1}^{\lceil T/\delta\rceil}\int_{t_{k-1}}^{t_k}\int_D[\rho\delta^{-1/2}\nabla\tau(x)\bar{\bZ}_k]\cdot\nabla\zeta dxdt\right|\\
\leq&\sum_{k=1}^{\lceil T/\delta\rceil}\int_{t_{k-1}}^{t_k}\int_D\rho\delta^{-1/2}E\left|m^{1/2}(\nabla_x\hat{\Psi}_{k}-\nabla\Psi)-\nabla\tau(x)\bar{\bZ}_k\right| \cdot|\nabla\zeta| dxdt
\leq C'\ve.
\end{split}
\end{equation}

Let $W_\delta(t)=\sqrt{\delta}\sum_{k=1}^{\lceil t/\delta\rceil}\bZ_k$. Since $\{\bZ_k\}$ are i.i.d standard normal vectors, there exists a $d$-dimensional Brownian motion $\{W(t)\}$, such that $\{W_\delta(t)\}$ weakly converges to $\{W(t)\}$ on $\cC([0,T])$. By Skorokhod's representation theorem, we can assume the almost surely convergence. Then as $\delta\ra 0$, 
\begin{equation}\label{eq44-4}
\begin{split}
& \sum_{k=1}^{\lceil T/\delta\rceil}\int_{t_{k-1}}^{t_k}\int_D[\rho\delta^{-1/2}\nabla\tau\bar{\bZ}_k]\cdot\nabla\zeta dxdt  \\
=& \sum_{k=1}^{\lceil T/\delta\rceil}\int_{t_{k-1}}^{t_k}\left\{\int_D(\nabla\zeta)'\rho\nabla\tau(x)dx\right\}\frac{W_\delta(k\delta)}{k\delta}dt\\
\ra&\int_0^T\left\{\int_D(\nabla\zeta)'\rho\nabla\tau(x)dx\right\}t^{-1}W(t)dt\\
=&\int_{\mR_+\times\mR^d}V_1(\pt_t \zeta-\nabla\Psi\cdot\nabla\zeta+\Delta\zeta)dtdx.
\end{split}
\end{equation}
The convergence in the third line of (\ref{eq44-4}) is true since both integrals are negligible around 0 given (A8) and we can show the convergence of the integral on $t\in[\eta,T]$. The second term in (\ref{eq520n2}) is a higher-order term. So we have
\[\int_{\mR_+\times \mR^d}[\hat{\rho}_{\delta,1}^{m}-\rho](m/\delta)^{1/2}(\nabla_x\hat{\Psi}_{\lceil t/\delta\rceil}-\nabla\Psi) \cdot\nabla\zeta dtdx=o_p(1).\]
Then combining (\ref{eq520n2}), (\ref{eq520n3}), (\ref{eq44-4}), as $\delta m\ra 0$, we get
\[\int_{\mR_+\times\mR^d}(\hat{V}_{\delta,1}^m-V_1)(\pt_t \zeta- \nabla\Psi\cdot\nabla\zeta +\Delta\zeta))dtdx=o_p(1).\]
Next we will show the existence and uniqueness of the solution of $V_1$. Consider the solution to the Fokker-Planck equation:
\[
\pt_t P=\Div(P\nabla\Psi)+\Delta P, \ \ \ P(0,\cdot)=\delta_0.
\]
Then the solution with initial distribution $\rho_0$ can be written as
\begin{equation}\label{eqa-0}
\rho(t,x)=\int_{\mR^d} P(t,x-y)\rho_0(y)dy,
\end{equation}
and we can give the explicit expression of $V_1$:
\begin{equation}\label{eqa-1}
V_1(t,x)=\int_0^t\int_{\mR^d} P(t-s,x-y)\Div[\rho(s,y)\nabla \tau(y)s^{-1}W(s)]dsdy.
\end{equation}
We can show that $V_1$ is a.s. finite. By (A8), we can choose $a<1/2$, such that for any $s<t$,  
\[\int_{\mR^d} |\Div[\rho(s,y)\nabla \tau(y)]|dy=O(s^{-a}).\]
Then we can show 
\[\int_{\mR^d} P(t-s,x-y)|\Div[\rho(s,y)\nabla \tau(y)]|dy=O(s^{-a}),\]
since when $t-s>\eta>0$, $P(t-s,x-y)$ is bounded, and when $s\ra t$, the LHS converges to $|\Div[\rho(t,x)\nabla \tau(x)]|$ which is bounded. Let $b\in (a,1/2)$, then $\sup_{[0,T]}\frac{|W(s)|}{s^b}=O_P(1)$, and
\[
\begin{aligned}
&\left|\int_0^t\left[\int_{\mR^d} P(t-s,x-y)\Div[\rho(s,y)\nabla \tau(y)]dy\right]s^{-1}W(s)ds\right| \\
\leq &C\int_0^t s^{-a}s^{-1}\cC s^b ds 
=O_P(1).
\end{aligned}
\]
If we replace $P$ with $\pt_t P, \nabla P, \Delta P$ in (\ref{eqa-1}), the resulting integral is also a.s. finite. Then the direct calculation gives the following:
\[
\begin{aligned}
\pt_t V_1&=\Div[\rho(t,x)\nabla \tau(x)t^{-1}W(t)]+\int_0^t\int_{\mR^d} \pt_t P(t-s,x-y)\Div[\rho(s,y)\nabla \tau(y)s^{-1}W(s)]dsdy \\
&=\Div[\rho(t,x)\nabla \tau(x)t^{-1}W(t)]+\Div(V_1\nabla\Psi)+\Delta V_1.
\end{aligned}
\]
Finally, the difference of two solutions of $V_1$ will solve the original Fokker-Planck equation with $\rho_0=0$. By the explicit solution of the Fokker-Planck equation (\ref{eqa-0}), the solution will be zero when the initial value is zero, so the solution must be unique, see also \cite{otto1996l1}.

If we only assume (A7),(A8), then by the Fokker-Planck equation of $\tilde{\rho}_1^m$, we can obtain the following equality, which is similar to (\ref{eq520n3}) without the $O_P(\delta)$ term.
\begin{equation*}\label{eq520nn}
 \int_{\mR_+\times\mR^d}(\tilde{\rho}_{1}^{m}-\rho)(\pt_t\zeta-\nabla\Psi\cdot\nabla\zeta+\Delta\zeta) dtdx
=\int_{\mR_+\times \mR^d}\tilde{\rho}_{1}^{m}(\nabla_x\hat{\Psi}_{\lceil t/\delta\rceil}-\nabla\Psi) \cdot\nabla\zeta dtdx.
\end{equation*}
Then using the same argument we can prove $\tilde{V}_1^m$ also converges to $V_1$ as $\delta m^{1+2\alpha}\ra\iy$ (Note that here we do not need to assume $\delta m\ra 0$ and $\delta m^{1/(2-2\gamma)}>C>0$). Furthermore, if we do not assume (A7) (only assume (A8)), then in (\ref{eq520n3}), we have
\[m^{1/2}(\nabla_x\hat{\Psi}_{k}-\nabla\Psi)-\nabla\tau(x)\bar{\bZ}_k=O_P(m^{-1/2}),\]
then as $\delta\ra 0, m\ra\iy, \delta m\ra\iy$, $\tilde{V}_1^m$ still converges to $ V_1$.




\subsection{Sketch of the proof of Theorem 5.2}
\label{appB6}
The proof is very similar to that of Theorem 5.1. Here we show the main difference. Let 
$\epsilon_{m,j}=\sqrt{m}(\nabla_x\Psi_{\hat{\theta}^{(j)}_m}(x)-\nabla_x\Psi_\theta(x))-\nabla\tau(x)\bZ_j$, then by (A9),
\[|E\epsilon_{m,j}|=O(m^{-1/2-\alpha}), E|\epsilon_{m,j}|^2=o(1).\]
Since $\epsilon_{m,j}$'s are independent for different $j$. When $k=o(m^{1+2\alpha})$, a similar calculation gives
\[E\left|\sqrt{mk}(\nabla_x\hat{\Psi}_m^{(k)}(x)-\nabla_x\Psi_\theta(x))-\sqrt{k}\nabla\tau(x)\bar{\bZ}_k\right|^2=E\left|k^{-1/2}\sum_{j=1}^k\epsilon_{m,j}\right|^2=o(1).\]
Then as $\delta\ra 0, m\ra\iy, \delta m^{1+2\alpha}\ra\iy$, we have
\begin{equation}\label{eq44-2}
\begin{split}
&E\left|\int_{[0,T]\times D}\rho(m/\delta)^{1/2}(\nabla\hat{\Psi}_m^{(\lceil t/\delta\rceil)}-\nabla\Psi) \cdot\nabla\zeta dtdx-\sum_{k=1}^{\lceil T/\delta\rceil}\int_{t_{k-1}}^{t_k}\int_D[\rho\delta^{-1/2}\nabla\tau(x)\bar{\bZ}_k]\cdot\nabla\zeta dxdt\right| \\
\leq&\sum_{k=1}^{\lceil T/\delta\rceil}\int_{t_{k-1}}^{t_k}\int_D\rho\delta^{-1/2}E\left|m^{1/2}(\nabla\hat{\Psi}_m^{(k)}-\nabla\Psi)-\nabla\tau(x)\bar{\bZ}_k\right| \cdot|\nabla\zeta| dxdt\\
\leq&\sum_{k=1}^{\lceil T/\delta\rceil}\cC\delta^{1/2}o(k^{-1/2})
=o(1).
\end{split}
\end{equation}
The rest of the proof follows the same framework as the proof of Theorem 5.1.

\end{appendix}

\bibliographystyle{imsart-number}    
\bibliography{myReferences-ML}       

\newpage
\setcounter{section}{0}
\setcounter{equation}{0}
\def\theequation{S\arabic{section}.\arabic{equation}}
\def\thesection{S\arabic{section}}

{\center{\bf Supplementary Materials for ``Computational and Statistical Asymptotic Analysis of the JKO Scheme for Iterative Algorithms to Update Distributions''}}
\section{Proof in Section 4}
\label{appA} 
To study the asymptotic property of $\hat{\rho}_\delta^n(t,x)$, first we need to examine the asymptotic property of $\hat{\theta}_n$. Note that the estimator $\hat{\theta}_n$  satisfies 
\begin{equation}\label{Esti}
\sum_{i=1}^n\nabla_x\Psi_{\hat{\theta}_n}(X(i\eta))=0 ,
\end{equation}
where $X(t)$ follows the Langevin equation. 
\begin{equation}\label{Lang}
dX(t)=-\nabla\Psi(X(t))dt+\sqrt{2/\beta}dB_t, X(0)=X_0
\end{equation}
We can establish the following asymptotic normality for $\hat{\theta}_n$. 

\begin{lem}\label{prop-th}
Assume $\Psi_\theta(x)$ satisfies (A1), let $f_\theta(x)=\nabla_x\Psi_\theta(x)$ satisfies (A2), (A3), (A4), then we have
\begin{equation} \label{asymp-theta} 
\sqrt{n}(\hat{\theta}_n-\theta)\stackrel{d}{\rightarrow} \gamma_\theta \, \bZ , 
\end{equation} 
where $\bZ$ denotes a standard normal vector, $\gamma_\theta=\Gamma_\theta^{1/2}$, 
$\Gamma_\theta=A^{-1}\Sigma_f(A')^{-1}$, $A=E_\pi(\nabla_\theta'f_\theta(X))$, $\Sigma_f={\rm Var}_\pi(f(X(0)))+2\sum_{i=1}^\infty\cov_\pi(f(X(0),f(X(i\eta)))$, and the variance 
${\rm Var}_\pi$ and covariance $\cov_\pi$ are evaluated under the situation that $X(0)$ follows the invariant distribution $\pi$ and the Langevin diffusion 
$X(t)$ evolves according to (\ref{Lang}). 
\end{lem}

\begin{proof}
The Proposition 4.12 in \cite{locherbach2013ergodicity} shows that for any multivariate diffusion process satisfying:
\begin{equation}\label{diffusion}
dX(t)=b(X(t))dt+\sigma(X(t))dB_t, X(0)=x\in\mR^d
\end{equation}
Assume 

(A2*): $b$ and $\sigma$ are in $C^2$ and of linear growth and $\sigma$ is bounded. 

(A3*) (Uniform ellipticity): Denote $a(x)=\sigma(x)\sigma'(x)$, there exists $\lambda_->0$, such that $\langle a(x)y,y \rangle\geq \lambda_-|y|^2$ for all $y\in\mR^d$.

(A4*) (Drift condition): There exists $M_0>0$ and $r>2$ such that
\[\langle b(x),x\rangle+\frac{1}{2}tr[a(x)]\leq -r\left[\left<a(x)\frac{x}{|x|},\frac{x}{|x|}\right>\right],\ {\rm for\ all}\ |x|\geq M_0.\]

Then for any fixed $k<r-2, p\in(4+2k,2r)$, we have $E_\pi|X|^p<\infty$, and polynomial ergodicity:
\[\|P_t(x,\cdot)-\pi\|_{TV}\leq Ct^{-(k+1)}(1+|x|^p),\]
where $P_t(x,\cdot)$ is the distribution of $X(t)$ starting at $x$, $\pi$ is the invariant distribution of $X(t)$.

Then markov chain $\{X(0),X(\eta),X(2\eta),X(3\eta),\dots\}$ has polynomial ergodicity:
\begin{equation}\label{eq-ergo}
\|P^j(x,\cdot)-\pi\|_{TV}\leq C\eta^{-(k+1)}j^{-(k+1)}(1+|x|^p)
\end{equation}
The Corollary 2 in \cite{jones2004markov} shows that if a markov chain has polynomial ergodicity, say
\[\|P^j(x,\cdot)-\pi\|_{TV}\leq Cj^{-m}M(x),\]
if $E_{\pi}M(X)<\iy$, and there exists $\kappa>\frac{2}{m-1}$, such that $E_\pi|f(X)|^{2+\kappa}<\infty$ for some function $f$, then the CLT follows:
\[\sqrt{n}(\bar{f}_n-E_\pi f)\stackrel{d}{\rightarrow}N(0,\Sigma_f)\]
where $\bar{f}_n=n^{-1}\sum_{i=1}^n f(X(i\eta))$, $\Sigma_f=\mbox{Var}_\pi(f(X(0)))+2\sum_{i=1}^\infty \cov_\pi(f(X(0)),f(X(i\eta)))$.
Here in the expression of $\bar{f}_n$, $X(t)$ can initiate at any random vector---that is, $X(t)$ satisfies (\ref{diffusion}) with $X(0)=X_0$, the initial value $X_0$ is a random vector.

Now recall the Langevin equation (\ref{Lang}), let us check assumptions (A2*),(A3*),(A4*) for this equation, here $b(x)=-\nabla\Psi(x)$, $\sigma(x)\equiv \sqrt{2/\beta}I_d$, (A2*) requires that $\nabla\Psi(x)$ is in $C^2$ of linear growth, (A3*) is evident, since $a(x)=2/\beta I_d$, (A4*) requires
\[
\langle \nabla\Psi(x),x\rangle\geq \frac{2r+d}{\beta}\ {\rm for}\ |x|\geq M_0,
\]
which is exactly (A3). Thus, if $f(x)=\nabla_x\Psi_\theta(x)$ satisfies (A2) and (A3), then $X(t)$ has polynomial ergodicity (\ref{eq-ergo}) and $E_\pi|X|^p<\iy$. If we further assume (A4), then
\[\sqrt{n}(\bar{f}_n-E_\pi f)\stackrel{d}{\rightarrow}N(0,\Sigma_f)\]
Recall the invariant distribution of $X(t)$:
\[
\pi(x)=\frac{1}{Z}\exp[-\beta\Psi(x)], Z=\int_{\mR^d}\exp[-\beta\Psi(x)],
\]
we have $\log(\pi(x))=-\beta\Psi(x)-\log(Z)$, $\nabla\log(\pi(x))=-\beta\nabla\Psi(x)$,
\[E_\pi[\nabla\Psi(X)]=-\frac{1}{\beta}E_\pi[\nabla\log(\pi(X))]=-\frac{1}{\beta}\int_{\mR^d}\frac{\nabla \pi(x)}{\pi(x)}\pi(x)dx=0\]
That is, $E_\pi f(X)=0$, then
\[n^{-1/2}\sum_{i=1}^n\nabla_x\Psi_\theta(X(i\eta))\stackrel{d}{\rightarrow}N(0,\Sigma_f)\]
Together with (\ref{Esti}), we have
\[n^{-1/2}\sum_{i=1}^n\left[\nabla_x\Psi_{\hat{\theta}_n}(X(i\eta))-\nabla_x\Psi_\theta(X(i\eta))\right]\stackrel{d}{\rightarrow}N(0,\Sigma_f)\]
When $\nabla_\theta'\nabla_x\Psi_\theta(x)$ is continuous,  $\nabla_x\Psi_{\hat{\theta}_n}(X(i\eta))-\nabla_x\Psi_\theta(X(i\eta))=\nabla_\theta'\nabla_x\Psi_\theta(X(i\eta))(\hat{\theta}_n-\theta)+o_p(\hat{\theta}_n-\theta)$,
then we have
\[n^{-1}\sum_{i=1}^n\nabla_\theta'\nabla_x\Psi_\theta(X(i\eta))\cdot \sqrt{n}(\hat{\theta}_n-\theta)\stackrel{d}{\rightarrow}N(0,\Sigma_f),\]
Note that $n^{-1}\sum_{i=1}^n\nabla_\theta'\nabla_x\Psi_\theta(X(i\eta))\stackrel{p}{\ra}E_\pi[\nabla_\theta'\nabla_x\Psi_\theta(X)]=A$, so
\[\sqrt{n}(\hat{\theta}_n-\theta)\stackrel{d}{\rightarrow}N(0,\Gamma_\theta),\]
where $\Gamma_\theta=A^{-1}\Sigma_f(A')^{-1}$.
\end{proof}

\begin{remark}\label{rem-1}
(A3) holds as long as $\nabla\Psi(x)$ and $x$ are roughly in the same direction, e.g. $\Psi(x)=\|x\|^2/2$, $\nabla\Psi(x)=x$, then $\langle \nabla\Psi(x),x\rangle\ra\iy$ as $|x|\ra\iy$, In this case, (A3) is true for any $r>0$ as long as $M_0$ is large enough, then in (A4) $\kappa$ can be as small as possible, and $f(X)$ only needs to have finite moment for some order slightly larger than $2$ under the invariant distribution of $X(t)$.
\end{remark}

Let $\tau(x)=\nabla_\theta'\Psi_\theta(x)\gamma_\theta$. Note that $\Psi(x)=\Psi_\theta(x)$, and $\hat{\Psi}_n(x)=\Psi_{\hat{\theta}_n}(x)$. Using Lemma \ref{prop-th} and applying the delta method we obtain  
\[\sqrt{n}(\hat{\Psi}_n(x)-\Psi(x))\stackrel{d}{\rightarrow}\tau(x)\bZ,
\quad 
 \sqrt{n} [\nabla \hat{\Psi}_n(x) - \nabla\Psi(x)] \stackrel{d}{\rightarrow}  \nabla\tau(x) \bZ,
\]
where $\tau(x)$ is a row vector, denote $\tau(x)=(\tau_1(x),\dots,\tau_d(x))$, then $\nabla\tau(x)$ is a $d\times d$ matrix, the $j$-th column of $\nabla\tau(x)$ is $\nabla\tau_j(x)$.

Since $\Psi$ is smooth, we have that for a compact set $D \subset \mR^d$, 
\[ \sup_{x \in D} | \hat{\Psi}_n(x) - \Psi(x) | = O_P(n^{-1/2} ) ,\ \   \sup_{x \in D} | \nabla \hat{\Psi}_n(x) - \nabla \Psi(x) | = O_P(n^{-1/2}).
 \]

Furthermore, using the Skorokhod representation theorem we may have 

 \[ \sup_{x \in D} | \hat{\Psi}_n(x) - \Psi(x) - n^{-1/2} \tau(x)\bZ  | = O_P(n^{-1} ) = o_P(n^{-1/2})  , \]
 \[   \sup_{x \in D} | \nabla \hat{\Psi}_n(x) - \nabla \Psi(x)  - n^{-1/2} \nabla\tau(x) \bZ  | =O_P(n^{-1} ) = o_P(n^{-1/2}), 
 \]
where the equations may be realized in different probability spaces. 

\vspace{0.1in}

Before we proceed to the proof of the theorem, we need another important lamma, this result has been proved in \cite{jordan1998variational}, I summarize the lemma and proof here just to keep the paper self-contained.  

\begin{lem}\label{prop-ip}
Assume
\begin{equation}\label{WGF}
\rho^*={\rm arg}\min_{\rho\in \Xi} \left\{\frac{1}{2}d_W^2(\rho^o,\rho)+\delta \mF_{\Psi}(\rho)\right\},
\end{equation}
then for any $\zeta\in\cC_0^\iy(\mR^d)$,
\[\left|\int_{\mR^d}\left\{(\rho^*-\rho^o)\zeta dy+\delta(\nabla\Psi\cdot\nabla\zeta-\beta^{-1}\Delta\zeta)\rho^*\right\}dy\right|\leq \frac{1}{2}\sup_{\mR^d}|\nabla^2\zeta|d_W^2(\rho^o,\rho^*).\]
\end{lem}

\begin{proof}
For any smooth vector field with bounded support, $\xi:\mR^d\ra\mR^d$, $\xi\in\cC_0^\iy$, define a flux $\Phi_\tau:\mR^d\ra\mR^d$ by
\[\pt_\tau\Phi_\tau(y)=\xi(\Phi_\tau(y)), \Phi_0(y)=y\]
For any $\tau$, let the measure $\rho_\tau(y)dy$ be the push forward of $\rho^*(y)dy$ under $\Phi_\tau$, then if $Y$ has density $\rho^*$, $\Phi_\tau(Y)$ will have density $\rho_\tau$, and for any $\zeta\in\cC_0(\mR^d)$,
\begin{equation}\label{eq-lm2}
  \mE_{\rho_\tau}\zeta(X)=\mE_{\rho^*}\zeta(\Phi_\tau(Y))  
\end{equation}
By the definition of $\rho^*$, 
\[\frac{1}{2}d_W^2(\rho^o,\rho_\tau)+\delta \mF_{\Psi}(\rho_\tau)-\frac{1}{2}d_W^2(\rho^o,\rho^*)-\delta \mF_{\Psi}(\rho^*)\geq 0\]
Let $P$ be optimal in the definition of $d_W^2(\rho^o,\rho^*)$, assume $(X,Y)$ follows distribution $P$, then the distribution of $(X,\Phi_\tau(Y))$ belongs to $\cP(\rho^o,\rho_\tau)$, we have
\[d_W^2(\rho^o,\rho_\tau)-d_W^2(\rho^o,\rho^*)\leq \mE_P(|\Phi_\tau(Y)-X|^2-|Y-X|^2)\]
Therefore
\[\mE_P\left(\frac{1}{2}|\Phi_\tau(Y)-X|^2-\frac{1}{2}|Y-X|^2\right)+\delta(\mF_{\Psi}(\rho_\tau)-\mF_{\Psi}(\rho^*))\geq 0\]
Thus the derivative of $\mE_P\left(\frac{1}{2}|\Phi_\tau(Y)-X|^2\right)+\delta\mF_{\Psi}(\rho_\tau)$ at $\tau=0$ is nonnegative. First we have
\[\pt_\tau\mE_P\left.\frac{1}{2}|\Phi_\tau(Y)-X|^2\right|_{\tau=0}=\mE_P((Y-X)\cdot\xi(Y)),\]
Next, using (\ref{eq-lm2}), since $\Psi\geq 0$, approximating log by some function bounded below and taking limit, we get
\[\mF_{\Psi}(\rho_\tau)=\mE_{\rho_\tau}(\Psi(X)+\beta^{-1}\log[\rho_\tau(X)])=\mE_{\rho^*}(\Psi(\Phi_\tau(Y))+\beta^{-1}\log[\rho_\tau(\Phi_\tau(Y))]),\]
and we can calculate that 
\[\pt_\tau\mF_\Psi(\rho_\tau)|_{\tau=0}=\mE_{\rho^*}(\nabla\Psi(Y)\cdot\xi(Y)-\beta^{-1}\Div(\xi(Y)))\]
So we conclude
\begin{equation}\label{eq-lmp1}
\int_{\mR^d\times\mR^d}(y-x)\cdot\xi(y)P(dxdy)+\delta\int_{\mR^d}(\nabla\Psi(y)\cdot\xi(y)-\beta^{-1}\Div\xi(y))\rho^*(y)dy\geq 0,
\end{equation}
replacing $\xi$ with $-\xi$, we obtain the equality.

Then, note that
\[\int_{\mR^d}(\rho^*-\rho^o)\zeta dy=\mE_P(\zeta(Y)-\zeta(X)),\]
let $\xi=\nabla\zeta$, using (\ref{eq-lmp1}) with equality, we get
\[\begin{split}
&\left|\int_{\mR^d}\left\{(\rho^*-\rho^o)\zeta dy+\delta(\nabla\Psi\cdot\nabla\zeta-\beta^{-1}\Delta\zeta)\rho^*\right\}dy\right|\\
=&\left|\mE_P(\zeta(Y)-\zeta(X)-(Y-X)\cdot\nabla\zeta(Y))\right|\\
\leq& \frac{1}{2}\sup_{\mR^d}|\nabla^2\zeta|\mE_P|X-Y|^2\\
=&\frac{1}{2}\sup_{\mR^d}|\nabla^2\zeta|d_W^2(\rho^o,\rho^*).
\end{split}\]

\end{proof}




Finally, we can give the proof of Theorem 4.1.

\subsection{Proof of Theorem 4.1}
\label{appA3}
Recall that
\begin{equation*}
\hat{\rho}_{n}^{(k)}={\rm arg}\min_{\rho\in \Xi} \left\{\frac{1}{2}\left[d_W\left(\rho,\hat{\rho}_{n}^{(k-1)}\right)\right]^2+\delta \mF_{\hat{\Psi}_n}(\rho)\right\}
\ ,\hat{\rho}_{n}^{(0)}=\rho^0,
\end{equation*}
and
\[\hat{\rho}_{\delta}^n(t,x)=\hat{\rho}_{n}^{(\lceil t/\delta\rceil)}(x),\]
Assume w.l.o.g. $\beta=1$, denote $t_k=k\delta$. For any $\zeta\in \cC_0^\iy(\mR_+\times\mR^d)$ and $k\geq 1$, applying Lemma \ref{prop-ip} on $\hat{\rho}_n^{(k-1)},\hat{\rho}_n^{(k)},\hat{\Psi}_n$ and $\zeta(t_{k-1},\cdot)$, we get
\begin{equation}\label{eq-31}
\begin{split}
&\left|\int_{\mR^d}\left\{(\hat{\rho}_n^{(k)}-\hat{\rho}_n^{(k-1)})\zeta(t_{k-1},x)+\delta(\nabla\hat{\Psi}_n\cdot\nabla\zeta(t_{k-1},x)-\Delta\zeta(t_{k-1},x))\hat{\rho}_n^{(k)}\right\}dx\right|\\
\leq\ &\frac{1}{2}\sup_{x\in\mR^d}|\nabla^2\zeta(t_{k-1},x)|d_W^2(\hat{\rho}_n^{(k-1)},\hat{\rho}_n^{(k)}).
\end{split}
\end{equation}
Since $\zeta$ has compact support, assume the support is a subset of $[0,T]\times D$, where $T>1$, $D\subset\mR^d$ is a compact set. Let $N_\delta=\lfloor T/\delta\rfloor+1$, then
\begin{equation}\label{eq-32}
\begin{split}
&\left|\sum_{k=1}^{N_\delta}\int_{\mR^d}\left\{(\hat{\rho}_n^{(k)}-\hat{\rho}_n^{(k-1)})\zeta(t_{k-1},x)+\delta(\nabla\hat{\Psi}_n\cdot\nabla\zeta(t_{k-1},x)-\Delta\zeta(t_{k-1},x))\hat{\rho}_n^{(k)}\right\}dx\right|\\
\leq\ &\frac{1}{2}\sup_{t,x}|\nabla^2\zeta|\sum_{k=1}^{N_\delta}d_W^2(\hat{\rho}_n^{(k-1)},\hat{\rho}_n^{(k)})
\end{split}
\end{equation}
For fix $x$, we have 
\begin{equation}\label{eq-33}
\begin{split}
&\sum_{k=1}^{N_\delta}(\hat{\rho}_n^{(k)}-\hat{\rho}_n^{(k-1)})\zeta(t_{k-1},x)\\
=&-\zeta(0,x)\hat{\rho}_n^{(0)}-\sum_{k=1}^{N_\delta}(\zeta(t_k,x)-\zeta(t_{k-1},x))\hat{\rho}_n^{(k)}\\
=&-\zeta(0,x)\hat{\rho}_n^{(0)}-\int_0^{\iy}\pt_t\zeta(t,x)\hat{\rho}^n_\delta(t,x)dt
\end{split}
\end{equation}
Since any order derivative of $\zeta$ is continuous with compact support, and $\max_{x\in D}|\nabla\hat{\Psi}_n(x)|=O_P(1)$, so
\begin{equation}\label{eq-34}
\begin{split}
&\left|\int_{\mR^d}\left[\sum_{k=1}^{N_\delta}\left[\delta(\nabla\hat{\Psi}_n\cdot\nabla\zeta(t_{k-1},x)-\Delta\zeta(t_{k-1},x))\hat{\rho}_n^{(k)}\right]-\int_0^{\iy}(\nabla\hat{\Psi}_n\cdot\nabla\zeta-\Delta\zeta)\hat{\rho}_\delta^n dt\right]dx\right|\\
\leq&\int_{\mR^d}\sum_{k=1}^{N_\delta}\int_{t_{k-1}}^{t_k}|\nabla\hat{\Psi}_n\cdot\nabla\zeta(t_{k-1},x)-\Delta\zeta(t_{k-1},x)-(\nabla\hat{\Psi}_n\cdot\nabla\zeta-\Delta\zeta)|\hat{\rho}_n^{(k)}dtdx\\
\leq&\int_{\mR^d}\sum_{k=1}^{N_\delta}\int_{t_{k-1}}^{t_k} C\delta\left(\max_{x\in D}|\nabla\hat{\Psi}_n(x)|+1\right)\hat{\rho}_n^{(k)}dtdx\\
\leq&N_\delta C\delta^2O_P(1)\\
=&O_P(\delta)
\end{split}
\end{equation}
According to the definition of $\hat{\rho}_n^{(k)}$, we have
\[\frac{1}{2}d_W^2(\hat{\rho}_n^{(k-1)},\hat{\rho}_n^{(k)})\leq\delta(\mF_{\hat{\Psi}_n}(\hat{\rho}_n^{(k-1)})-\mF_{\hat{\Psi}_n}(\hat{\rho}_n^{(k)})),\]
for any $1\leq N\leq N_\delta$, we have
\[\begin{split}
\frac{1}{2}\sum_{k=1}^{N}d_W^2(\hat{\rho}_n^{(k-1)},\hat{\rho}_n^{(k)})
&\leq\delta\sum_{k=1}^{N}(\mF_{\hat{\Psi}_n}(\hat{\rho}_n^{(k-1)})-\mF_{\hat{\Psi}_n}(\hat{\rho}_n^{(k)}))\\
&\leq\delta\left(\mF_{\hat{\Psi}_n}(\hat{\rho}_n^{(0)})-\mF_{\hat{\Psi}_n}(\hat{\rho}_n^{(N)})\right)
\end{split}\]
Since $\Psi\geq 0$, by (14) in \cite{jordan1998variational}
\[\mF_\Psi(\rho)\geq -\cS(\rho)\geq -C(\cM(\rho)+1)^\alpha\]
where $C$ is a constant that depends on $d$, the dimension of $x$, and $\alpha\in\left(\frac{d}{d+2},1\right)$. Then
\begin{equation}\label{eq-35}
\frac{1}{2}\sum_{k=1}^{N}d_W^2(\hat{\rho}_n^{(k-1)},\hat{\rho}_n^{(k)})
\leq\delta\left(|\mF_{\hat{\Psi}_n}(\hat{\rho}_n^{(0)})|+C(\cM(\hat{\rho}_n^{(N)})+1)^\alpha\right)
\end{equation}
In addition, we have
\begin{equation}\label{eq-36}
\cM(\hat{\rho}_n^{(N)})\leq 2d_W^2(\hat{\rho}_n^{(0)},\hat{\rho}_n^{(N)})+2\cM(\hat{\rho}_n^{(0)})\leq 2N\sum_{k=1}^Nd_W^2(\hat{\rho}_n^{(k-1)},\hat{\rho}_n^{(k)})+2\cM(\hat{\rho}_n^{(0)})
\end{equation}
Denote $\cR=\cM(\hat{\rho}_n^{(N)})+1$, note that $|\mF_{\hat{\Psi}_n}(\hat{\rho}_n^{(0)})|=O_P(1)$ and $N\delta$ is bounded. Combining (\ref{eq-35}),(\ref{eq-36}), we get $\cR\leq C \cR^\alpha+\cC$, where $\cC=O_P(1)$. By Young's inequality, 
\[\cR\leq C \cR^\alpha+\cC\leq (1-\alpha)C^{1/(1-\alpha)}+\alpha\cR+\cC\]
\[\cR\leq C^{1/(1-\alpha)}+\cC/(1-\alpha)\]
Since $\alpha$ is a constant, $\cR=O_P(1)$. Then we conclude
\begin{equation}\label{eq-37}
\frac{1}{2}\sum_{k=1}^{N}d_W^2(\hat{\rho}_n^{(k-1)},\hat{\rho}_n^{(k)})=O_P(\delta)
\end{equation}
Combining (\ref{eq-32}),(\ref{eq-33}),(\ref{eq-34}),(\ref{eq-37}), we get
\[
\left|-\int_{\mR^d}\zeta(0,x)\rho^0 dx-\int_{\mR^d}\int_0^{\iy}\hat{\rho}_\delta^n(\pt_t\zeta-\nabla\hat{\Psi}_n\cdot\nabla\zeta+\Delta\zeta) dtdx\right|
=O_P(\delta)
\]
Together with
\begin{equation}\label{eq-38}
\begin{split}
&\int_{\mR^d}\int_0^{\iy}\rho(\pt_t\zeta-\nabla\Psi\cdot\nabla\zeta+\Delta\zeta) dtdx\\
=&-\int_{\mR^d}\zeta(0,x)\rho^0dx+\int_{\mR^d}\int_0^{\iy}(-\pt_t\rho+\Div(\rho\nabla\Psi)+\Delta\rho)\zeta dtdx\\
=&-\int_{\mR^d}\zeta(0,x)\rho^0dx,
\end{split}
\end{equation}
we find

\begin{equation}\label{eq-td1}
\begin{split}
 & \int_{\mR_+\times\mR^d}[\hat{\rho}_\delta^n-\rho](\pt_t \zeta- \nabla\Psi\cdot\nabla\zeta +\Delta\zeta))dt dx  \\
=& \int_{\mR_+\times\mR^d}\hat{\rho}_\delta^n(\nabla\hat{\Psi}_n-\nabla\Psi) \cdot\nabla\zeta dtdx+O_P(\delta),
\end{split}
\end{equation}
then
\begin{equation}\label{eq-t11}
\begin{split}
& \int_{\mR_+\times\mR^d}\hat{V}_\delta^n(\pt_t \zeta- \nabla\Psi\cdot\nabla\zeta +\Delta\zeta))dt dx  \\
=& \int_{\mR_+\times\mR^d}\hat{\rho}_\delta^n\sqrt{n}(\nabla\hat{\Psi}_n-\nabla\Psi) \cdot\nabla\zeta dtdx+O_P(\delta\sqrt{n}) \\
\end{split}
\end{equation}
Assume the support of $\zeta$ is a subset of $[0,T]\times D$, where $D$ is a compact subset of $\mR^d$. Since $\sup_{x \in D} | \sqrt{n}(\nabla \hat{\Psi}_n(x) - \nabla \Psi(x))  - \nabla\tau(x) \bZ  |=o_P(1)$, we have
\begin{equation}\label{eq-t12}
\begin{split}
&\int_{\mR_+\times\mR^d}\hat{\rho}_\delta^n\sqrt{n}(\nabla\hat{\Psi}_n-\nabla\Psi) \cdot\nabla\zeta dtdx \\
=& \int_{\mR_+\times\mR^d}\hat{\rho}_\delta^n\nabla\tau(x) \bZ \cdot\nabla\zeta dtdx+o_P(1)\int_{[0,T]\times D}\hat{\rho}_\delta^n dtdx\\
=& \int_{\mR_+\times\mR^d}\rho\nabla\tau(x) \bZ \cdot\nabla\zeta dtdx+\int_{\mR_+\times\mR^d}(\hat{\rho}_\delta^n-\rho)\nabla\tau(x) \bZ \cdot\nabla\zeta dtdx+o_P(1)
\end{split}
\end{equation}
since for any $\xi \in \cC^\infty_0(\mR_+\times \mR^d)$, $\left|  \int_{\mR_+ \times \mR^d} ( \hat{\rho}_\delta^n - \rho) \xi  dt dx \right| =o_P(1)$, so
\[\int_{\mR_+\times\mR^d}(\hat{\rho}_\delta^n-\rho)\nabla\tau(x) \bZ \cdot\nabla\zeta dtdx=\int_{\mR_+\times\mR^d}(\hat{\rho}_\delta^n-\rho)\nabla'\zeta\nabla\tau(x)dtdx\cdot\bZ=o_P(1)\cdot\bZ=o_P(1)\]
Combining (\ref{eq-t11}), (\ref{eq-t12}), since $\delta\sqrt{n}\ra 0$, we get
\begin{equation}\label{eq-t4}
\begin{split}
& \int_{\mR_+\times\mR^d}\hat{V}_\delta^n(\pt_t \zeta- \nabla\Psi\cdot\nabla\zeta +\Delta\zeta))dt dx  \\
\stackrel{p}{\ra}& \int_{\mR_+\times\mR^d}[\rho\nabla\tau(x) \bZ]\cdot\nabla\zeta dtdx\\
=&\int_{\mR_+\times\mR^d}-\Div(\rho\nabla\tau(x) \bZ)\zeta dtdx\\
=&\int_{\mR_+\times\mR^d}(-\pt_t V + \Div ( V \nabla \Psi ) + \Delta V)\zeta dtdx\\
=&\int_{\mR_+\times\mR^d}V(\pt_t \zeta-\nabla\Psi\cdot\nabla\zeta+\Delta\zeta)dtdx
\end{split}
\end{equation}
For any $\xi\in\cC_0^\iy(\mR_+\times\mR^d)$, find $\zeta\in\cC_0^\iy(\mR_+\times\mR^d)$ such that
\[\xi=\pt_t\zeta-\nabla\Psi\cdot\nabla\zeta+\Delta\zeta,\]
then we get the desired result. The linear dependence is straightforward. 

By the Fokker-Planck equation of $\tilde{\rho}^n$, we can obtain the similar equality as (\ref{eq-td1}), where we can replace $\hat{\rho}_{\delta}^n$ with $\tilde{\rho}_n$ and remove the $O_P(\delta)$ term. Then use the same argument we can prove $\tilde{V}^n$ also converges to $V$ as $n\ra\iy$. 

\section{Proof in Section 5}
\label{appB}

Assume
\begin{equation*}
\hat{\rho}^{(k)}={\rm arg}\min_{\rho\in \Xi} \left\{\frac{1}{2}\left[d_W\left(\rho,\hat{\rho}^{(k-1)}\right)\right]^2+\delta \mF_{\hat{\Psi}_k}(\rho)\right\}
\ ,\hat{\rho}^{(0)}=\rho^0
\end{equation*}
and 
\begin{equation*}
\hat{\rho}_\delta(t,x)=\hat{\rho}^{(\lceil t/\delta\rceil)}(x)
\end{equation*}
First we give an important lemma on $\hat{\rho}_\delta$:
\begin{lem}\label{thm-41}
Assume $\{\hat{\Psi}_k\}$ satisfies (A5), (A6), $\delta m^{1/(2-2\gamma)}>C>0$, then for any $\zeta \in \cC^\infty_0(\mR_+\times \mR^d)$, we have
\[
\left|-\int_{\mR^d}\zeta(0,x)\rho^0dx-\int_{\mR^d}\int_0^{\iy}\hat{\rho}_\delta(\pt_t\zeta-\nabla\hat{\Psi}_{\lceil t/\delta\rceil}\cdot\nabla\zeta+\Delta\zeta) dtdx\right|
=O_P(\delta).
\]
\end{lem}

\begin{proof}
Assume w.l.o.g. $\beta=1$, denote $t_k=k\delta$. For any $\zeta\in \cC_0^\iy(\mR_+\times\mR^d)$ and $k\geq 1$, applying Proposition \ref{prop-ip} on $\hat{\rho}^{(k-1)},\hat{\rho}^{(k)},\hat{\Psi}_k$ and $\zeta(t_{k-1},\cdot)$, we get
\begin{equation}\label{eq-thm1}
\begin{split}
&\left|\int_{\mR^d}\left\{(\hat{\rho}^{(k)}-\hat{\rho}^{(k-1)})\zeta(t_{k-1},x)+\delta(\nabla\hat{\Psi}_k\cdot\nabla\zeta(t_{k-1},x)-\Delta\zeta(t_{k-1},x))\hat{\rho}^{(k)}\right\}dx\right|\\
\leq\ &\frac{1}{2}\sup_{x\in\mR^d}|\nabla^2\zeta(t_{k-1},x)|d_W^2(\hat{\rho}^{(k-1)},\hat{\rho}^{(k)}).
\end{split}
\end{equation}
Since $\zeta$ has compact support, assume the support is a subset of $[0,T]\times D$, where $T>1$, $D\subset\mR^d$ is a compact set. Let $N_\delta=\lceil T/\delta\rceil$, then
\begin{equation}\label{eq-thm2}
\begin{split}
&\left|\sum_{k=1}^{N_\delta}\int_{\mR^d}\left\{(\hat{\rho}^{(k)}-\hat{\rho}^{(k-1)})\zeta(t_{k-1},x)+\delta(\nabla\hat{\Psi}_k\cdot\nabla\zeta(t_{k-1},x)-\Delta\zeta(t_{k-1},x)) \hat{\rho}^{(k)}\right\}dx\right|\\
\leq\ &\frac{1}{2}\sup_{t,x}|\nabla^2\zeta|\sum_{k=1}^{N_\delta}d_W^2(\hat{\rho}^{(k-1)},\hat{\rho}^{(k)})
\end{split}
\end{equation}
For fix $x$, we have
\begin{equation}\label{eq-thm3}
\begin{split}
&\sum_{k=1}^{N_\delta}(\hat{\rho}^{(k)}-\hat{\rho}^{(k-1)})\zeta(t_{k-1},x)\\
=&-\zeta(0,x)\hat{\rho}^{(0)}-\sum_{k=1}^{N_\delta}(\zeta(t_k,x)-\zeta(t_{k-1},x))\hat{\rho}^{(k)}\\
=&-\zeta(0,x)\hat{\rho}^{(0)}-\int_0^{\iy}\pt_t\zeta(t,x)\hat{\rho}_\delta(t,x)dt
\end{split}
\end{equation}
Since any order derivative of $\zeta$ is continuous with compact support, and $\{\hat{\Psi}_k\}$ satisfies (A5), 
\begin{equation}\label{eq-thm4}
\begin{split}
&\left|\int_{\mR^d}\left[\sum_{k=1}^{N_\delta}\left[\delta(\nabla\hat{\Psi}_k\cdot\nabla\zeta(t_{k-1},x)-\Delta\zeta(t_{k-1},x)) \hat{\rho}^{(k)}\right]-\int_0^{\iy}(\nabla\hat{\Psi}_{\lceil t/\delta\rceil}\cdot\nabla\zeta-\Delta\zeta)\hat{\rho}_\delta dt\right]dx\right|\\
&\leq\int_{\mR^d}\sum_{k=1}^{N_\delta}\int_{t_{k-1}}^{t_k}|\nabla\hat{\Psi}_k\cdot\nabla\zeta(t_{k-1},x)-\Delta\zeta(t_{k-1},x)-(\nabla\hat{\Psi}_k\cdot\nabla\zeta-\Delta\zeta)|\hat{\rho}^{(k)}dtdx\\
&\leq\int_{\mR^d}\sum_{k=1}^{N_\delta}\int_{t_{k-1}}^{t_k} C\delta\left(\max_{x\in D}|\nabla\hat{\Psi}_k(x)|+1\right)\hat{\rho}^{(k)}dtdx\\
&\leq \sum_{k=1}^{N_\delta}C\delta^2\left(\max_{x\in D}|\nabla\hat{\Psi}_k(x)|+1\right)\\
&=O_P(\delta)
\end{split}
\end{equation}

According to the definition of $\hat{\rho}^{(k)}$, we have
\[\frac{1}{2}d_W^2(\hat{\rho}^{(k-1)},\hat{\rho}^{(k)})\leq\delta(\mF_{\hat{\Psi}_k}(\hat{\rho}^{(k-1)})-\mF_{\hat{\Psi}_k}(\hat{\rho}^{(k)})),\]
for any $1\leq N\leq N_\delta$, we have
\[\begin{split}
&\frac{1}{2}\sum_{k=1}^{N}d_W^2(\hat{\rho}^{(k-1)},\hat{\rho}^{(k)})\\
\leq&\delta\sum_{k=1}^{N}(\mF_{\hat{\Psi}_k}(\hat{\rho}^{(k-1)})-\mF_{\hat{\Psi}_k}(\hat{\rho}^{(k)}))\\
\leq&\delta\left(\mF_{\hat{\Psi}_1}(\hat{\rho}^{(0)})-\mF_{\hat{\Psi}_{N}}(\hat{\rho}^{(N)})+\sum_{k=1}^{N-1}(\cE_{\hat{\Psi}_{k+1}}(\hat{\rho}^{(k)})-\cE_{\hat{\Psi}_k}(\hat{\rho}^{(k)}))\right)
\end{split}\]
Since $\Psi\geq 0$, by (14) in \cite{jordan1998variational}
\[\mF_\Psi(\rho)\geq -\cS(\rho)\geq -C(\cM(\rho)+1)^\alpha\]
where $C$ is a constant that depends on $d$, the dimension of $x$, and $\alpha\in\left(\frac{d}{d+2},1\right)$.
Since $\{\hat{\Psi}_k\}$ satisfies (A6), let
\[\cC=\max_{x\in\mR^d, 1\leq k< \lceil T/\delta\rceil}\frac{k^\gamma m^{1/2}}{(1+|x|^2)}|\hat{\Psi}_{k+1}(x)-\hat{\Psi}_k(x)|=O_P(1),\]
then
\[\left|\cE_{\hat{\Psi}_{k+1}}(\hat{\rho}^{(k)})-\cE_{\hat{\Psi}_k}(\hat{\rho}^{(k)})\right|\leq\int_{\mR^d}\cC m^{-1/2}k^{-\gamma}(1+|x|^2)\hat{\rho}^{(k)}(x)dx=\cC m^{-1/2}k^{-\gamma}(1+\cM(\hat{\rho}^{(k)})),\]
therefore
\begin{equation}\label{eq-thm5}
\frac{1}{2}\sum_{k=1}^{N}d_W^2(\hat{\rho}^{(k-1)},\hat{\rho}^{(k)})
\leq\delta\left(|\mF_{\hat{\Psi}_1}(\hat{\rho}^{(0)})|+C(\cM(\hat{\rho}^{(N)})+1)^\alpha+\cC m^{-1/2}\sum_{k=1}^{N-1}k^{-\gamma}(1+\cM(\hat{\rho}^{(k)}))\right)
\end{equation}
In addition, we have
\begin{equation}\label{eq-thm6}
\cM(\hat{\rho}^{(N)})\leq 2d_W^2(\hat{\rho}^{(0)},\hat{\rho}^{(N)})+2\cM(\hat{\rho}^{(0)})\leq 2N\sum_{k=1}^Nd_W^2(\hat{\rho}^{(k-1)},\hat{\rho}^{(k)})+2\cM(\hat{\rho}^{(0)})
\end{equation}
Denote $a_N=\cM(\hat{\rho}^{(N)})+1$, $\cC=O_P(1)$ is bounded in probability, which may differ from each other. Note that $|\mF_{\hat{\Psi}_1}(\hat{\rho}^{(0)})|=O_P(1)$ and $N\delta$ is bounded. Combining (\ref{eq-thm5}),(\ref{eq-thm6}), we get
\[a_N\leq \cC(a_N^\alpha+ m^{-1/2}\sum_{k=1}^{N-1}k^{-\gamma}a_k)+\cC\]
Since $a_1\leq \cC a_1^\alpha+\cC$ and $\alpha<1$, so $a_1\leq\cC'$, update $\cC=\max\{\cC,\cC'\}$. By Young's inequality, we get
\[a_N\leq \cC^{1/(1-\alpha)}+(\cC m^{-1/2}\sum_{k=1}^{N-1}k^{-\gamma}a_k+\cC)/(1-\alpha)\] 
We update the constant $\cC$ and rewrite it as
\[a_N\leq \cC+\cC m^{-1/2}\sum_{k=1}^{N-1}k^{-\gamma}a_k\]
Since $\delta m^{1/(2-2\gamma)}>C$, $N\leq T/\delta$, we have $N^{1-\gamma} m^{-1/2}=O(1)$. Let $S_N=\sum_{k=1}^N k^{-\gamma}a_k$, then
\[
\begin{split}
S_N&\leq \cC N^{-\gamma}+(1+\cC N^{-\gamma}m^{-1/2})S_{N-1} \\
&\leq \left[\prod_{k=1}^N(1+\cC k^{-\gamma}m^{-1/2})\right]\sum_{k=1}^N \cC k^{-\gamma} \\
&\leq \exp\left[\sum_{k=1}^N \cC k^{-\gamma}m^{-1/2}\right] \cC N^{1-\gamma} \\
&\leq \cC N^{1-\gamma}
\end{split}
\]
then
\[a_N\leq \cC+\cC m^{-1/2} S_N=O_P(1)\]
from (\ref{eq-thm5}),
\begin{equation}\label{eq-thm7}
\frac{1}{2}\sum_{k=1}^{N_\delta}d_W^2(\hat{\rho}^{(k-1)},\hat{\rho}^{(k)})\leq \delta(O_P(1)+\cC m^{-1/2} S_{N_\delta-1})=O_P(\delta)
\end{equation}
Combining (\ref{eq-thm2}),(\ref{eq-thm3}),(\ref{eq-thm4}),(\ref{eq-thm7}), we get
\[
\left|-\int_{\mR^d}\zeta(0,x)\hat{\rho}^{(0)}dx-\int_{\mR^d}\int_0^{\iy}\hat{\rho}_\delta(\pt_t\zeta-\nabla\hat{\Psi}_{\lceil t/\delta\rceil}\cdot\nabla\zeta+\Delta\zeta) dtdx\right|
=O_P(\delta),
\]
which gives the desired result.
\end{proof}

\begin{remark}
For condition (A6), in general case, since
\[\nabla_x\Psi_{\hat{\theta}}(x)\approx \nabla_x\Psi_\theta(x)+\nabla_\theta'\nabla_x\Psi_\theta(x)(\hat{\theta}-\theta)\]
We have the following approximation of $\hat{\theta}^k$ with (hopefully) accuracy $(mk)^{-1}$:
\[\hat{\theta}^k\approx \theta-\left[\sum_{j=1}^k A_j\right]^{-1}\left[\sum_{j=1}^k b_j\right]\]
where $A_j=\sum_{i=1}^m \nabla_\theta'\nabla_x\Psi_\theta(X^{(j)}(i\eta))$, $b_j=\sum_{i=1}^m \nabla_x\Psi_\theta(X^{(j)}(i\eta))$. As $m\ra\iy$, $A_j/m\ra A=E_\pi(\nabla_\theta'\nabla_x\Psi_\theta(X))$ is invertible. By CLT $b_j=O_P(m^{1/2})$. Then
\[\hat{\theta}^k-\theta=O_P(km^{1/2}/(km))=O_P(m^{-1/2})\]
Since
\[ \hat{\theta}^k-\hat{\theta}^{k+1}=\left[\sum_{j=1}^{k+1} A_j\right]^{-1}[b_{k+1}+A_{k+1}(\hat{\theta}^k-\theta)]\]
and $b_{k+1}+A_{k+1}(\hat{\theta}^k-\theta)=O_P(m^{1/2})$, $\sum_{j=1}^{k+1} A_j=\Omega_P(mk)$, we arrive at 
\[\hat{\theta}^k-\hat{\theta}^{k+1}=O_P(k^{-1}m^{-1/2}).\]
And (A6) is satisfied for $\gamma=1$.
\end{remark}

In the next lemma, we will control the difference between  $\nabla_x\Psi_{\hat{\theta}^k}(x)$ and $\nabla_x\Psi_\theta(x)$,

\begin{lem}\label{lem5-1}
Assume there exists $\alpha>0$, $\left|E(\sum_{i=1}^m\nabla_x\Psi_{\theta}(X(i\eta)))\right|=O(m^{-\alpha})$, and $k=o(m^{1+2\alpha})$, we have
\begin{equation}\label{eqq-51}
\sqrt{mk}(\nabla_x\Psi_{\hat{\theta}^k}(x)-\nabla_x\Psi_\theta(x))-\sqrt{k}\nabla\tau(x)\bar{\bZ}_k\xrightarrow{L_2}0
\end{equation}
uniformly for $x\in D$, here $D$ is a compact set.
\end{lem}

\begin{proof}
First denote $X_i^{(j)}=X^{(j)}(i\eta)$, then for any $j$, we have
\[-m^{-1/2}\sum_{i=1}^m\nabla_x\Psi_\theta(X_i^{(j)})\ra \sigma_f\bZ_j\]
where $\bZ_j$ are mutually independent standard normal random vectors, $\sigma_f=\Sigma_f^{1/2}$ defined in Lemma \ref{prop-th}. By Schorohod's Theorem, we can establish the $L_2$ and almost surely convergence. Denote $\ve_{m,j}=-m^{-1/2}\sum_{i=1}^m\nabla_x\Psi_\theta(X_i^{(j)})-\sigma_f\bZ_j$, then $\ve_{m,j}$ are i.i.d for different $j$, and
\[|E\ve_{m,j}|=O(m^{-1/2-\alpha}),\ \  E|\ve_{m,j}|^2=o(1)\]
Then
\begin{equation}\label{eqq511}
\begin{aligned}
E\left|k^{-1/2}\sum_{j=1}^k\ve_{m,j}\right|^2&\leq k^{-1}\sum_{j=1}^kE|\ve_{m,j}|^2+2k^{-1}\sum_{j<l\leq k}|E\ve_{m,j}|\cdot|E\ve_{m,l}|\\
&=o(1)+O(km^{-1-2\alpha})\\
&=o(1)
\end{aligned}
\end{equation}
Together with the equation of the estimator $\hat{\theta}^k$, we get 
\begin{equation}\label{eqq-54}
(mk)^{-1/2}\sum_{j=1}^k\sum_{i=1}^m(\nabla_x\Psi_{\hat{\theta}^k}(X_i^{(j)})-\nabla_x\Psi_\theta(X_i^{(j)}))-\sqrt{k}\sigma_f\bar{\bZ}_k\xrightarrow{L_2}0
\end{equation}
uniformly for $k=o(m^{1+2\alpha})$, since 
\[\nabla_x\Psi_{\hat{\theta}^k}(X_i^{(j)})-\nabla_x\Psi_\theta(X_i^{(j)}))=\nabla_\theta'\nabla_x\Psi_\theta(X_i^{(j)})(\hat{\theta}^k-\theta)+o_p(\hat{\theta}^k-\theta)\]
and
\begin{equation}\label{eqq-551}
(mk)^{-1}\sum_{j=1}^k\sum_{i=1}^m\nabla_\theta'\nabla_x\Psi_\theta(X_i^{(j)})=A+o_p(1)
\end{equation}
where $A=E_\pi[\nabla_\theta'\nabla_x\Psi_\theta(X)]$ is invertible as defined in Lemma \ref{prop-th}. we get
\[
\sqrt{mk}(\hat{\theta}^k-\theta)-\sqrt{k}A^{-1}\sigma_f\bar{\bZ}_k\xrightarrow{L_2}0,
\]
then for $x\in D$,
\[
\sqrt{mk}\nabla_\theta'\nabla_x\Psi_\theta(x)(\hat{\theta}^k-\theta)-\sqrt{k}\nabla\tau(x)\bar{\bZ}_k\xrightarrow{L_2}0,
\]
here we use $\nabla\tau(x)=\nabla_\theta'\nabla_x\Psi_\theta(x)A^{-1}\sigma_f$.
Since $\sqrt{mk}(\hat{\theta}^k-\theta)=O_P(1)$,
\[
\sqrt{mk}(\nabla_x\Psi_{\hat{\theta}^k}(x)-\nabla_x\Psi_\theta(x)-\nabla_\theta'\nabla_x\Psi_\theta(x)(\hat{\theta}^k-\theta))=\sqrt{mk}o_p(\hat{\theta}^k-\theta)=o_p(1)
\]
Then we arrive at
\[
\sqrt{mk}(\nabla_x\Psi_{\hat{\theta}^k}(x)-\nabla_x\Psi_\theta(x))-\sqrt{k}\nabla\tau(x)\bar{\bZ}_k\xrightarrow{L_2}0,
\]
which is the desired result.
\end{proof}

Now we can give the proof of Theorem 5.1,

\subsection{Proof of Theorem 5.1}
Assume any $\zeta\in\cC_0^\iy(\mR_+\times\mR^d)$ with compact support $[0,T]\times D$ and w.l.o.g. $\beta=1$. Given (A5),(A6), by Lemma \ref{thm-41} we have
\[
\int_{\mR_+\times\mR^d}\hat{\rho}_{\delta,1}^{m}(\pt_t\zeta-\nabla_x\hat{\Psi}_{\lceil t/\delta\rceil}\cdot\nabla\zeta+\Delta\zeta) dtdx=-\int_{\mR^d}\zeta(0,x)\rho^0dx+O_P(\delta),
\]
since
\[
\int_{\mR_+\times\mR^d}\rho(\pt_t\zeta-\nabla\Psi\cdot\nabla\zeta+\Delta\zeta) dtdx=-\int_{\mR^d}\zeta(0,x)\rho^0dx
\]
we get
\begin{equation}\label{eq520n1}
\begin{split}
& \int_{\mR_+\times\mR^d}(\hat{\rho}_{\delta,1}^{m}-\rho)(\pt_t\zeta-\nabla\Psi\cdot\nabla\zeta+\Delta\zeta) dtdx \\
=&\int_{\mR_+\times \mR^d}\hat{\rho}_{\delta,1}^{m}(\nabla_x\hat{\Psi}_{\lceil t/\delta\rceil}-\nabla\Psi) \cdot\nabla\zeta dtdx+O_P(\delta).
\end{split}
\end{equation}
then
\begin{equation}\label{eq520n2}
\begin{split}
& \int_{\mR_+\times\mR^d}\hat{V}_{\delta,1}^m(\pt_t \zeta- \nabla\Psi\cdot\nabla\zeta +\Delta\zeta))dt dx  \\
=& \int_{\mR_+\times\mR^d}\rho(m/\delta)^{1/2}(\nabla_x\hat{\Psi}_{\lceil t/\delta\rceil}-\nabla\Psi) \cdot\nabla\zeta dtdx+\\
&\int_{\mR_+\times \mR^d}[\hat{\rho}_{\delta,1}^{m}-\rho](m/\delta)^{1/2}(\nabla_x\hat{\Psi}_{\lceil t/\delta\rceil}-\nabla\Psi) \cdot\nabla\zeta dtdx+O_P((m\delta)^{1/2}).
\end{split}
\end{equation}

By Lemma \ref{lem5-1}, since $T/\delta=o(m^{1+2\alpha})$, for any $\ve>0$, $\exists N$, when $m>N$ and $k\leq T/\delta$, we have
\[E\left|m^{1/2}(\nabla_x\hat{\Psi}_{k}-\nabla\Psi)-\nabla\tau(x)\bar{\bZ}_k\right|<\ve k^{-1/2}\]
Then
\begin{equation}\label{eq520n3}
\begin{split}
&E\left|\int_{[0,T]\times D}\rho(m/\delta)^{1/2}(\nabla_x\hat{\Psi}_{\lceil t/\delta\rceil}-\nabla\Psi) \cdot\nabla\zeta dtdx-\sum_{k=1}^{\lceil T/\delta\rceil}\int_{t_{k-1}}^{t_k}\int_D[\rho\delta^{-1/2}\nabla\tau(x)\bar{\bZ}_k]\cdot\nabla\zeta dxdt\right|\\
\leq&\sum_{k=1}^{\lceil T/\delta\rceil}\int_{t_{k-1}}^{t_k}\int_D\rho\delta^{-1/2}E\left|m^{1/2}(\nabla_x\hat{\Psi}_{k}-\nabla\Psi)-\nabla\tau(x)\bar{\bZ}_k\right| \cdot|\nabla\zeta| dxdt\\
\leq&\sum_{k=1}^{\lceil T/\delta\rceil}C\delta^{1/2}\ve k^{-1/2}\\
\leq&C'\ve
\end{split}
\end{equation}

Let $W_\delta(t)=\sqrt{\delta}\sum_{k=1}^{\lceil t/\delta\rceil}\bZ_k$, Since $\{\bZ_k\}$ are i.i.d standard normal vectors, so there exists a $d$-dimensional Brownian motion $\{W(t)\}$, such that $\{W_\delta(t)\}$ weakly converge to $\{W(t)\}$ on $\cC([0,T])$, by Skorokhod representation theorem, we can assume the almost surely convergence. Then as $\delta\ra 0$, 
\begin{equation}\label{eq44-4}
\begin{split}
& \sum_{k=1}^{\lceil T/\delta\rceil}\int_{t_{k-1}}^{t_k}\int_D[\rho\delta^{-1/2}\nabla\tau\bar{\bZ}_k]\cdot\nabla\zeta dxdt  \\
=& \sum_{k=1}^{\lceil T/\delta\rceil}\int_{t_{k-1}}^{t_k}\left\{\int_D(\nabla\zeta)'\rho\nabla\tau(x)dx\right\}\frac{W_\delta(k\delta)}{k\delta}dt\\
\ra&\int_0^T\left\{\int_D(\nabla\zeta)'\rho\nabla\tau(x)dx\right\}t^{-1}W(t)dt\\
=&\int_{\mR_+\times\mR^d}-\Div(\rho\nabla\tau(x)t^{-1}W(t))\zeta dtdx\\
=&\int_{\mR_+\times\mR^d}(-\pt_t V_1 + \Div ( V_1 \nabla \Psi ) + \Delta V_1)\zeta dtdx\\
=&\int_{\mR_+\times\mR^d}V_1(\pt_t \zeta-\nabla\Psi\cdot\nabla\zeta+\Delta\zeta)dtdx
\end{split}
\end{equation}
For the convergence in the third line of (\ref{eq44-4}), we will give an explicit proof. First let 
\[g(t)=\int_D(\nabla\zeta)'\rho(t,x)\nabla\tau(x)dx=-\int_D\Div[\rho(t,x)\nabla\tau(x)]\zeta dx\]
Then given (A8), there exists $a<1/2$,
\[|g(t)|\leq C\int_D|\Div[\rho(t,x)\nabla\tau(x)]| dx=O(t^{-a})\]
Let $b\in (a,1/2)$, then $\sup_{[0,T]}\frac{|W(t)|}{t^b}=O_P(1)$, and 
\[\left|\int_0^\eta g(t)t^{-1}W(t)dt\right|=\int_0^\eta O_P(t^{b-a-1})dt =o_P(1)\]
as $\eta\ra 0$. For the second line of (\ref{eq44-4}), we have
\begin{equation}\label{eqb-1}
\begin{split}
&\left|\int_0^\eta g(t)\frac{W_\delta(t)}{\lceil t/\delta\rceil\delta} dt\right| \\
=&\left|\int_0^\delta g(t)\frac{\sqrt{\delta}\bZ_1}{\delta} dt\right|+\left|\int_\delta^\eta g(t)\frac{W_\delta(t)}{\lceil t/\delta\rceil\delta} dt\right| \\
\leq &O_P(\delta^{-1/2})\int_0^\delta Ct^{-a} dt+\int_\delta^\eta Ct^{-a}\frac{O_P(t^{1/2})}{t} dt \\
=& O_P(\delta^{1/2-a})+O_P(\eta^{1/2-a}) \\
=&o_P(1)
\end{split}
\end{equation}
uniformly for $\delta<\eta$ as $\eta\ra 0$. In (\ref{eqb-1}), we used $E|W_\delta(t)|^2\leq 2dt$ for $\delta<t$, so $W_\delta(t)=O_P(t^{1/2})$ uniformly for $\delta$. Thus for any $\ve,\nu>0$, we can find fixed $\eta>0$, such that for any $\delta<\eta$,
\[\mP\left(\left|\int_0^\eta g(t)t^{-1}W(t)dt\right|>\nu/3\right)<\ve/3,\]
\[\mP\left(\left|\int_0^\eta g(t)\frac{W_\delta(t)}{\lceil t/\delta\rceil\delta}dt\right|>\nu/3\right)<\ve/3,\]
and since $W_\delta(t)\ra W(t)$ as $\delta\ra 0$, we can find $\delta_u<\eta$, when $\delta<\delta_u$, 
\[\mP\left(\left|\int_\eta^T g(t)\left(\frac{W_\delta(t)}{\lceil t/\delta\rceil\delta}-\frac{W(t)}{t}\right)dt\right|>\nu/3\right)<\ve/3,\]
therefore when $\delta<\delta_u$, 
\[\mP\left(\left|\int_0^T g(t)\frac{W_\delta(t)}{\lceil t/\delta\rceil\delta}-\int_0^T g(t)\frac{W(t)}{t}dt\right|>\nu\right)<\ve.\]

The second term in (\ref{eq520n2}) is a higher-order term, we have
\[\int_{\mR_+\times \mR^d}[\hat{\rho}_{\delta,1}^{m}-\rho](m/\delta)^{1/2}(\nabla_x\hat{\Psi}_{\lceil t/\delta\rceil}-\nabla\Psi) \cdot\nabla\zeta dtdx=o_p(1)\]
Then combining (\ref{eq520n2}), (\ref{eq520n3}), (\ref{eq44-4}), as $\delta m\ra 0$, we have
\[\int_{\mR_+\times\mR^d}(\hat{V}_{\delta,1}^m-V_1)(\pt_t \zeta- \nabla\Psi\cdot\nabla\zeta +\Delta\zeta))dtdx=o_p(1).\]
For any $\xi\in\cC_0^\iy(\mR_+\times\mR^d)$, find $\zeta\in\cC_0^\iy(\mR_+\times\mR^d)$ such that
\[\xi=\pt_t\zeta-\nabla\Psi\cdot\nabla\zeta+\Delta\zeta,\]
then we get the desired result.

Next we will show the existence and uniqueness of the solution of $V_1$. First I will briefly introduce the solution to the heat equation and the relating SPDE. When solving the heat equation
\[
\pt_t u=\Delta u,\ \ , u(0,x)=u_0(x)
\]
we can use the heat semigroup $e^{t\Delta}u_0(x)=\int_{\mR^d} H(t,x-y)u_0(y)dy$ with heat kernel $H(t,x)=(4\pi t)^{-\frac{d}{2}}e^{-\frac{|x|^2}{4t}}$. For the heat SPDE
\[
\pt_t u=\Delta u+\xi(t,x),\ \ , u(0,x)=u_0(x)
\]
$\xi(t,x)$ is a space-time random process, e.g., white noise. The solution is defined by
\[
u=(\pt_t-\Delta)^{-1}\xi+e^{t\Delta}u_0
\]
where the inverted linear differential operator is defined by
\[
(\pt_t-\Delta)^{-1}\xi(t,x)=\int_0^t\int_{\mR^d}H(t-s,x-y)\xi(ds,dy)
\]
Now turn to the equation of $V_1(t,x)$ below,
\begin{equation}\label{Equa-V}
\pt_t V_1=\Div(V_1\nabla\Psi)+\Div[\rho\nabla \tau(x)t^{-1}W(t)]+\Delta V_1, V_1(0,x)=0,
\end{equation}

analog to the heat kernel $H(t,x)$, we consider the solution to the Fokker-Planck equation:
\[
\pt_t P=\Div(P\nabla\Psi)+\Delta P, \ \ \ P(0,\cdot)=\delta_0
\]
Then the solution with initial distribution $\rho_0$ can be written as
\begin{equation}\label{eqa-0}
\rho(t,x)=\int_{\mR^d} P(t,x-y)\rho_0(y)dy
\end{equation}
And we can give the explicit expression of $V_1$:
\begin{equation}\label{eqa-1}
V_1(t,x)=\int_0^t\int_{\mR^d} P(t-s,x-y)\Div[\rho(s,y)\nabla \tau(y)s^{-1}W(s)]dsdy
\end{equation}
We can show that $V_1$ is a.s. finite. By (A8), we can choose $a<1/2$, such that for any $s<t$,  
\[\int_{\mR^d} |\Div[\rho(s,y)\nabla \tau(y)]|dy=O(s^{-a})\]
Then we can show 
\[\int_{\mR^d} P(t-s,x-y)|\Div[\rho(s,y)\nabla \tau(y)]|dy=O(s^{-a})\]
Since when $t-s>\eta>0$, $P(t-s,x-y)$ is bounded, and when $s\ra t$, The LHS converges to $|\Div[\rho(t,x)\nabla \tau(x)]|$ which is bounded. Let $b\in (a,1/2)$, then $\sup_{[0,T]}\frac{|W(s)|}{s^b}=O_P(1)$, and
\[
\begin{aligned}
&\left|\int_0^t\left[\int_{\mR^d} P(t-s,x-y)\Div[\rho(s,y)\nabla \tau(y)]dy\right]s^{-1}W(s)ds\right| \\
\leq &C\int_0^t s^{-a}s^{-1}\cC s^b ds \\
=& O_P(1)
\end{aligned}
\]
If we replace $P$ with $\pt_t P, \nabla P, \Delta P$ in (\ref{eqa-1}), the resulting integral is also a.s. finite. Then the direct calculation gives:
\[
\begin{aligned}
\pt_t V_1&=\Div[\rho(t,x)\nabla \tau(x)t^{-1}W(t)]+\int_0^t\int_{\mR^d} \pt_t P(t-s,x-y)\Div[\rho(s,y)\nabla \tau(y)s^{-1}W(s)]dsdy \\
&=\Div[\rho(t,x)\nabla \tau(x)t^{-1}W(t)]+\Div(V_1\nabla\Psi)+\Delta V_1
\end{aligned}
\]
Finally, the difference of two solutions to (\ref{Equa-V}) will solve the original Fokker-Planck equation with $\rho_0=0$. By the explicit solution of Fokker-Planck equation (\ref{eqa-0}), The solution will be zero when the initial value is zero, so the solution to (\ref{Equa-V}) must be unique, see also \cite{otto1996l1}.

If we only assume (A7),(A8), then by the Fokker-Planck equation of $\tilde{\rho}_1^m$, we can obtain the following equality, which is similar to (\ref{eq520n3}) without the $O_P(\delta)$ term.
\begin{equation}\label{eq520nn}
\begin{split}
& \int_{\mR_+\times\mR^d}(\tilde{\rho}_{1}^{m}-\rho)(\pt_t\zeta-\nabla\Psi\cdot\nabla\zeta+\Delta\zeta) dtdx \\
=&\int_{\mR_+\times \mR^d}\tilde{\rho}_{1}^{m}(\nabla_x\hat{\Psi}_{\lceil t/\delta\rceil}-\nabla\Psi) \cdot\nabla\zeta dtdx.
\end{split}
\end{equation}
Then use the same argument we can prove $\tilde{V}_1^m$ also converges to $V_1$ as $\delta m^{1+2\alpha}\ra\iy$. (Note that here we do not need to assume $\delta m\ra 0$ and $\delta m^{1/(2-2\gamma)}>C>0$). Furthermore, if we do not assume (A7) (only assume (A8)), then in (\ref{eq520n3}), we have
\[m^{1/2}(\nabla_x\hat{\Psi}_{k}-\nabla\Psi)-\nabla\tau(x)\bar{\bZ}_k=O_P(m^{-1/2}),\]
then as $\delta\ra 0, m\ra\iy, \delta m\ra\iy$, $\tilde{V}_1^m$ still converges to $ V_1$.

\begin{remark}
Assumption (A7): $\left|E(\sum_{i=1}^m\nabla_x\Psi_{\theta}(X(i\eta)))\right|=O(m^{-\alpha})$ is important to control the bias $\hat{\theta}^k-\theta$ in a small order. Here is a case when (A7) is also a necessary condition. Consider $\Psi_\theta(x)=\frac{1}{2}|x-\theta|^2$, in this case, $\hat{\theta}^k=\bar{X}$ which is the average of all $mk$ samples. Since $E(X(t))=\theta+e^{-t}(E(X_0)-\theta)$, 
\[E(\hat{\theta}^k)=\theta+\frac{1}{m}\sum_{i=1}^m e^{-i\eta}(E(X_0)-\theta)=\theta+\frac{e^{-\eta}(1-e^{-m\eta})}{m(1-e^{-\eta})}(E(X_0)-\theta).\]
If $E(X_0)=\theta$, then (A7) is true for arbitrary $\alpha$; if $E(X_0)\neq\theta$, then
\[E\left(\sum_{i=1}^m\nabla_x\Psi_{\theta}(X(i\eta))\right)=m(E(\hat{\theta}^k)-\theta)=\Theta(1)\]
We can only take $\alpha=0$ in (A7), which means $\delta m$ must converge to infinity. From another aspect,
\[E(m^{1/2}(\nabla_x\hat{\Psi}_{k}-\nabla\Psi)-\nabla\tau(x)\bar{\bZ}_k)=m^{1/2}E(\theta-\hat{\theta}^k)=\Theta(m^{-1/2})\]
This is the same as the case when we only use $m$ samples to estimate $\theta$ (it means using $km$ samples to estimate $\theta$ does not decrease the bias). As a result, in (\ref{eq520n3}), we will need $\delta m\ra\iy$, which is a contradicted with $\delta m\ra 0$. So $E(X_0)=\theta$ in this case is also a necessary condition.

We can solve the SPDE of $V_1$ in certain examples. e.g. If $d=1, \Psi(x)=\frac{1}{2}x^2, \beta=1$, the solution of the Langevin equation is $X_t=e^{-t}X_0+\sqrt{2}e^{-t}\int_0^t e^sdB_s$. When $X_0=0$, $X_t$ follows Gaussian distribution $N(0, 1-e^{-2t})$. So
\[P(t,x)=[2\pi(1-e^{-2t})]^{-1/2}\exp{\left(-\frac{x^2}{2(1-e^{-2t})}\right)}\]
For simplicity, assume $\rho^0\sim\pi=N(0,1)$, then $\rho(t,\cdot)\sim N(0,1)$, $\nabla\tau(x)=-\gamma_\theta=-\sqrt{\frac{1+e^{-\eta}}{1-e^{-\eta}}}$. Then the explicit solution of $V_1$ can be written as:
\[V_1(t,x)=\int_0^t\int_{\mR} (2\pi)^{-1}\left[1-e^{-2(t-s)}\right]^{-1/2}\exp{\left(-\frac{(x-y)^2}{2(1-e^{-2(t-s)})}-\frac{y^2}{2}\right)}y\gamma_\theta s^{-1}W(s)dyds\]
Since
\[\int_{\mR}y\exp{\left(-\frac{(x-y)^2}{2\sigma^2}-\frac{y^2}{2}\right)}dy=(2\pi\sigma^2)^{1/2}(\sigma^2+1)^{-3/2}x\exp{\left(-\frac{x^2}{2(\sigma^2+1)}\right)}\]
We obtain
\[V_1(t,x)=x\gamma_\theta(2\pi)^{-1/2}\int_0^t \left[2-e^{-2(t-s)}\right]^{-3/2}\exp{\left(-\frac{x^2}{2(2-e^{-2(t-s)})}\right)} s^{-1}W(s)ds\]
Therefore $V_1(t,x)$ follows Gaussian distribution and $V_1(t,-x)=-V_1(t,x)$. We show the graph of the function $V_1$ for an example in Section 7 of the main manuscript.
\end{remark}




\subsection{Proof of Theorem 5.2}
\label{appB6}
Assume $\beta=1$, for any $\zeta\in\cC_0^\iy(\mR_+\times\mR^d)$, the support of $\zeta$ is a subset of $[0,T]\times D$, where $D$ is a compact subset of $\mR^d$. By Lemma \ref{thm-41},
\begin{equation}\label{eq44-0}
\int_{\mR_+\times\mR^d}\hat{\rho}_{\delta,2}^{m}(\pt_t\zeta-\nabla\hat{\Psi}_m^{(\lceil t/\delta\rceil)}\cdot\nabla\zeta+\Delta\zeta) dtdx=-\int_{\mR^d}\zeta(0,x)\rho^0dx+O_P(\delta),
\end{equation}
together with (\ref{eq-38}), we get
\begin{equation}\label{eq44-1}
\begin{split}
& \int_{\mR_+\times\mR^d}\hat{V}_{\delta,2}^m(\pt_t \zeta- \nabla\Psi\cdot\nabla\zeta +\Delta\zeta))dt dx  \\
=& \int_{\mR_+\times\mR^d}\rho(m/\delta)^{1/2}(\nabla\hat{\Psi}_m^{(\lceil t/\delta\rceil)}-\nabla\Psi) \cdot\nabla\zeta dtdx+\\
&\int_{\mR_+\times \mR^d}[\hat{\rho}_{\delta,2}^{m}-\rho](m/\delta)^{1/2}(\nabla\hat{\Psi}_m^{(\lceil t/\delta\rceil)}-\nabla\Psi) \cdot\nabla\zeta dtdx+O_P((m\delta)^{1/2}).
\end{split}
\end{equation}
Next we can show a similar result as Lemma \ref{lem5-1}, let 
$\epsilon_{m,j}=\sqrt{m}(\nabla_x\Psi_{\hat{\theta}^{(j)}_m}(x)-\nabla_x\Psi_\theta(x))-\nabla\tau(x)\bZ_j$, then by (A9),
\[|E\epsilon_{m,j}|=O(m^{-1/2-\alpha}), E|\epsilon_{m,j}|^2=o(1)\]
Since $\epsilon_{m,j}$ are independent, when $k=o(m^{1+2\alpha})$, similar calculation gives
\[E\left|\sqrt{mk}(\nabla_x\hat{\Psi}_m^{(k)}(x)-\nabla_x\Psi_\theta(x))-\sqrt{k}\nabla\tau(x)\bar{\bZ}_k\right|^2=E\left|k^{-1/2}\sum_{j=1}^k\epsilon_{m,j}\right|^2=o(1)\]
Then as $\delta\ra 0, m\ra\iy, \delta m^{1+2\alpha}\ra\iy$, we have
\begin{equation}\label{eq44-2}
\begin{split}
&E\left|\int_{[0,T]\times D}\rho(m/\delta)^{1/2}(\nabla\hat{\Psi}_m^{(\lceil t/\delta\rceil)}-\nabla\Psi) \cdot\nabla\zeta dtdx-\sum_{k=1}^{\lceil T/\delta\rceil}\int_{t_{k-1}}^{t_k}\int_D[\rho\delta^{-1/2}\nabla\tau(x)\bar{\bZ}_k]\cdot\nabla\zeta dxdt\right| \\
\leq&\sum_{k=1}^{\lceil T/\delta\rceil}\int_{t_{k-1}}^{t_k}\int_D\rho\delta^{-1/2}E\left|m^{1/2}(\nabla\hat{\Psi}_m^{(k)}-\nabla\Psi)-\nabla\tau(x)\bar{\bZ}_k\right| \cdot|\nabla\zeta| dxdt\\
\leq&\sum_{k=1}^{\lceil T/\delta\rceil}\cC\delta^{1/2}o(k^{-1/2})\\
=&o(1)
\end{split}
\end{equation}
and
\begin{equation}\label{eq44-3}
\int_{\mR_+\times \mR^d}[\hat{\rho}_{3}^{m}-\rho](m/\delta)^{1/2}(\nabla\hat{\Psi}_m^{(\lceil t/\delta\rceil)}-\nabla\Psi) \cdot\nabla\zeta dtdx=o_P(1)
\end{equation}
As $\delta\ra 0$, 
\begin{equation}\label{eq44-41}
\sum_{k=1}^{\lceil T/\delta\rceil}\int_{t_{k-1}}^{t_k}\int_D[\rho\delta^{-1/2}\nabla\tau\bar{\bZ}_k]\cdot\nabla\zeta dxdt  \ra\int_{\mR_+\times\mR^d}V_1(\pt_t \zeta-\nabla\Psi\cdot\nabla\zeta+\Delta\zeta)dtdx
\end{equation}
Combining (\ref{eq44-1}), (\ref{eq44-2}), (\ref{eq44-3}) and (\ref{eq44-41}), as $\delta m\ra 0$, we get
\[\int_{\mR_+\times\mR^d}(\hat{V}_{\delta,2}^m-V_1)(\pt_t \zeta- \nabla\Psi\cdot\nabla\zeta +\Delta\zeta))dtdx=o_p(1).\]
For any $\xi\in\cC_0^\iy(\mR_+\times\mR^d)$, find $\zeta\in\cC_0^\iy(\mR_+\times\mR^d)$ such that
\[\xi=\pt_t\zeta-\nabla\Psi\cdot\nabla\zeta+\Delta\zeta,\]
then we get the desired result. 

If we only assume (A8),(A9), then by the Fokker-Planck equation of $\tilde{\rho}_2^m$, we can obtain the similar equality as (\ref{eq44-0}), where we can replace $\hat{\rho}_{\delta,2}^m$ with $\tilde{\rho}_2^m$ and remove the $O_P(\delta)$ term. Then use the same argument we can prove $\tilde{V}_2^m$ also converges to $V_1$ as $\delta m^{1+2\alpha}\ra\iy$. Furthermore, if we do not assume (A9) (only assume (A8)), then in (\ref{eq44-2}), we have
\[m^{1/2}(\nabla\hat{\Psi}_m^{(k)}-\nabla\Psi)-\nabla\tau(x)\bar{\bZ}_k=O_P(m^{-1/2}),\]
then as $\delta\ra 0, m\ra\iy,\delta m\ra\iy$, $\tilde{V}_2^m$ still converges to $ V_1$.

\subsection{Proof of Proposition 5.1}
\label{appB4}
Assume $\beta=1$, For any $\zeta\in\cC_0^\iy(\mR_+\times\mR^d)$, the support of $\zeta$ is a subset of $[0,T]\times D$, where $D$ is a compact subset of $\mR^d$. Note that
\[
\int_{\mR_+\times\mR^d}\tilde{\rho}_{3}^{m}(\pt_t\zeta-\nabla_x\Psi_{\hat{\theta}^{(\lceil t/\delta\rceil)}_m}\cdot\nabla\zeta+\Delta\zeta) dtdx=-\int_{\mR^d}\zeta(0,x)\rho^0dx,
\]
together with (\ref{eq-38}), we get
\begin{equation}\label{eq42-1}
\begin{split}
& \int_{\mR_+\times\mR^d}\tilde{V}_{3}^m(\pt_t \zeta- \nabla\Psi\cdot\nabla\zeta +\Delta\zeta))dt dx  \\
=& \int_{\mR_+\times\mR^d}\rho(m/\delta)^{1/2}(\nabla_x\Psi_{\hat{\theta}^{(\lceil t/\delta\rceil)}_m}-\nabla\Psi) \cdot\nabla\zeta dtdx+\\
&\int_{\mR_+\times \mR^d}[\tilde{\rho}_{3}^{m}-\rho](m/\delta)^{1/2}(\nabla_x\Psi_{\hat{\theta}^{(\lceil t/\delta\rceil)}_m}-\nabla\Psi) \cdot\nabla\zeta dtdx.
\end{split}
\end{equation}
As $\delta\ra 0, m\ra\iy, \delta m\ra\iy$, we have
\begin{equation}\label{eq42-2}
\begin{split}
&\left|\int_{[0,T]\times D}\rho(m/\delta)^{1/2}(\nabla_x\Psi_{\hat{\theta}^{(\lceil t/\delta\rceil)}_m}-\nabla\Psi) \cdot\nabla\zeta dtdx-\sum_{k=1}^{\lceil T/\delta\rceil}\int_{t_{k-1}}^{t_k}\int_D[\rho\delta^{-1/2}\nabla\tau(x)\bZ_k]\cdot\nabla\zeta dxdt\right|\\
\leq&\sum_{k=1}^{\lceil T/\delta\rceil}\int_{t_{k-1}}^{t_k}\int_D\rho\delta^{-1/2}\left|m^{1/2}(\nabla_x\Psi_{\hat{\theta}^{(k)}_m}-\nabla\Psi)-\nabla\tau(x)\bZ_k\right| \cdot|\nabla\zeta| dxdt\\
\leq&\sum_{k=1}^{\lceil T/\delta\rceil}\cC\delta^{1/2}m^{-1/2}\\
=&o_P(1)
\end{split}
\end{equation}
and
\begin{equation}\label{eq42-3}
\int_{\mR_+\times \mR^d}[\tilde{\rho}_{3}^{m}-\rho](m/\delta)^{1/2}(\nabla_x\Psi_{\hat{\theta}^{(\lceil t/\delta\rceil)}_m}-\nabla\Psi) \cdot\nabla\zeta dtdx=o_P(1)
\end{equation}
As $\delta\ra 0$, we have 
\begin{equation}\label{eq42-4}
\begin{split}
& \sum_{k=1}^{\lceil T/\delta\rceil}\int_{t_{k-1}}^{t_k}\int_D[\rho\delta^{-1/2}\nabla\tau(x)\bZ_k]\cdot\nabla\zeta dxdt  \\
=& \sum_{k=1}^{\lceil T/\delta\rceil}\int_{t_{k-1}}^{t_k}\left\{\int_D(\nabla\zeta)'\rho\nabla\tau(x)dx\right\}\delta^{-1/2}\bZ_kdt\\
\ra&\int_0^T\left\{\int_D(\nabla\zeta)'\rho\nabla\tau(x)dx\right\}dW_t\\
=&\int_{\mR_+\times\mR^d}(\nabla\zeta)'\rho\nabla\tau(x)\dot{W}(t)dtdx\\
=&\int_{\mR_+\times\mR^d}-\Div(\rho\nabla\tau(x)\dot{W}(t))\zeta dtdx\\
=&\int_{\mR_+\times\mR^d}(-\pt_t V_2 + \Div ( V_2 \nabla \Psi ) + \Delta V_2)\zeta dtdx\\
=&\int_{\mR_+\times\mR^d}V_2(\pt_t \zeta-\nabla\Psi\cdot\nabla\zeta+\Delta\zeta)dtdx
\end{split}
\end{equation}
Combining (\ref{eq42-1}), (\ref{eq42-2}), (\ref{eq42-3}) and (\ref{eq42-4}), we get
\[\int_{\mR_+\times\mR^d}(\tilde{V}_{3}^m-V_2)(\pt_t \zeta- \nabla\Psi\cdot\nabla\zeta +\Delta\zeta))dtdx=o_p(1).\]
For any $\xi\in\cC_0^\iy(\mR_+\times\mR^d)$, find $\zeta\in\cC_0^\iy(\mR_+\times\mR^d)$ such that
\[\xi=\pt_t\zeta-\nabla\Psi\cdot\nabla\zeta+\Delta\zeta,\]
then we get the desired result. 

We can prove the solution is a.s. finite and unique using the same argument in the proof of Theorem 5.1, 
we only need to replace $t^{-1}W(s)$ with $\dot{W}(t)$ therein, and notice that for $a<1/2$,
\[\int_0^t s^{-a}\dot{W}(t)dt=\int_0^t s^{-a}dW(t)=O_P(1)\]




\subsection{Proof of Proposition 5.2}
\label{appB8}

Recall that, for fixed $\eta$, we have the CLT
\[(t/\delta)^{1/2}(\hat{\theta}_\delta(t)-\theta)\stackrel{d}{\ra} N(0,\gamma_\theta^2)\]
In the first part of the proof, we will show that there exists a Brownian motion $W(t)$, such that as $\delta\ra 0$,
\[t\delta^{-1/2}\gamma_\theta^{-1}(\hat{\theta}_\delta(t)-\theta)\Rightarrow W(t)\]
Indeed, denote $Y_i=f(X(i\eta))=\nabla_x\Psi_\theta(X(i\eta))$, $E_\pi Y_i=0$, denote the autocovariance $\gamma_k=\cov_\pi(Y_i,Y_{i+k})$ (this should be distinguished from $\gamma_\theta$), denote $S_n=\sum_{i=1}^n Y_i$, then for $m\geq 2n$, 
\[\cov_\pi(S_n,S_m)=n\gamma_0+\sum_{i=1}^{n-1}(2n-i)\gamma_i+n\sum_{i=n}^{m-n}\gamma_i+\sum_{i=m-n+1}^{m-1}(m-i)\gamma_i\]
for $n\leq m<2n$,
\[\cov_\pi(S_n,S_m)=n\gamma_0+\sum_{i=1}^{m-n}(2n-i)\gamma_i+\sum_{i=m-n+1}^{n-1}(m+n-2i)\gamma_i+\sum_{i=n}^{m-1}(m-i)\gamma_i\]
In both cases, when $\{Y_i\}$ satisfy certain mixing conditions (which is guaranteed under the assumptions of Proposition \ref{prop-th}), we will have
\[n^{-1}\cov_\pi(S_n,S_m)\ra \gamma_0+2\sum_{i=1}^{\iy}\gamma_i=\Sigma_f\]
Thus for $0<s\leq t$, as $\delta\ra 0$,
\[\delta\cov_\pi(S_{\lceil s/\delta\rceil},S_{\lceil t/\delta\rceil})\ra s\Sigma_f\]
Therefore
\begin{equation}
\left(
\begin{array}{cc}
     \delta^{1/2}S_{\lceil t/\delta\rceil}  \\
     \delta^{1/2}S_{\lceil s/\delta\rceil}
\end{array}
\right)\stackrel{d}{\ra}
N\left(0,\left(
\begin{array}{cc}
     t\Sigma_f &\  s\Sigma_f  \\
     s\Sigma_f &\  s\Sigma_f
\end{array}
\right)\right)
\end{equation}
Use the same argument in the proof of Proposition \ref{prop-th}, we get
\begin{equation}
\left(
\begin{array}{cc}
     \delta^{1/2}\sum_{i=1}^{\lceil t/\delta\rceil}\nabla_\theta'f(X(i\eta))\cdot(\hat{\theta}_\delta(t)-\theta)  \\
     \delta^{1/2}\sum_{i=1}^{\lceil s/\delta\rceil}\nabla_\theta'f(X(i\eta))\cdot(\hat{\theta}_\delta(s)-\theta)
\end{array}
\right)\stackrel{d}{\ra}
N\left(0,\left(
\begin{array}{cc}
     t\Sigma_f &\  s\Sigma_f  \\
     s\Sigma_f &\  s\Sigma_f
\end{array}
\right)\right)
\end{equation}
Since $\frac{1}{t/\delta}\sum_{i=1}^{\lceil t/\delta\rceil}\nabla_\theta'f(X(i\eta))\ra E_\pi(\nabla_\theta'f(X))=A$, then
\begin{equation}
\left(
\begin{array}{cc}
     t\delta^{-1/2}(\hat{\theta}_\delta(t)-\theta)  \\
     s\delta^{-1/2}(\hat{\theta}_\delta(s)-\theta)
\end{array}
\right)\stackrel{d}{\ra}
N\left(0,\left(
\begin{array}{cc}
     t\Gamma_\theta &\  s\Gamma_\theta  \\
     s\Gamma_\theta &\  s\Gamma_\theta
\end{array}
\right)\right)
\end{equation}
Left-multiplied by $\gamma_\theta^{-1}$, we get the finite-dimensional convergence:
\begin{equation}
\left(
\begin{array}{cc}
     t\delta^{-1/2}\gamma_\theta^{-1}(\hat{\theta}_\delta(t)-\theta)  \\
     s\delta^{-1/2}\gamma_\theta^{-1}(\hat{\theta}_\delta(s)-\theta)
\end{array}
\right)\stackrel{d}{\ra}
\left(
\begin{array}{cc}
     W(t)  \\
     W(s)
\end{array}
\right)
\end{equation}
Let $W_{\delta,1}(t)=t\delta^{-1/2}\gamma_\theta^{-1}(\hat{\theta}_\delta(t)-\theta)$, since $W_{\delta,1}(t)-W_{\delta,1}(s)$ converges to $W(t)-W(s)$, we can show the tightness of $W_{\delta,1}(t)-W_{\delta,1}(s)$. Thus $W_{\delta,1}(t)$ weakly converges to the Brownian motion $W(t)$. By Skorohord's representation theorem, we can realize them on the same probability space and get the almost surely convergence.

Assume $\beta=1$, for any $\zeta\in\cC_0^\iy(\mR_+\times\mR^d)$, assume the support of $\zeta$ is a subset of $[0,T]\times D$, where $D$ is a compact subset of $\mR^d$. By Taylor expansion, for $x\in D$, we have uniformly
\[t\delta^{-1/2}(\nabla_x\Psi_{\hat{\theta}_\delta(t)}-\nabla\Psi)-\nabla\tau(x)W(t)=o_P(1)\]
Note that
\[\int_{\mR_+\times\mR^d}\tilde{\rho}_2(\pt_t\zeta-\nabla_x\Psi_{\hat{\theta}_\delta(t)}\cdot\nabla\zeta+\Delta\zeta) dtdx=-\int_{\mR^d}\zeta(0,x)\rho^0(x)dx\]
together with (\ref{eq-38}), we get
\begin{equation}\label{eq46-1}
\begin{split}
& \int_{\mR_+\times\mR^d}\tilde{V}_2(\pt_t \zeta- \nabla\Psi\cdot\nabla\zeta +\Delta\zeta))dt dx  \\
=& \int_{\mR_+\times\mR^d}\rho\delta^{-1/2}(\nabla_x\Psi_{\hat{\theta}_\delta(t)}-\nabla\Psi)\cdot\nabla\zeta dtdx+\\
&\int_{\mR_+\times \mR^d}[\tilde{\rho}_2-\rho]\delta^{-1/2}(\nabla_x\Psi_{\hat{\theta}_\delta(t)}-\nabla\Psi) \cdot\nabla\zeta dtdx\\
\ra&\int_{\mR_+\times\mR^d}\rho\nabla\tau(x)t^{-1}W(t)\cdot\nabla\zeta dtdx \\
=&\int_{\mR_+\times\mR^d}V_1(\pt_t \zeta-\nabla\Psi\cdot\nabla\zeta+\Delta\zeta)dtdx
\end{split}
\end{equation}
For the convergence in the fourth line, we can use the same argument as in the proof of Theorem 5.1
, that is, we can prove both integrals are negligible around 0 and show the convergence of the integral on $t\in[\eta,T]$. 

Finally for any $\xi\in\cC_0^\iy(\mR_+\times\mR^d)$, find $\zeta\in\cC_0^\iy(\mR_+\times\mR^d)$ such that
\[\xi=\pt_t\zeta-\nabla\Psi\cdot\nabla\zeta+\Delta\zeta,\]
then we get the desired result.

\subsection{Proof of Proposition 5.3}
\label{appB9}
Solving the Langevin equation (\ref{Lang}), we get
\[X_t=\theta+e^{-t}(X_0-\theta)+\sqrt{2}e^{-t}\int_0^t e^sdB_s,\]
\[\hat{\theta}(t)-\theta=\frac{1}{t}\left((1-e^{-t})(X_0-\theta)+\sqrt{2}\int_0^t (1-e^{u-t})dB_u\right),\]
Thus it is biased. One may try to subtract the bias to create an unbiased estimator of $\theta$, e.g. let 
\[\hat{\theta}'(t)=\frac{\hat{\theta}(t)-a_tX_0}{1-a_t},\]
where $a_t=(1-e^{-t})/t$, then
\[\hat{\theta}'(t)-\theta=\frac{\sqrt{2}}{t(1-a_t)}\int_0^t (1-e^{u-t})dB_u,\]
which is unbiased but still has non-negligible variance. 

Next, let $Y_t=X_t-\theta$, we have
\[Y_t=e^{-t}Y_0+\sqrt{2}e^{-t}\int_0^t e^sdB_s\]
We plug in the solution of $Y_t$ to the expression of $\hat{\theta}_\delta(t)$ and $\hat{\theta}(t)$, note that $Y_0=0$,
\[\hat{\theta}_\delta(t)-\theta=\frac{1}{n}\sum_{i=1}^n\left(\sqrt{2}e^{-i\delta}\int_0^{i\delta}e^sdB_s\right)\]
\[\hat{\theta}(t)-\theta=\frac{1}{t}\left(\sqrt{2}\int_0^t (1-e^{s-t})dB_s\right)\]
then 
\[\hat{\theta}_\delta(t)-\hat{\theta}(t)=\sqrt{2}\sum_{i=0}^{n-1}\int_{i\delta}^{(i+1)\delta}\left[\frac{e^s}{n}(e^{-(i+1)\delta}+\cdots+e^{-n\delta})-\frac{1-e^{s-t}}{t}\right]dB_s\]
Now we analyze the variance of $t(\hat{\theta}_\delta(t)-\hat{\theta}(t))$, consider the integral below
\[\begin{split}
J_i&=\int_{i\delta}^{(i+1)\delta}\left[\frac{\delta e^se^{-(i+1)\delta}(1-e^{-(n-i)\delta})}{1-e^{-\delta}}-(1-e^{s-t})\right]^2ds\\
&=\int_{0}^{\delta}\left[\frac{\delta e^s(1-e^{i\delta-t})}{e^\delta-1}+e^{s+i\delta-t}-1\right]^2ds
\end{split}\]
Denote $a_i=\frac{\delta(1-e^{i\delta-t})}{e^\delta-1}+e^{i\delta-t}$ (the subscript $i$ is omitted below for convenience), then
\[J_i=\int_{0}^{\delta}(ae^s-1)^2ds=\frac{a^2}{2}(e^{2\delta}-1)-2a(e^{\delta}-1)+\delta\]
Since $\frac{\delta}{e^{\delta}-1}=1-\delta/2+o(\delta)$, $a-1=(1-e^{i\delta-t})[-1/2+o(1)]\delta$, $a=1+o(1)$, we have
\[\begin{split}
J_i&=\frac{a^2}{2}(2\delta+2\delta^2+8\delta^3/6+o(\delta^3))-2a(\delta+\delta^2/2+\delta^3/6+o(\delta^3))+\delta\\
&=(a-1)^2\delta+a(a-1)\delta^2+(2a^2/3-a/3)\delta^3+o(\delta^3)\\
&=(1-e^{i\delta-t})^2[-1/2+o(1)]^2\delta^3+[1+o(1)](1-e^{i\delta-t})[-1/2+o(1)]\delta^3+\\
&[1/3+o(1)]\delta^3+o(\delta^3)\\
&=((1-e^{i\delta-t})^2/4-(1-e^{i\delta-t})/2+1/3)\delta^3+o(\delta^3)\\
&=\frac{1}{12}(1+3e^{2(i\delta-t)})\delta^3+o(\delta^3)
\end{split}\]
The variance of $n(\hat{\theta}_\delta(t)-\hat{\theta}(t))$ is $2\sum_{i=0}^{n-1}J_i/\delta^2$, when $\delta=t/n\ra 0$, 
\[2\sum_{i=0}^{n-1}J_i/\delta^2=\sum_{i=0}^{n-1}\left[\frac{1}{6}(1+3e^{2(i\delta-t)})\delta+o(\delta)\right]\ra\frac{1}{6}\int_0^t[1+3e^{2(s-t)}]ds=\frac{1}{6}t+\frac{1}{4}(1-e^{-2t})\]
Since $n(\hat{\theta}_\delta(t)-\hat{\theta}(t))$ is normally distributed with mean 0,
\[n(\hat{\theta}_\delta(t)-\hat{\theta}(t))\stackrel{d}{\ra} N\left(0, \frac{1}{6}t+\frac{1}{4}(1-e^{-2t})\right).\]

However, the variance of $t/\delta(\hat{\theta}_\delta(t)-\hat{\theta}(t))$ does not converge for arbitrary $\delta\ra 0$, e.g., take $t=1$, $\delta_m=0.0001m$, $m=1,2,\dots,100$, we calculate and plot the variance in Figure \ref{fig0}. We find that only when $m$ is a factor of $10000$, that is $m=1,2,4,5,8,10,16,20,25,40,50,80,100$, $t/\delta$ is integer, the corresponding variances converges to $\frac{1}{6}+\frac{1}{4}(1-e^{-2})\approx 0.3828$, for other $m$, the variances do not converge, this is because the ceil of $t/\delta$ creates an irregular bias $\lceil t/\delta\rceil\delta-t$, which is of order $\delta$. If we change the order of $\delta$ in the CLT, the limiting distribution will be either $0$ or $\iy$, which is trivial. Thus we conclude that there are no non-trivial higher-order convergence result for $\hat{\theta}_\delta(t)-\hat{\theta}(t)$ for arbitrary sequence $\{\delta\ra 0\}$.

\begin{figure}[htbp]
\centering
\includegraphics[width=0.88\textwidth]{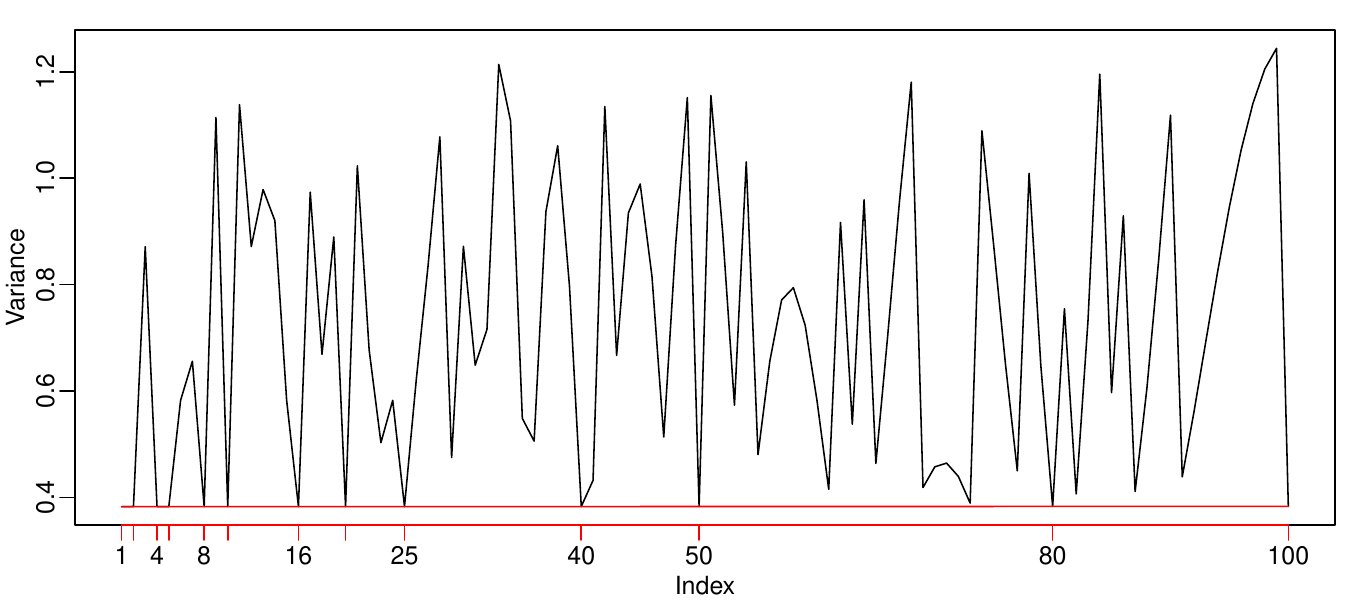}
\caption{Variance of $t/\delta(\hat{\theta}_\delta(t)-\hat{\theta}(t))$ for $\delta_m=0.0001m$, $m=1,2,\dots,100$}
\label{fig0}
\end{figure}

For $s$ and $t$ such that $t/s$ is an irrational number, it is impossible to find $\delta$ such that both $s/\delta$ and $t/\delta$ are integers. This means we cannot establish the finite dimensional convergence for $(s/\delta(\hat{\theta}_\delta(s)-\hat{\theta}(s)),t/\delta(\hat{\theta}_\delta(t)-\hat{\theta}(t)))$, even for a particular convergent sequence of $\delta$. Therefore we cannot derive any higher-order convergence result for the process $\{\hat{\theta}_\delta(t)-\hat{\theta}(t)\}$.

\newpage
\section{Proof in Section 6}

First we show a result which controls the difference between the algorithm and the ODE, 

\begin{prop}\label{prop66}
Assume $|\nabla\Psi(x)|\leq C(1+|x|^p)$ for some $p>0$, and for fixed $T>0$, $\Sigma_t$ is non-degenerated over the interval $[0,T]$, then we have
\[\max_{t\in[0,T]}|\mu_\delta(t)-\mu(t)|=O(\delta), \max_{t\in[0,T]}|\Sigma_\delta(t)-\Sigma(t)|=O(\delta).\]
\end{prop} 

\begin{proof}
The partial derivative of $p(x;\mu_,\Sigma)$ w.r.t. $\mu$ and $\Sigma$ are:
\begin{equation}\label{eqe2}
\begin{aligned}
\nabla_\mu p(x;\mu_,\Sigma)&=p\Sigma^{-1}(x-\mu) \\
\nabla_\Sigma p(x;\mu_,\Sigma)&=\frac{1}{2}p(-\Sigma^{-1}+\Sigma^{-1}(x-\mu)(x-\mu)'\Sigma^{-1})
\end{aligned}
\end{equation}
When $\Sigma_t$ is non-degenerated and continuous over a fixed interval $[0,T]$, all eigenvalues of $\Sigma_t$ are bounded in $[C_1,C_2], (C_1>0)$. Then $\nabla_\mu p(x;\mu_,\Sigma)$ and $\nabla_\Sigma p(x;\mu_,\Sigma)$ are bounded, and the normal density function $p$ is a Lipschitz function of $\mu$ and $\Sigma$, that is, there exists $L>0$, such that,
\begin{equation}\label{eqe3}
|p(x;\mu_1,\Sigma_1)-p(x;\mu_2,\Sigma_2)|\leq L(|\mu_1-\mu_2|+\|\Sigma_1-\Sigma_2\|_F)
\end{equation}
Moreover, if $|\nabla\Psi(x)|\leq C(1+|x|^p)$, then
\begin{equation}\label{eqe31}
\begin{aligned}
&\mE_{p_1}(\nabla\Psi(X))-\mE_{p_2}(\nabla\Psi(X)) \\
=&\int_{\mR^d}(p(x;\mu_1,\Sigma_1)-p(x;\mu_2,\Sigma_2))\nabla\Psi(x)dx \\
\leq&\int_{\mR^d}(|\nabla_\mu p(x;\tilde{\mu},\tilde{\Sigma})|\cdot|\mu_1-\mu_2|+\|\nabla_\Sigma p(x;\tilde{\mu},\tilde{\Sigma})\|_F\cdot\|\Sigma_1-\Sigma_2\|_F)\nabla\Psi(x)dx \\
\leq&C(|\mu_1-\mu_2|+\|\Sigma_1-\Sigma_2\|_F)
\end{aligned}
\end{equation}

Now we look at the difference of $\mu$ and $\Sigma$ between consecutive steps in the algorithm. In the current step, the distribution is denoted by $p_0=N(\mu_0,\Sigma_0)$, and in the next step, we take  $p=N(\mu,\Sigma)$ as the minimizer of $\frac{1}{2} \left[ d_W \left(p, p_0  \right)\right]^2 + \delta \mF_\Psi(p)$, then by \cite{lambert2022variational}, $\mu-\mu_0=O(\delta)$ and $\Sigma-\Sigma_0=O(\delta)$, and by taking partial derivative in the JKO scheme, we get
\[\frac{\mu-\mu_0}{\delta}=-\mE_{p}(\nabla\Psi(X))=-\mE_{p_0}(\nabla\Psi(X))+O(\delta)\]
Furthermore, we have
\begin{equation}\label{eqe4}
\begin{aligned}
\mu-\mu_0&=-\mE_{p_0}(\nabla\Psi(X))\delta+O(\delta^2) \\
\Sigma-\Sigma_0&=[2I-\mE_{p_0}(\nabla\Psi(X)(X-\mu_0)'+(X-\mu_0)\nabla '\Psi(X))]\delta+O(\delta^2) 
\end{aligned}
\end{equation}

Let $a_k=|\mu_\delta(k\delta)-\mu(k\delta)|$, $b_k=|\Sigma_\delta(k\delta)-\Sigma(k\delta)|$, then $a_0=b_0=0$, and
\begin{equation}\label{eqe5}
\begin{aligned}
a_{k+1}&\leq a_k+|\mu_\delta((k+1)\delta)-\mu_\delta(k\delta)-[\mu((k+1)\delta)-\mu(k\delta)]| \\
&=a_k+\left|-\mE_{p_\delta(k\delta)}(\nabla\Psi(X))\delta+O(\delta^2)+\int_{k\delta}^{(k+1)\delta}\mE_{p(s)}(\nabla\Psi(X))ds\right| \\
&=a_k+|-(\mE_{p_\delta(k\delta)}(\nabla\Psi(X))-\mE_{p(k\delta)}(\nabla\Psi(X)))\delta+O(\delta^2)| \\
&\leq a_k+C\delta(a_k+b_k)+O(\delta^2)
\end{aligned}
\end{equation}
Similarly, we have
\[b_{k+1}\leq b_k+C\delta(a_k+b_k)+O(\delta^2)\]
then if $k<T/\delta$,
\[a_{k+1}+b_{k+1}\leq (1+C\delta)(a_k+b_k)+O(\delta^2)\leq T/\delta(1+C\delta)^{T/\delta}O(\delta^2)=O(\delta) \]
\end{proof}

\subsection{Proof of Theorem 6.1}
Recall that
\[ \sqrt{n} [\nabla \hat{\Psi}_n(x) - \nabla\Psi(x)] \stackrel{a.s.}{\longrightarrow}  \nabla\tau(x) \bZ.\]
Since $\dot{\mu}_t=-\mE_{p_t}(\nabla\Psi(X)), \dot{\mu}_t^n=-\mE_{p_t^n}(\nabla\hat{\Psi}_n(X))$, we have
\begin{equation}\label{eqe7}
\begin{aligned}
\dot V_\mu^n(t)&=-\int_{\mR^d} \sqrt{n}[\nabla \hat{\Psi}_n(x) - \nabla\Psi(x)]p_t^n(x)dx-\int_{\mR^d}\nabla\Psi(x)\sqrt{n}(p_t^n(x)-p_t(x))dx \\
&=O_P(1)-\int_{\mR^d}\nabla\Psi(x) \sqrt{n}\left[\nabla_\mu p(x;\tilde{\mu},\tilde{\Sigma})\cdot(\mu_t^n-\mu_t)+\nabla_\Sigma p(x;\tilde{\mu},\tilde{\Sigma})\cdot(\Sigma_t^n-\Sigma_t)\right]dx,
\end{aligned}
\end{equation}
then 
\[|\dot V_\mu^n(t)|\leq O_P(1)+O_P(1)(|V_\mu^n(t)|+|V_\Sigma^n(t)|)\]
Similarly,
\[|\dot V_\Sigma^n(t)|\leq O_P(1)+O_P(1)(|V_\mu^n(t)|+|V_\Sigma^n(t)|),\]
using Gronwall's inequality, we can show that
\begin{equation}\label{eqe8}
\max_{t\in[0,T]}(|V_\mu^n(t)|+|V_\Sigma^n(t)|)=O_P(1), \max_{t\in[0,T]}(|\dot V_\mu^n(t)|+|\dot V_\Sigma^n(t)|)=O_P(1)
\end{equation}
Then by the first line of (\ref{eqe7}),
\begin{equation}\label{eqe9}
\dot V_\mu^n(t)=-\int_{\mR^d} \nabla\tau(x) \bZ p_t(x)dx-\int_{\mR^d}\nabla\Psi(x)\left[\nabla_\mu p(x;\mu_t,\Sigma_t)\cdot V_\mu^n(t)+\nabla_\Sigma p(x;\mu_t,\Sigma_t)\cdot V_\Sigma^n(t)\right]dx+o_p(1)
\end{equation}
For $V_\Sigma^n(t)$, we can also get
\begin{equation}\label{eqe10}
\begin{aligned}
\dot V_\Sigma^n(t)=&-\int_{\mR^d} [\nabla\tau(x) \bZ(x-\mu_t)'-\nabla\Psi(x)V_\mu^n(t)']p_t(x)dx- \\
&\int_{\mR^d}\nabla\Psi(x)(x-\mu_t)'[\nabla_\mu p_t(x)\cdot V_\mu^n(t)+\nabla_\Sigma p_t(x)\cdot V_\Sigma^n(t)]dx- \\
&\int_{\mR^d} [(x-\mu_t)(\nabla\tau(x) \bZ)'-V_\mu^n(t)\nabla'\Psi(x)]p_t(x)dx- \\
&\int_{\mR^d}(x-\mu_t)\nabla'\Psi(x)[\nabla_\mu p_t(x)\cdot V_\mu^n(t)+\nabla_\Sigma p_t(x)\cdot V_\Sigma^n(t)]dx+o_p(1)
\end{aligned}
\end{equation}
then as $n\ra\iy$, $V_\mu^n(t)$ and $V_\Sigma^n(t)$ will converge to the solution to $V_\mu(t)$ and $V_\Sigma(t)$. By the tightness (\ref{eqe8}), we can also show the weak convergence result. 

Since $\frac{|\nabla\hat{\Psi}_n(x)|}{1+|x|^p}=O_P(1)$,
we can prove Proposition \ref{prop66} with $O(\delta)$ replaced by $O_P(\delta)$. If $\delta\sqrt{n}\ra 0$, by Delta method we have
\[\sqrt{n}(\mu_\delta^n(t)-\mu(t))\ra V_\mu(t), \sqrt{n}(\Sigma_\delta^n(t)-\Sigma(t))\ra V_\Sigma(t).\] 

\subsection{Proof of Theorem 6.2}
Given (A7) and $\delta m^{1+2\alpha}\ra \iy$, use the same technique in the proof of Theorem 5.1
, we can also show:
\begin{equation}\label{eqe13}
\int_{[0,T]\times\mR^d} \left[(m/\delta)^{1/2}(\nabla \hat{\Psi}_{\lceil t/\delta \rceil}(x) - \nabla\Psi(x))-\delta^{-1/2}\nabla\tau(x) \bar{\bZ}_{\lceil t/\delta \rceil}\right] p_t(x)dtdx=o_P(1),
\end{equation}
and
\begin{equation}\label{eqe15}
\int_{[0,T]\times\mR^d}\delta^{-1/2}\nabla\tau(x) \bar{\bZ}_{\lceil t/\delta \rceil} p_t(x)dtdx\ra \int_{[0,T]\times\mR^d}\nabla\tau(x)t^{-1}W_tp_t(x)dtdx=O_P(1)
\end{equation}
By (\ref{eqe13}) and (\ref{eqe15}), we have
\[\int_{[0,T]\times\mR^d} (m/\delta)^{1/2}(\nabla \hat{\Psi}_{\lceil t/\delta \rceil}(x) - \nabla\Psi(x))p_t(x)dtdx=O_P(1)\]
Just like (\ref{eqe7}), we get
\begin{equation}\label{eqf1}
\begin{aligned}
V_{\mu,1}^m(T)=&-\int_{[0,T]\times\mR^d} (m/\delta)^{1/2}[\nabla \hat{\Psi}_{\lceil t/\delta \rceil}(x) -\nabla\Psi(x)]p_1^m(t,x)dtdx-\\
&\int_{[0,T]\times\mR^d}\nabla\Psi(x)(m/\delta)^{1/2}(p_1^m(t,x)-p_t(x))dtdx \\
=&O_P(1)-\int_{[0,T]\times\mR^d}\nabla\Psi(x) (m/\delta)^{1/2}\left[\nabla_\mu p(x;\tilde{\mu},\tilde{\Sigma})\cdot(\mu_1^m(t)-\mu(t))+\right.\\
&\left.\nabla_\Sigma p(x;\tilde{\mu},\tilde{\Sigma})\cdot(\Sigma_1^m(t)-\Sigma(t))\right]dtdx \\
\leq &O_P(1)+\int_0^t O_P(1)(|V_{\mu,1}^{m}(t)|+|V_{\Sigma,1}^{m}(t)|)dt
\end{aligned}
\end{equation}
Similarly, we have
\[|V_{\Sigma,1}^{m}(T)|\leq O_P(1)+\int_0^T O_P(1)(|V_{\mu,1}^{m}(t)|+|V_{\Sigma,1}^{m}(t)|)ds,\]
By Gronwall's inequality,
\begin{equation}\label{eqf2}
\max_{t\in[0,T]}(|V_{\mu,1}^{m}(t)|+|V_{\Sigma,1}^{m}(t)|)=O_P(1)
\end{equation}
Then by (\ref{eqe13}), (\ref{eqe15}), (\ref{eqf1}) and (\ref{eqf2}),
\[\begin{split}
V_{\mu,1}^m(T)=&-\int_{[0,T]\times\mR^d} \nabla\tau(x) t^{-1}W_t p_t(x)dtdx- \\
&\int_{[0,T]\times\mR^d}\nabla\Psi(x)\left[\nabla_\mu p_t(x)\cdot V_{\mu,1}^m(t)+\nabla_\Sigma p_t(x)\cdot V_{\Sigma,1}^m(t)\right]dtdx+o_P(1)
\end{split}\]
Then $V_{\mu,1}^m(t)$ will converge to $V_{\mu,1}(t)$, similarly we can show $V_{\Sigma,1}^{m}(t)$ converges to $V_{\Sigma,1}(t)$.

Since $\frac{\hat{\Psi}_k(x)}{1+|x|^p}=O_P(1)$, Proposition \ref{prop66} still holds with $O_P(\delta)$, then if $\delta m\ra 0$, by Delta method we have
\[(m/\delta)^{1/2}(\mu_{\delta,1}^{m}(t)-\mu(t))\ra V_{\mu,1}(t), (n/\delta)^{1/2}(\Sigma_{\delta,1}^{m}(t)-\Sigma(t))\ra V_{\Sigma,1}(t).\] 

\end{document}